\documentclass[10pt]{amsart}

\usepackage[utf8]{inputenc}
\usepackage{kantlipsum}
\usepackage[ colorlinks=true,linkcolor=blue, citecolor={green}]{hyperref}
\usepackage{lmodern}\usepackage[T1]{fontenc}
\usepackage{amsmath,amssymb,amsfonts,euscript,graphicx,hyperref,indentfirst}
\usepackage{enumerate,linegoal}
\usepackage[all]{xy}
\usepackage{tikz-cd}
\usepackage{stmaryrd}
\usepackage{relsize,exscale}
\usepackage{bigints}
\usepackage{cancel}
\usepackage{accents}
\usepackage{bbm}
\usepackage[font=footnotesize]{caption} 
\usepackage{diagbox}
\usepackage{soul}
\usepackage[extra]{tipa}
\usepackage{array, boldline, makecell, booktabs}
\usepackage{multirow}
\usepackage{algpseudocode}
\usepackage{mathrsfs}
\usepackage{subcaption}
\usepackage{colortbl}

\newcolumntype{?}{!{\vrule width 2pt}}
\newcolumntype{a}{!{\vrule width 1pt}}

\newcommand{\ME}{ \text{\tiny  \it ME}}
\newcommand{\CE}{ \text{\tiny  \it CE}}
\newcommand{\Tree}{ \text{\tiny  \it Tree}}
\newcommand{\TreeEj}{ \text{\tiny  \it TreeEj}}
\newcommand{\minus}{\scalebox{0.5}{$-$}}
\newcommand{\plus}{\scalebox{0.5}{$+$}}

\calclayout

\newcommand{\Sec}{Section }

\newtheorem*{corollary*}{Corollary}

\newtheorem{theorem}{Theorem}[section]
\newtheorem{lemma}[theorem]{Lemma}
\newtheorem{proposition}[theorem]{Proposition}

\newtheorem{definition-theorem}[theorem]{Definition-Theorem}

\theoremstyle{definition}
\newtheorem{definition}[theorem]{Definition}
\newtheorem{definition-notation}[theorem]{Definition-Notation}
\newtheorem{example}[theorem]{Example}

\newtheorem{algorithm}[theorem]{Algorithm}

\theoremstyle{remark}
\newtheorem{remark}[theorem]{Remark}
\makeatletter
\def\l@subsection{\@tocline{2}{0pt}{2.5pc}{5pc}{}} 
\makeatother

\numberwithin{equation}{section}

\title{On marginal feature attributions of tree-based models}

\author[Filom, Miroshnikov, Kotsiopoulos, Ravi Kannan]{Khashayar Filom$^{*,\dagger}\quad$ Alexey Miroshnikov$^{*,\ddagger}\quad$
Konstandinos Kotsiopoulos $^{*,\S}\quad$   \\ Arjun Ravi Kannan $^{*,\P}$}

\thanks{$^*$Emerging Capabilities Research Group, Discover Financial Services Inc., Riverwoods, IL}
\thanks{$^\dagger$ first author, khashayarfilom@discover.com}
\thanks{$^\ddagger$ alexeymiroshnikov@discover.com}
\thanks{$^\S$ kostaskotsiopoulos@discover.com}
\thanks{$^\P$ arjunravikannan@discover.com}

\date{}

\begin{document}
\vspace*{-2cm}
\maketitle
\begin{abstract}
Due to their power and ease of use, tree-based machine learning models, such as random forests and gradient-boosted tree ensembles, have become very popular. To interpret them, local feature attributions based on marginal expectations, e.g. marginal (interventional) Shapley, Owen or Banzhaf values, may be employed. Such methods are true to the model and implementation invariant, i.e. dependent only on the input-output function of the model. 
We contrast this with the popular TreeSHAP algorithm by presenting two (statistically similar) decision trees that compute the exact same function for which the ``path-dependent'' TreeSHAP yields different rankings of features, whereas the marginal Shapley values coincide. Furthermore, we discuss how the internal structure of tree-based models may be leveraged to help with computing their marginal feature attributions according to a linear game value. One important observation is that these are simple (piecewise-constant) functions with respect to a certain grid partition of the input space determined by the trained model. Another crucial observation, showcased by experiments with XGBoost, LightGBM  and CatBoost libraries, is that only a portion of all features appears in a tree from the ensemble. Thus, the complexity of computing marginal Shapley (or Owen or Banzhaf) feature attributions may be reduced. This remains valid for a broader class of game values which we shall axiomatically characterize. A prime example is the case of CatBoost models where the trees are oblivious (symmetric) and the number of features in each of them is no larger than the depth. We exploit the symmetry to derive an explicit formula, with improved complexity and only in terms of the internal model parameters, for marginal Shapley (and Banzhaf and Owen) values of CatBoost models. This results in a fast, accurate algorithm for estimating these feature attributions.
\end{abstract}



\section{Introduction}\label{sec:Intro}
\subsection{Motivation and summary of results}\label{subsec:motive}
Ensemble methods combine a group of weak learners to produce a strong learner \cite{dietterich2000ensemble,hastie2009elements}.  In the context of tree-based models, examples of such an approach include random forests \cite{breiman2001random} and gradient-boosted tree ensembles \cite{friedman2001greedy}. Given their superior performance on structured data in various tasks \cite{roe2005boosted,caruana2006empirical,wu2010adapting,he2014practical,grinsztajn2022tree,shwartz2022tabular}, tree-based models are frequently used in regulated domains such as financial services \cite{chopra2018application,hu2021supervised} and healthcare 
\cite{turgeman2016mixed,zhang2019predictive}.\\
\indent
Ensemble models are considered to be complex which raises concerns about their interpretability \cite{rudin2019stop,2021arXiv211101743S}, especially in view of regulations that require it, such as 
the Equal Credit Opportunity Act (ECOA) \nocite{ECOA} and 
the Fair Housing Act (FHA). \nocite{FHA}
Financial institutions in the United States (US), 
for instance, are  required under the ECOA 
to notify declined or negatively impacted applicants 
of the main factors that led to the adverse action.
Determining the factor contributing the most to an outcome of a model may be done via individualized feature attributions.\\
\indent
An important approach to feature attribution is to utilize the celebrated work of Shapley \cite{shapley1953value} from cooperative game theory \cite{vstrumbelj2014explaining,lundberg2017unified}.
To elaborate, consider the features (predictors) as random variables $\mathbf{X}=(X_1,\dots,X_n)$ on a probability space $(\Omega,\mathcal{F},\Bbb{P})$. Given a model $f(\mathbf{X})$, one can define certain games  with $X_1,\dots,X_n$ being the players. The machinery of the \textit{Shapley value} \cite{shapley1953value} then allows us to quantify the contribution of each feature to a prediction of the model. Two of the most notable games in the literature are the \textit{marginal} and \textit{conditional} games, 
which are defined respectively in terms of \textit{marginal expectations} or 
\textit{conditional expectations}\footnote{The expectations $v^\ME(S;\mathbf{X},f)$ and $v^\CE(S;\mathbf{X},f)$ are sometimes called \textit{interventional} and \textit{observational} expectations in the literature. 
Other common notations are  
$\Bbb{E}[f(\mathbf{x}_S,\mathbf{X}_{-S})]|_{x_S=\mathbf{X}_S}$ or 
$\Bbb{E}_{\mathbf{X}_{-S}}\left[f(\mathbf{x}_S,\mathbf{X}_{-S})\right]$ for 
$v^\ME(S;\mathbf{X},f)$, and   
$\Bbb{E}[f(\mathbf{X})\mid S]$ for $v^\CE(S;\mathbf{X},f)$.} 
(see \cite{2020arXiv200616234C})
\begin{equation}\label{games}
v^\ME(S;\mathbf{X},f)(\mathbf{x})
:=\Bbb{E}[f(\mathbf{x}_S,\mathbf{X}_{-S})], \quad 
v^\CE(S;\mathbf{X},f)(\mathbf{x})
:=\Bbb{E}[f(\mathbf{X})\mid \mathbf{X}_S=\mathbf{x}_S]
\quad\quad (S\subseteq N).
\end{equation}  
Here, $\mathbf{x}$ is an arbitrary point of  $\Bbb{R}^n$, $S$ is a subset of $N:=\{1,\dots,n\}$, and $\mathbf{X}_S$ (respectively $\mathbf{X}_{-S}$) denote the collections of $X_i$'s with $i\in S$ (resp. $i\in N\setminus S$). Each 
$\left\{v^\ME(S;\mathbf{X},f)\right\}_{S\subseteq N}$ or 
$\left\{v^\CE(S;\mathbf{X},f)\right\}_{S\subseteq N}$ 
is a collection of functions $\Bbb{R}^n\rightarrow\Bbb{R}$, and hence they define  $n$-person games of the form $2^N\rightarrow\Bbb{R}$ for each $\mathbf{x}\in\Bbb{R}^n$. 
Thus, 
$v^\ME=v^\ME(\cdot;\mathbf{X},f)$
and 
$v^\CE=v^\CE(\cdot;\mathbf{X},f)$ 
are pointwise games. 
Shapley values provide a canonical way of assigning importance scores to $X_i$'s in any such game:
\begin{equation}\label{Shapley values}
\begin{split}
&\varphi_i\big[v^\ME\big]:=
\sum_{S\subseteq N\setminus \{i\}}\frac{|S|!\,(|N|-|S|-1)!}{|N|!}
\left(v^\ME(S\cup\{i\};\mathbf{X},f)-v^\ME(S;\mathbf{X},f)\right),\\
&\varphi_i\big[v^\CE\big]:=
\sum_{S\subseteq N\setminus \{i\}}\frac{|S|!\,(|N|-|S|-1)!}{|N|!}
\left(v^\CE(S\cup\{i\};\mathbf{X},f)-v^\CE(S;\mathbf{X},f)\right).
\end{split}    
\end{equation}
We call these feature attributions marginal and conditional Shapley values respectively. The former can be described as \textit{true to the model} and are dependent on the structure of the model (thus better suited for explaining specific models), whereas the latter can be described as \textit{true to the data} and take the joint distribution of features into account (hence they are much harder to compute) \cite{2020arXiv200616234C}. 
More generally, one can apply any \textit{linear game value} to the games in \eqref{games} to obtain the corresponding \textit{marginal} and \textit{conditional} feature attributions. We will mainly consider game values of the form
\begin{equation}\label{linear general}
h_i[v]:=\sum_{S\subseteq N\setminus \{i\}}w(S;n,i)\left(v(S\cup\{i\})-v(S)\right)\quad
(v \text{ an } n\textit{-person game}, \text{ a function }2^N\rightarrow\Bbb{R})
\end{equation}
which are generalizations of the Shapley value, and satisfy the desirable \textit{null-player property}, i.e. they assign zero to players that do not contribute to any coalition. This results in the so-called \textit{missingness property} of the corresponding explainers \cite{lundberg2017unified}. 
We shall exploit this property in our treatment of tree ensembles 
to reduce the complexity of computing marginal feature attributions; compare with \cite{campbell2022exact}.\\
\indent 
Game-theoretic feature attributions can be investigated from different angles: 
the model $f$ in hand, the game chosen based on $f$ and the predictors, and the game value applied to it. 
Let us motivate this:
\begin{enumerate}[(i)]
\setlength\itemsep{0.3em}
\item For tree-based models, $f$ is piecewise constant, i.e. a \textit{simple} function in the measure-theoretic sense.
\item The conditional and marginal games from \eqref{games}, and hence $\varphi_i\big[v^\ME\big]$ and 
$\varphi_i\big[v^\CE\big]$, are dependent only on the input-output function $f$, not on how $f$ is implemented or on any internal parameters such as weights of a neural network or proportions of splits of data points in a decision tree. When $f$ is obtained from a tree ensemble, we shall prove that the marginal game  $v^\ME$, and thus feature attributions $h_i\big[v^\ME\big]$ obtained from a linear game value $h$, 
are also simple functions; see Theorem \ref{main theorem}. This is not the case for the conditional game $v^\CE$  as demonstrated in Example \ref{main example}. Furthermore, for tree-based models, there is  an  empirical game $S\mapsto v^\Tree(S;\mathcal{T})(\mathbf{x})$
(see Definition \ref{TreeSHAP definition}) where $\mathcal{T}$ is an ensemble trained on a dataset. 
This is the game whose Shapley values are outputs of the \textit{path-dependent} variant of the popular \textit{TreeSHAP algorithm} \cite{2018arXiv180203888L,lundberg2020local}.
In \Sec \ref{subsec:theorem}, we shall show that this game does not estimate either the marginal or the conditional game, even for very big datasets. Moreover, we observe that the feature attributions generated by TreeSHAP fail to 
satisfy the desirable property of \textit{implementation invariance} which is posed as an axiom in 
\cite{sundararajan2017axiomatic}. 
This is shown in \Sec \ref{subsec:example}, where for two topologically distinct decision tree regressors
that exhibit identical input-output functions and even very close impurity measures, the most contributing features in terms of TreeSHAP turn out to be different over a non-negligible subset of data (cf. Figure \ref{fig:failure1}).  
A similar example is presented for the “eject” variant of TreeSHAP (\cite{campbell2022exact}) in Appendix \ref{subappendix:eject}.
\item Game values other than Shapley appear in the literature too. 
One well-known example is the \textit{Banzhaf value} \cite{banzhaf1964weighted}. 
In particular, it is suggested that to bridge the gap between the marginal and conditional frameworks one can group the features based on a dependence measure, and then either consider a \textit{quotient game}, or utilize  a \textit{coalitional game value}, for example the \textit{Owen value} \cite{owen1977values} or the
\textit{two-step Shapley value} \cite{kamijo2009two}, that takes into account 
the provided partition  of features  \cite{2021arXiv210210878M,aas2021explaining,2021arXiv210612228J} 
(see Appendix \ref{subappendix:OneHot} for another natural application of coalitional values).
We advocate for the \textit{carrier-dependence} property (and its coalitional analog) that, along with the null-player property, allows us to reduce the dimensionality of the problem of computing marginal feature attributions. See Appendix \ref{appendix:game} for necessary background material from cooperative game theory, and Theorem \ref{classification} for a classification of game values with properties that we deem desirable when it comes to tree ensembles.   
\end{enumerate}

A key insight, alluded to above, is that  the number of distinct features on which a  tree from the ensemble splits is usually smaller than the total number of features; compare with experiments in \Sec \ref{subsec:experiments1}.\footnote{The number of features on which  a learner depends can also be limited through hyperparameter tuning e.g. 
the  
\texttt{max\_feature}
hyperparameter in scikit-learn's bagging module or the
\texttt{colsample\_bytree}
hyperparameter in XGBoost and LightGBM.} 
A feature $X_i$ absent from a tree is a \textit{null player} for the corresponding marginal game  (not valid for the conditional game; cf. Lemma \ref{dummy players}). 
So when a linear game value $h$ with null-player and carrier-dependence properties is used (e.g. Shapley, Banzhaf etc.),  features not appearing in a tree do not get a contribution from it; and for the rest, the computations only involve features that do come up. 
Therefore, the complexity $O\left(2^n\right)$ of computing the value $h_i$ for the game $v^\ME$ becomes 
$O\left(2^r\cdot |\mathcal{T}|\right)$ 
where $r<n$ is the maximum number of features relevant to
a tree, and $|\mathcal{T}|$ is the total number of trees in the ensemble $\mathcal{T}$.
The observation just made on the number of different features per tree is best demonstrated by 
\textit{Explainable Boosting Machine (EBM)} \cite{2019arXiv190909223N}, where  each tree is dependent on at most two features (see Example \ref{EBM}); and more importantly, by \textit{CatBoost} models \cite{2018arXiv181011363V} where the trees are 
\textit{oblivious} (symmetric). Indeed, the number of distinct features appearing in such a tree does not exceed the depth which is rarely larger than $15$.\footnote{The recommended range of the hyper-parameter 
\texttt{depth}
is $6$ to $10$ per CatBoost documentation \url{https://catboost.ai/en/docs/concepts/parameter-tuning} \cite{CatBoostdoc}.}\\
\indent 
As a matter of fact, oblivious trees play a central role in this article:
\begin{itemize}
\setlength\itemsep{0.3em}
\item [$\blacktriangleright$] 
In \Sec \ref{subsec:CatBoost}, we simplify the Shapley formula to obtain a formula of reduced complexity for the marginal Shapley values of a symmetric decision tree (or an ensemble of such objects)  which is solely in terms of the internal parameters of the model and does not require any access to the training data; see Theorem \ref{Catboost theorem}.
\item [$\blacktriangleright$] We generalize the aforementioned explicit formula from the Shapley value to any other game value apt for explaining tree ensembles according to Theorem \ref{classification} (e.g. the Banzhaf value). It can also be generalized for a large family of coalitional game values (of which the Owen value is an example). See Appendix \ref{appendix:generalization} for more details.
\item [$\blacktriangleright$] Our explicit formulas for marginal feature attributions of ensembles of oblivious trees can be thought of as an ``intrinsic interpretability method'' 
(see \cite{kamath2021explainable} for a taxonomy of Explainable AI) due to their reliance on internal model parameters rather than on any background dataset. The benefit of symmetry to interpretability showcased by such formulas is reminiscent of a general philosophy that intrinsic interpretability should be induced from certain constraints \cite{yang2020enhancing}. 
Examples of this approach include working with ReLU networks or with  neural network architectures which implement functions of a special form \cite{alvarez2018towards,2018arXiv180601933V,2020arXiv201104041S,yang2021gami,2021arXiv210508589Z},
or papers \cite{2019arXiv190909223N,2022arXiv220111931S} on tree ensembles where the number of features per tree or the total number of splits across the ensemble are restricted. Nevertheless, unlike those articles, we work with game-theoretic local feature attributions. Furthermore, in our case, it is well documented that 
tree ensembles can have competitive predictive power even with the symmetry constraint \cite{2016arXiv160905610F,hancock2020catboost}. 
\item [$\blacktriangleright$] Estimating marginal Shapley values based on a background dataset, i.e. considering the \textit{empirical marginal game} (cf. \eqref{empirical}), can be subtle: The background dataset should be large for the sake of statistical accuracy because the 
mean squared error of such estimators is typically inversely proportional to the size of the background dataset (see Lemma \ref{error}); but the complexity increases with the size of the background dataset---this last point is especially manifested for the \textit{interventional} TreeSHAP algorithm \cite{lundberg2020local}. Building on Theorem \ref{Catboost theorem},
we present Algorithm \ref{algorithm} for oblivious ensembles that alleviates this problem of reliance on a background dataset, and computes the exact Shapley values of the empirical marginal game based on the whole training set. We carry out a rigorous error analysis for this algorithm in Theorem \ref{error analysis}.
Table \ref{Tab: complexity} below summarizes all of these facts by comparing our method with TreeSHAP in terms of accuracy and complexity.
\end{itemize}
\begin{table}[ht]
\small
\setlength\extrarowheight{2pt}
\begin{tabular}{|c|c|c|c|}
\hline
\shortstack{Algorithm\\\phantom{a}} & \shortstack{\\Precomputation complexity\\ (per leaf)} & 
\shortstack{\\Computation complexity\\ (for an input explicand)} 
& \shortstack{\\Variance of error \\ is governed by}\\
\Xhline{2pt}
Path-dependent TreeSHAP \cite{2018arXiv180203888L} & N/A & 
$O\left(|\mathcal{T}|\cdot \mathcal{L} \cdot \log^2(\mathcal{L})\right)$ & N/A\\[1mm]
\hline 
Interventional TreeSHAP \cite{lundberg2020local} & N/A &
$O(|\mathcal{T}|\cdot\mathcal{L} \cdot |D_*|)$ & $\frac{1}{|D_*|}$\\[1mm]
\hline 
Algorithm \ref{algorithm} & 
$O\left(|\mathcal{T}|\cdot \mathcal{L}^{\log_2 3}\cdot\log(\mathcal{L})\right)$ &
$O(|\mathcal{T}|\cdot\log(\mathcal{L}))$ & $\frac{1}{|D|}$\\[1mm]
\hline
\end{tabular}
\normalsize
\caption{
The complexity of various explanation algorithms are compared for an ensemble $\mathcal{T}$ of oblivious decision trees where each tree has at most $\mathcal{L}$ leaves. 
For an oblivious decision tree, the path-dependent TreeSHAP and the marginal Shapley values are the same for data points ending up at the same leaf; cf. Theorem \ref{main theorem}.
Algorithm \ref{algorithm} first precomputes all marginal Shapley values for all leaves of all trees and then saves them as look-up tables. The total time complexity is $O\left(|\mathcal{T}|\cdot \mathcal{L}^{\log_2 6}\cdot\log(\mathcal{L})\right)$, and the memory required to store them is 
$O(|\mathcal{T}|\cdot \mathcal{L}\cdot\log(\mathcal{L}))$. As long as the model is in production, the saved tables can be used to estimate marginal Shapley for any individual in time $O(|\mathcal{T}|\cdot\log(\mathcal{L}))$. On the other hand, variants of TreeSHAP do not build look-up tables and their time complexity for one individual (i.e. one leaf of each tree) are reflected above. As for the accuracy in estimating marginal Shapley values, the path-dependent TreeSHAP in general does not converge to marginal Shapley values (cf. Example \ref{main example}) while the interventional TreeSHAP has small error only for large background datasets $D_*$. In contrast, Algorithm \ref{algorithm} does not require any background dataset and instead utilizes model parameters such as leaf weights which are based on
the training set $D$ (which is typically very large).}
\label{Tab: complexity}
\end{table}
\indent
It must be pointed out that results of \Sec \ref{subsec:CatBoost}
can yield an analytic formula for the marginal Shapley values of any decision tree, even if they are not symmetric (e.g. those constructed by \textit{LightGBM} \cite{ke2017lightgbm} or by \textit{XGBoost} \cite{chen2016xgboost}): 
from any given decision tree $T$,
one can construct an oblivious decision tree ${\rm{obl}}(T)$ computing the exact same function (cf. Figure \ref{fig:construction}) to which 
Theorem \ref{Catboost theorem} can then be applied to obtain an explicit formula for the marginal Shapley values. Nevertheless, in the absence of symmetry, ${\rm{obl}}(T)$ determines a finer partition of the domain whose associated probabilities cannot always be estimated based on the trained tree $T$. 
These probabilities appear in the formula, and can in principle be precomputed using a background dataset. We elaborate more on the case of non-oblivious trees in Appendix \ref{appendix:general tree}.
Finally, we point out that the rectangularity of regions cut by decision trees is crucial to our results. In fact,  it is observed in Appendix \ref{appendix:ReLU} that, in case of ReLU networks,  marginal Shapley values are much more complicated as piecewise functions with respect to the activation regions.

\subsection{Outline}\label{subsec:outline}
\Sec \ref{sec:background} is devoted to the necessary background material
including  a very brief review of machine learning explainability and a short discussion on different boosting libraries along with the TreeSHAP method for interpreting them.  
We present our main results in \Sec \ref{sec:main}: 
In \Sec \ref{subsec:theorem}, we show that for tree-based models feature attributions arising from either TreeSHAP or the marginal game are simple functions. However, the former can depend on the model's make-up, and hence are not implementation invariant; see \Sec \ref{subsec:example}.
Next, focusing on marginal feature attributions, in \Sec \ref{subsec:observations} we observe that, for computing the marginal contribution of a feature, 
only the subset of trees which split on that feature are relevant; and each of those trees often depends only on a portion of the variables. This observation, showcased through experiments with XGBoost, LightGBM  and CatBoost models in \Sec \ref{subsec:experiments1}, can be utilized to reduce the complexity of computing marginal feature attributions for tree ensembles. The most important example is the case of CatBoost models where, leveraging the symmetry of oblivious trees, in \Sec \ref{subsec:CatBoost}, we obtain an explicit formula for marginal Shapley values as simple functions, and we propose an algorithm based on that.
The relevant experiments appear in Sections \ref{subsec:experiments2} and \ref{subsec:experiments3}. 
See Appendix \ref{appendix:code and data} for code and data availability.


\section{Preliminaries}\label{sec:background}
\subsection{Basic conventions and notation}\label{subsec:convention}
\begin{itemize}
\item For random variables, upper-case letters are used; and vector quantities are written in bold font. 
\item In this article, $N$ always is a finite non-vacuous subset of positive integers whose cardinality is denoted by $n$. Except in the appendices,  $N$ is taken to $\{1,\dots,n\}$ unless stated otherwise.
\item For an $n$-dimensional vector and a subset $S\subseteq N$ of indices, we use $S$ as a subscript to show the vector formed by components whose indices come from $S$. For instance, if $\mathbf{a}=(a_1,\dots,a_n)$, then $\mathbf{a}_S:=(a_i)_{i\in S}$. Moreover, $(a_i)_{i\in N\setminus S}$ is denoted by $\mathbf{a}_{-S}$; thus one may write $\mathbf{a}$ as $(\mathbf{a}_S,\mathbf{a}_{-S})$.
\item In modeling problems, the features are denoted by a vector $\mathbf{X}=(X_1,\dots,X_n)$ of random variables on an ambient probability space $(\Omega,\mathcal{F},\Bbb{P})$. The joint probability distribution ${\rm{P}}_{\mathbf{X}}$ is the
Borel probability measure on $\Bbb{R}^n$ obtained from the  pushforward of the probability measure $\Bbb{P}$ on $\Omega$. 
The model is thought of as a Borel measurable function $f:\Bbb{R}^n\rightarrow\Bbb{R}$.
We use $D\subset\Bbb{R}^n$ to denote a finite data sample. Occasionally, a smaller background dataset may be required for estimating marginal expectations; that will be denoted by $D_{*}\subset D$.
A random sample drawn i.i.d. from  $\Bbb{R}^n$ according to the distribution ${\rm{P}}_{\mathbf{X}}$ is shown by $\mathbf{D}$ whose elements are random vectors 
$$\mathcal{X}^{(1)},\dots,\mathcal{X}^{(\mathscr{D})}:(\Omega,\mathcal{F})\rightarrow(\Bbb{R}^n,\text{Borels}).$$
These are i.i.d. and for each of them the induced measure on $\Bbb{R}^n$ is ${\rm{P}}_{\mathbf{X}}$. 
\item We assume that the features have numeric value; so they can be continuous, ordinal or encoded categorical features. (See Appendix \ref{subappendix:OneHot} for more on the case of categorical features.)
We assume that each $X_i$ takes its values in an interval $[l_i,u_i]$. Thus ${\rm{P}}_{\mathbf{X}}$ is supported in the  hypercube 
\begin{equation}\label{support}
\mathcal{B}:=\prod_{i=1}^n[l_i,u_i].
\end{equation}

\end{itemize}

\subsection{A review of machine learning explainability}\label{subsec:XAI}
There is a vast literature on explaining complicated machine learning models;
see \cite{molnar2020interpretable,kamath2021explainable} for an overview. 
There are \textit{global} methods such as  PDP (Partial Dependence Plots) \cite{friedman2001greedy} 
or BETA (Black Box Explanations through Transparent Approximations) \cite{2017arXiv170701154L}
which describe the overall effect of features as well as \textit{local} methods such as 
the rule-based method Anchors \cite{ribeiro2018anchors}, or 
LIME (Linear Interpretable Model-agnostic Explanation) \cite{ribeiro2016should} 
and SHAP (SHapley Additive exPlanations) \cite{lundberg2017unified}
which provide \textit{individualized} feature attributions to explain a single prediction. 
The SHAP paper builds on ideas from game theory \cite{shapley1953value} 
(also see \cite{vstrumbelj2014explaining}).
Moreover, it introduces the KernelSHAP algorithm for approximating Shapley values. 
For a survey on different methods for estimating Shapley values,
see \cite{2022arXiv220707605C}. \\
\indent 
This paper focuses on game-theoretic local feature attributions. 
All the aforementioned methods are \textit{model agnostic}. 
There are also \textit{model-specific} methods for estimating Shapley values including  
the DeepSHAP algorithm for neural networks \cite{chen2021explaining}, and the TreeSHAP algorithm 
\cite{2018arXiv180203888L,lundberg2020local} for tree ensembles. 
The focus of this paper is on tree-based models. 
After a brief review of TreeSHAP in \Sec \ref{subsec:TreeSHAP}, 
we compare its outputs with marginal feature attributions in \Sec \ref{sec:main} where we expose certain shortcomings of TreeSHAP,  and discuss how calculating marginal feature attributions for tree-based models can be done more efficiently. 

\subsection{Feature attributions via cooperative game theory}\label{subsec:gamebackground}
Machine learning explainers studied in this paper are constructed via game-theoretic methods.  
The first step is to define certain pointwise games based on the predictors $\mathbf{X}=(X_1,\dots,X_n)$ and the model $f$. At each point $\mathbf{x}\in\Bbb{R}^n$, these define a game with $N=\{1,\dots,n\}$ as its set of players, i.e. a set function $v:2^N\rightarrow\Bbb{R}$. Next, to compute feature attributions for the individual $\mathbf{x}$, a game value $h$ is applied to 
obtain a vector $\left(h_i[N,v]\right)_{i\in N}$ where the $i^{\rm{th}}$ component quantifies the ``contribution'' of player $i\in N$ (feature $X_i$) according to $h$. We refer the reader to Appendix \ref{appendix:game} for basic notions from cooperative game theory. 
\\ 
\indent
Two prominent games associated with a machine learning model $f$ are the marginal game $v^\ME=v^\ME(\cdot;\mathbf{X},f)$ and the conditional game $v^\CE=v^\CE(\cdot;\mathbf{X},f)$ defined in \eqref{games}.
Feature attributions obtained from them via applying a game value such as Shapley are characterized as true to the model and true to the data respectively; this is best demonstrated in Example \ref{regression}.
In general, the choice between the two approaches depends on the application \cite{2020arXiv200616234C}. 
See \cite{2021arXiv210210878M} for a detailed comparison of marginal and conditional feature attributions. \\
\indent
It is certainly possible to define other games based on the model $f$; see \cite{datta2016algorithmic,merrick2020explanation,sundararajan2020many}. Another game particularly important to our context is the game $v^\Tree(\cdot;\mathcal{T})$ that  the path-dependent TreeSHAP algorithm (\cite{2018arXiv180203888L}) introduces when 
$f$ is computed by a trained tree ensemble $\mathcal{T}$; cf. \Sec \ref{subsec:TreeSHAP}. 
The resulting feature attributions turn out to be very different from either marginal or conditional ones as we shall observe in \Sec \ref{subsec:example}.\\
\indent
It should be mentioned that in our treatment conditional and marginal games stem from features that are random variables defined on an ambient probability space. There is also the empirical marginal game which is an estimator for the marginal game based on the training set $D$ (or based on any other data sample):
\begin{equation}\label{empirical}
\hat{v}^{\ME}(S;D,f)(\mathbf{x})
:=\frac{1}{|D|}\sum_{\tilde{\mathbf{x}}\in D} f(\mathbf{x}_S,\tilde{\mathbf{x}}_{-S}).    
\end{equation}

An important observation, which is useful for reducing the computational complexity of generating marginal feature attributions, is that for the marginal game only variables that do appear in the model matter. This is not the case for the conditional game. 
\begin{lemma}\label{dummy players}
Consider $(\mathbf{X},f)$ where $\mathbf{X}=(X_1,\dots,X_n)$ are the predictors and $f:\Bbb{R}^n\rightarrow\Bbb{R}$ is a function.
\begin{enumerate}
\item If $U\subset N$ and $f(\mathbf{x})$ is independent of $x_i$ for $i\in N\setminus U$, 
then for ${\rm{P}}_{\mathbf{X}}$-almost every $\mathbf{x}$, 
each $i\in N\setminus U$ is a null player of the game 
$v^\ME(\cdot;\mathbf{X},f)(\mathbf{x})$ and $U$ is a carrier for it. Moreover, 
$v^\ME(\cdot;\mathbf{X},f)$ can be considered as a function of lower arity: 
$v^\ME(\cdot;\mathbf{X},f)(\mathbf{x})=v^\ME(\cdot;\mathbf{X}_U,\tilde{f})(\mathbf{x}_U)$ where 
$\tilde{f}$ is defined by $f(\mathbf{x})=\tilde{f}(\mathbf{x}_U)$. 
\item Suppose $U\subset N$ and $f(\mathbf{x})$ is independent of $x_i$ for $i\in N\setminus U$. 
If $X_i$ is independent of $\mathbf{X}_{-i}$, then $i\in N$ is a null player of the game $v^\CE(\cdot;\mathbf{X},f)(\mathbf{x})$ for ${\rm{P}}_{\mathbf{X}}$-almost every $\mathbf{x}$. Furthermore, if $\mathbf{X}_{U}$ is independent of $\mathbf{X}_{-U}$, then
$U$ is a carrier of $v^\CE(\cdot;\mathbf{X},f)(\mathbf{x})$ for ${\rm{P}}_{\mathbf{X}}$-almost every $\mathbf{x}$. 
\end{enumerate}
\end{lemma}
We refer the reader to Appendix \ref{subappendix:properties} for 
basic concepts from game theory such as null player or \textit{carrier}.  
The lemma above is proven in Appendix \ref{subappendix:games}.
Notice that, here, in the case of the marginal game, one can focus on games with a smaller set of players $U$ instead of $N$ whereas additional assumptions on predictors' joint distribution are required when it comes to the conditional game. 

\begin{remark}
The fact that a variable not appearing in the model can have non-zero conditional feature attribution was previously observed in \cite{janzing2020feature,sundararajan2020many}. This violates an axiom posed for feature attribution methods in \cite{sundararajan2017axiomatic,sundararajan2020many}.     
\end{remark}

Next, we discuss how certain properties of game values can be desirable in constructing machine learning explainers.   
Various properties of game values have been studied extensively in the game theory literature. 
The seminal paper of Shapley \cite{shapley1953value} shows that the Shapley value 
\begin{equation}\label{ShapleyFormula}
\varphi_i[N,v]:=\sum_{S\subseteq N\setminus \{i\}}
\frac{|S|!\,(|N|-|S|-1)!}{|N|!}
\left(v(S\cup\{i\})-v(S)\right) \quad(i\in N)
\end{equation}
is the unique game value with  \textit{linearity}, \textit{symmetry}, \textit{efficiency} and 
null-player properties.\footnote{In fact, Shapley's theorem holds with the weaker property of \textit{additivity} in place of linearity. Moreover, his paper combines efficiency and null-player properties into a single formula.}
Another characterization of the Shapley value is due to Young;  it drops additivity and replaces the null-player property with \textit{strong monotonicity} \cite{young1985monotonic}. Formal definitions of these properties can be found in Appendix \ref{subappendix:properties}. 
For our purposes, as mentioned in \Sec \ref{subsec:motive}, the null-player property is crucial. 
It turns out that linear game values with the null-player property are precisely those of form \eqref{linear general}. Indeed, it is possible to determine when game values of this form satisfy other desirable properties just mentioned; this is the content of Lemma \ref{properties}. 
The lemma allows us to deal with a formula of the form \eqref{linear general} (or \eqref{linear general variant}) rather than an abstract assignment $(N,v)\mapsto\left(h_i[N,v]\right)_{i\in N}$, a formula in which games associated with a machine learning model $(\mathbf{X},f)$ (e.g. \eqref{games}) can be plugged. Notice that the Shapley value \eqref{ShapleyFormula} is of the form \eqref{linear general}. The same holds for the famous Banzhaf value (cf. \cite{banzhaf1964weighted}) 
\begin{equation}\label{BanzhafFormula}
Bz_i[N,v]:=\sum_{S\subseteq N\setminus \{i\}}
\frac{1}{2^{n-1}}
\left(v(S\cup\{i\})-v(S)\right) \quad (i\in N).
\end{equation}
Lemma \ref{properties} immediately implies that both Shapley and Banzhaf game values are linear, symmetric, strongly monotonic, and satisfy the null-player property.\footnote{Nonetheless, among them, only the Shapley value satisfies the efficiency property due to the main result of \cite{shapley1953value}.} 
Moreover, they satisfy the carrier dependence property; see Lemma \ref{carrier}.\footnote{For the carrier dependence of the Shapley value, also see  \cite[Corollary 2]{shapley1953value}.}
Notice that, due to the symmetry, the weights in \eqref{ShapleyFormula} and \eqref{BanzhafFormula} depend only on the $|N|$ and $|S|$. Thus the formulas make sense even when the set of players $N$ is an arbitrary finite subset of $\Bbb{N}$ in bijection with $\{1,\dots,n\}$. 
When $N=\{1,\dots,n\}$ and the context is clear, we omit $N$ and write these values as 
$\varphi_i[v]$ or $Bz_i[v]$. 

\begin{example}\label{two-player}
For a game $v$ with the set of players $\{1,2\}$ one has 
\small 
\begin{equation}\label{two-player Shapley}
\varphi_1[v]=
\frac{1}{2}\left(v(\{1,2\})-v(\{2\})\right)+\frac{1}{2}\left(v(\{1\})-v(\varnothing)\right),\quad
\varphi_2[v]=
\frac{1}{2}\left(v(\{1,2\})-v(\{1\})\right)+\frac{1}{2}\left(v(\{2\})-v(\varnothing)\right).
\end{equation}
\normalsize
\end{example}
Different properties of game values mentioned so far have come up in the context of machine learning explainability. Paper \cite{lundberg2017unified} puts forward the  SHAP framework for individualized feature attribution which is argued to be the unique \textit{additive} method satisfying missingness and \textit{consistency}. These properties of the SHAP method  follow respectively from efficiency, null-player and strong monotonicity properties of the Shapley game value. Paper \cite{sundararajan2020many} discusses why different properties of the Shapley value are desirable and construct explainers by applying the Shapley value to a variety of games constructed based on the model in hand. In our setting of marginal explanations for tree ensembles, as mentioned in \Sec \ref{subsec:motive}, the key insight, best demonstrated by libraries such as EBM and CatBoost, is that the number of distinct features appearing in a tree can  be much smaller than the total number of variables. 
In view of this, we postulate that game values used for explaining tree ensembles should 
admit the following properties (see Definition \ref{terminology}):
\begin{enumerate}
\item linearity, so that feature attributions can be disaggregated across the ensemble;
\item symmetry, a natural property which is often assumed in the literature;
\item null-player, so that features on which a tree does not split get zero attributions from that tree;
\item carrier-dependence, so that for each tree the problem reduces to one only involving the features on which the tree actually splits. 
\end{enumerate}
The advantage of the null-player property has been pointed out in \cite{campbell2022exact} as well. But in this article, we formulate a unified framework by considering game values satisfying the four axioms above which include game values other than Shapley, e.g. Banzhaf. 
The theorem below completely classifies such game values.
\begin{theorem}\label{classification}
Let  $\mathcal{A}:=\{\alpha(s,n)\}_{\substack{n\in\Bbb{N}\\ 0\leq s<n}}$ 
be a collection of real numbers with a ``backward'' Pascal identity:
\begin{equation}\label{backward Pascal}
\alpha(s,n)+\alpha(s+1,n)=\alpha(s,n-1). 
\end{equation}
Define a game value $h^{\mathcal{A}}$ by setting 
\begin{equation}\label{good game value}
h^{\mathcal{A}}_i[N,v]:=\sum_{S\subseteq N\setminus\{i\}}\alpha(|S|,|N|)\left(v(S\cup\{i\})-v(S)\right)   
\end{equation}
for any cooperative game $(N,v)$ and any $i\in N$. Then $h^{\mathcal{A}}$ satisfies linearity, symmetry, null-player and carrier-dependence axioms. Conversely, a game value $h$ satisfying these four axioms is of the form $h^{\mathcal{A}}$ for such a collection $\mathcal{A}$. 
\end{theorem}
\noindent
A proof will be presented in Appendix \ref{subappendix:classification}. Moreover, such game values have a computational advantage because certain sums involving their coefficients may be simplified; this is the content of Lemma \ref{simplification}.

\begin{example}
In the case of the Shapley value one has $\alpha(s,n)=\frac{s!(n-s-1)!}{n!}$ while 
$\alpha(s,n)=\frac{1}{2^{n-1}}$ in the case of the Banzhaf value. It can be readily checked that the backward Pascal identity \eqref{backward Pascal} holds in both situations.    
\end{example}

At the end of this section, we allude to \textit{coalitional explainers}. On the game-theoretic side, they amounts to adding a coalition structure (cf. \cite{aumann1974cooperative}) and then applying  a coalitional game value; i.e. an assignment 
$\mathfrak{h}:(N,v,\mathfrak{P})\mapsto\left(\mathfrak{h}_i[N,v,\mathfrak{P}]\right)_{i\in N}$
where  $\mathfrak{P}$ is a partition of the set of players $N$. 
This pertains to machine learning explanation because it has been observed that grouping predictors can improve the stability of feature attributions and facilitate computations \cite{2021arXiv210210878M,aas2021explaining,2021arXiv210612228J}. 
One well-known example of a coalitional game value is the Owen value introduced in \cite{owen1977values}:
\begin{equation}\label{OwenFormula}
\begin{split}
&Ow_i[N,v,\mathfrak{P}]:=
\sum_{R\subseteq M\setminus \{j\}}
\sum_{K\subseteq S_j\setminus\{i\}}
\frac{|R|!\,(|M|-|R|-1)!}{|M|!}\cdot
\frac{|K|!\,(|S_j|-|K|-1)!}{|S_j|!}
\left(v\left(Q\cup K\cup\{i\}\right)-
v\left(Q\cup K\right)\right),\\
& \text{where } \mathfrak{P}=\{S_1,\dots,S_m\} \text{ is a partition of } N,
M:=\{1,\dots,m\}, i\in S_j \text{ and } Q:=\cup_{r\in R}S_r.    
\end{split}
\end{equation}
We shall discuss coalitional game values in more details in Appendix \ref{subappendix:coalitional game values}. Many of our results can be generalized from game values such as Shapley or Banzhaf to 
coalitional game values such as Owen. 
In particular, a coalitional version of Theorem \ref{classification} is presented in Appendix \ref{subappendix:classification coalitional}.
Furthermore, we shall see in Appendix \ref{subappendix:OneHot} that the Owen game value can be  used naturally to retrieve marginal Shapley values of categorical features which were one-hot encoded in the modeling process. 
For more on group explainers, including a detailed treatment of various axioms for coalitional game values and their implications, see \cite{2021arXiv210210878M}.


\subsection{Tree ensembles}\label{subsec:ensembles}
We start with some basic notation and terminology. 
\begin{itemize}
\item We are mainly concerned with tree-based regressors. Such a model $f$ corresponds to an ensemble $\mathcal{T}=\{T_1,\dots,T_t\}$ where $T_i$'s are regression decision trees. 
We always assume that the ensembles are trained, so $\mathcal{T}$ provides us with the knowledge of internal parameters such as values at the leaves and the splitting proportions of training instances at the internal nodes of trees. 
Denoting the function implemented by $T_i$ as $g_i$, one has $f=g_1+\dots+g_t$.
For classifiers, outputs of $g_i$'s become logit probability values and $f$ becomes the decision function (the population minimizer).\footnote{The function $f$ can be accessed by setting 
\texttt{prediction\_type="RawFormulaVal"} in CatBoost and 
\texttt{raw\_score=True} in LightGBM. Moreover, $f$ sometimes is an affine transformation of $g_1+\dots+g_t$.}  
\item In each decision tree from the ensemble, splits at non-terminal nodes
are based on whether a feature $X_i$ is smaller than a threshold or not. To avoid ambiguity about if $X_i<threshold$ should be strict or not, we assume all events $X_i=threshold$ are of probability zero.\footnote{Ordinal features fit in this framework too. For instance, if $X_i$ takes its values in $\{1,2,3,4\}$, one can take the threshold to be a non-integer from $(1,4)$.}
In other words:
\begin{equation}\label{assumption}
{\rm{P}}_{\mathbf{X}}(x_i=threshold)
=0 
\quad \begin{matrix}
\text{ if a splitting based on comparing } X_i \text{ with }\\
threshold\text{  takes place in a tree from the ensemble}.
\end{matrix}
\end{equation}
\item For each $i\in N$, the subset of trees in $\mathcal{T}$ that split on $X_i$ is denoted by 
$\mathcal{T}^{(i)}$.
\item Each decision tree $T$ from ensemble $\mathcal{T}$ computes a simple function 
\begin{equation}\label{simple function}
g(\mathbf{x}):=c_1\cdot\mathbbm{1}_{R_1}+\dots+c_\ell\cdot\mathbbm{1}_{R_\ell}      
\end{equation}
where $c_1,\dots,c_\ell$ are the values appearing at the leaves of $T$ and $R_1,\dots,R_\ell$ are the rectangular regions with disjoint interiors determined by it; here $R_i$ is the region where the tree assigns the value $c_i$ and its characteristic function is denoted by $\mathbbm{1}_{R_i}$. 
Therefore, the ambient hypercube $\mathcal{B}$ \eqref{support} is cut into smaller ones 
$R_1,\dots,R_\ell$ which are determined by the splits in the decision tree. We take these hypercubes to be closed. They then may intersect each other along the boundaries, but  that is negligible in view of \eqref{assumption}. 
Thus
\begin{equation}\label{partition}
\mathscr{P}(T):=\{R_1,\dots,R_\ell\}     
\end{equation}
is a partition of $\mathcal{B}$ into smaller hypercubes, at least in the measure-theoretic sense.\footnote{I.e. the union of $R_1,\dots,R_\ell$ covers $\mathcal{B}$ except for perhaps a measure zero subset, and the intersection of any two of them is of measure zero.}

\item Any partition $\mathscr{P}$ of $\mathcal{B}$ into rectangular regions may be completed into a grid $\widetilde{\mathscr{P}}$ which is the product of partitions determined by $\mathscr{P}$ across different dimensions. Figure \ref{fig:partition} illustrates a decision tree $T$ computing a simple function of two variables along with the corresponding partition  $\mathscr{P}(T)$ and the finer one $\widetilde{\mathscr{P}(T)}$.  
\end{itemize}

Various implementations of gradient boosting differ in terms of the training time, their hyperparameters, their optimization and regularization techniques, their approach to construct trees, their handling of categorical features or missing values, parallel processing etc.
For instance, the implementation of gradient boosting in Scikit-learn uses gradient descent for minimizing the cost function while XGBoost utilizes the Newton method; or CatBoost has a sophisticated way of handling categorical features which XGBoost lacks. 
Regardless  of implementation differences, when the base learner is a decision tree, ensembles obtained from bagging or boosting represent simple functions. The relevant point here is how trees constructed by different libraries are different topologically. The reader can check detailed comparisons of XGBoost, LightGBM and CatBoost methods in \cite{neptune1,neptune2,neptune3,al2019comparison}. Focusing on their (default) growth policy:
\begin{itemize}
\item Trees in XGBoost are constructed level-wise; they are grown to 
\texttt{max\_depth} and are pruned based on hyperparameters such as 
\texttt{min\_split\_loss} \cite{XGBoostdoc}. The splits within a level are not necessarily the same. 
\item Trees in LightGBM are constructed in a leaf-wise manner; tree complexity is governed by hyperparameters such as \texttt{max\_depth} and
\texttt{num\_leaves} \cite{LightGBMdoc}. This can result in asymmetric trees. 
\item Trees in CatBoost are symmetric or oblivious; this restriction can be thought of as a regularization helping to avoid overfitting   \cite{2018arXiv181011363V}. 
See below for a definition of oblivious decision trees. 
\end{itemize}

\begin{definition}\label{oblivious definition}
An oblivious (symmetric) decision tree is a perfect binary tree (i.e. $\#\text{leaves}=2^{\text{depth}}$) 
in which the splits across each level are done with respect to the same feature and threshold.
\end{definition}

\begin{example}
The second tree from Figure \ref{fig:failure1} and the tree in Figure \ref{fig:repeated} are oblivious whereas the first tree from Figure \ref{fig:failure1} and the tree in Figure \ref{fig:partition} are not. 
\end{example}

\subsection{TreeSHAP algorithm and its variants}\label{subsec:TreeSHAP}

Although there are numerous algorithms for computing global feature importance values for a tree-based model, local feature attribution methods are not fully investigated \cite[Supplementary Results, \Sec 2]{lundberg2020local}.
The TreeSHAP algorithm is one of the most common local methods for interpreting tree-based models.  
TreeSHAP is not model agnostic and takes internal parameters such as values at the leaves and the splitting proportions at internal nodes into account. The algorithm has two variants which both use dynamic programming to obtain polynomial-time performance.
The original TreeSHAP algorithm \cite{2018arXiv180203888L},  the path-dependent variant,
is meant to estimate conditional Shapley values; but its approach to approximating conditional expectations turns out to be imperfect \cite[p. 18]{2022arXiv220707605C}. 
Indeed, the Shapley values generated by the path-dependent TreeSHAP come from a certain game associated with the trained ensemble \cite[Algorithm 1]{2018arXiv180203888L}. This game is presented in Definition \ref{TreeSHAP definition} below. In \Sec \ref{sec:main}, we shall observe that this game differs from both conditional and marginal games, and as a matter of fact, it depends on the tree structure, not just the input-output function of the model.  
In contrast, the other variant, the interventional TreeSHAP,
estimates marginal Shapley values through utilizing a background dataset \cite{lundberg2020local}; it thus depends only on the input-output function of the model. 
The downside is that a fair number of background samples is required for the accuracy of the estimation while  the algorithm becomes slow for even moderately large background datasets (see the complexity analysis below). Indeed, experiments in the aforementioned paper use only $200$ background samples 
\cite[p. 66]{lundberg2020local}. 
According to TreeSHAP documentation \cite{TreeSHAPdoc}, the recommended size for the background dataset is between $100$ and $1000$. 
In \Sec \ref{subsec:CatBoost}, we resolve this issue in the case of oblivious ensembles by presenting a method for estimating marginal Shapley values which does not require any background dataset and instead, employs 
internal parameters; the accuracy of our method is dictated by the size of the training dataset.\\
\indent
In implementing TreeSHAP, the variant is determined by the hyperparameter 
\texttt{feature\_perturbation}
which should be 
\texttt{"tree\_path\_dependent"}
for path-dependent, and 
\texttt{"interventional"}
for the interventional variant---which is currently the default \cite{TreeSHAPdoc}. 
For an ensemble $\mathcal{T}$, the time complexity of the path-dependent TreeSHAP is 
\begin{equation}\label{path complexity}
O(|\mathcal{T}|\cdot \mathcal{L} \cdot \mathcal{D}^2)    
\quad (\mathcal{L} \,(\text{resp. } \mathcal{D}):=\text{max.  number of leaves (resp. max. depth) of any tree from }\mathcal{T}),
\end{equation}
while that of the interventional one is 
\begin{equation}\label{interventional complexity}
O(|\mathcal{T}|\cdot \mathcal{L} \cdot |D_*|)
\quad (\mathcal{L} \text{ as above and }D_* \text{ the background dataset});
\end{equation}
see \cite[pp. 64--66]{lundberg2020local} for descriptions of these algorithms and their complexity analysis.
\begin{definition}\label{TreeSHAP definition}
Let $\mathcal{T}$ be a trained ensemble of  decision trees and $\mathbf{X}=(X_1,\dots,X_n)$ the features. 
Denoting the training set with the response values removed by $D$, following our convention in \Sec \ref{subsec:ensembles}, 
information such as leaf scores or members of $D$ that end up at a given node can be read off from $\mathcal{T}$.
We define the associated TreeSHAP game 
$v^\Tree(\cdot;\mathcal{T})$ as 
$v^\Tree(\cdot;\mathcal{T})
=\sum_{T\in\mathcal{T}}v^\Tree(\cdot;T)$
where, for decision trees, games $v^\Tree(\cdot;T)$ are defined recursively 
in the following manner. 
In case that $T$ has no splits (so $T$ is a single leaf), $v^\Tree(\cdot;T)$ assigns the value at the unique leaf of $T$ to every subset of $N=\{1,\dots,n\}$.   
Next, suppose the split at the root of $T$ takes place with respect  to feature  
$X_{i_*}$ and threshold $t_*$. Thus we have the left subtree $T^<$ and the right subtree $T^>$
along with smaller datasets 
$D^<:=\{\mathbf{x}\in D\mid x_{i_*}<t_*\}$
and 
$D^>:=\{\mathbf{x}\in D\mid x_{i_*}>t_*\}$. Then set
\begin{equation}\label{TreeSHAP game}
    v^\Tree(S;T)(\mathbf{x}):=
    \begin{cases}
    v^\Tree(S;T^>)(\mathbf{x}) &\text{if }i_*\in S \text{ and } x_{i_*}>t_*\\
    v^\Tree(S;T^<)(\mathbf{x}) &\text{if }i_*\in S \text{ and } x_{i_*}<t_*\\
    \frac{|D^>|}{|D|}\cdot v^\Tree(S;T^>)(\mathbf{x})
    +\frac{|D^<|}{|D|}\cdot v^\Tree(S;T^<)(\mathbf{x}) & \text{if }i_*\notin S
    \end{cases}
    \quad (\mathbf{x}\in\Bbb{R}^n, S\subseteq N).
\end{equation}
\end{definition}
\noindent
Note that 
$\{i\in N\mid X_i \text{ appears in }T\}$ 
is a carrier for each $v^\Tree(\cdot;T)$; and 
$$
v^\Tree(\varnothing;T)(\mathbf{x})=\frac{1}{|D|}\sum_{\mathbf{x}'\in D}g(\mathbf{x}'), \quad 
v^\Tree(N;T)(\mathbf{x})=g(\mathbf{x})
$$
where $g:\Bbb{R}^n\rightarrow\Bbb{R}$ is the function computed by the decision tree $T$, and 
$g(\mathbf{x}')$ is the leaf value at $\mathbf{x}'$.

It is not hard to show that the TreeSHAP game $v^\Tree$ becomes the empirical marginal game $\hat{v}^{\ME}$ (see \eqref{empirical}) if the predictors are independent. But the games are different in general; compare with  \cite{amoukou2022accurate}.\\
\indent
We finish the section by pointing out  the related algorithms which mostly build upon the original TreeSHAP method \cite{2018arXiv180203888L}. Paper \cite{2021arXiv210909847Y} introduces ``Fast TreeSHAP'' as an improvement of the path-dependent TreeSHAP while \cite{karczmarz2022improved} discusses an improvement when the Banzhaf value is used in place of Shapley. 
Another approach is the unpublished work of Saabas \cite{Saabasdoc} (cf. \cite[Supplementary Results, \Sec 3]{lundberg2020local}) where only features appearing along the decision path can get non-zero attributions.\footnote{Just like the path-dependent TreeSHAP, Saabas' method also fails the implementation invariance.}  
Finally, \cite{campbell2022exact} introduces a modified version of $S\mapsto v^\Tree(S;T)$ from \eqref{TreeSHAP game} where, instead of taking a weighted average, one ``ejects'' the decision tree $T$ if the split is done with respect to a feature absent from $S$. This game is presented in Definition \ref{Eject definition} of Appendix \ref{subappendix:eject}. The authors then write the Shapley formula for this new game and simplify it to reduce the complexity. Nevertheless, as shown in the appendix, this approach also lacks implementation invariance, just like the original path-dependent TreeSHAP.

\section{Main results}\label{sec:main}

\subsection{TreeSHAP is not implementation invariant}\label{subsec:example}
When two models are \textit{functionally equivalent}, namely, they generate equal outputs for the same input, 
it is reasonable to ask for their associated feature attributions to coincide. This is the implementation invariance axiom for attribution methods which is set forth in \cite{sundararajan2017axiomatic}. 
Indeed, if this property holds, one can treat feature attributions as well-defined operators on some appropriate space of models. This has been carried out for conditional and marginal feature attributions in \cite{2021arXiv210210878M} where the resulting operators are thoroughly studied via tools from functional analysis. The goal of this section is to show that, unlike conditional and marginal frameworks, the path-dependent TreeSHAP fails this axiom and can depend on the model's make-up.

\begin{example}\label{main example}
\begin{figure}
\includegraphics[width=11cm]{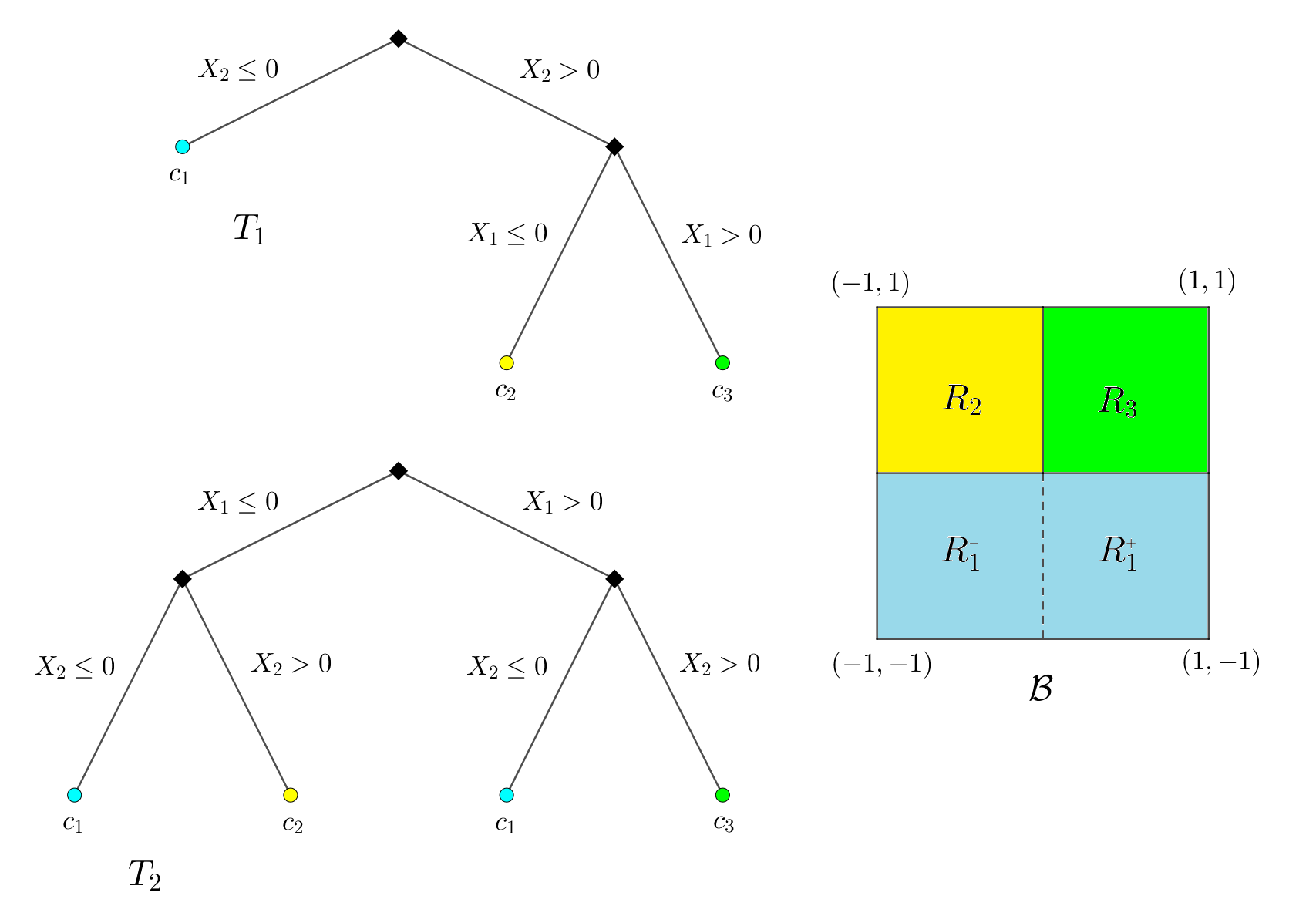}
\caption{The picture for Example \ref{main example} demonstrating that TreeSHAP (\cite{2018arXiv180203888L}) can depend on the model make-up. 
Here, the features $X_1$ and $X_2$ are supported in the rectangle $\mathcal{B}=[-1,1]\times[-1,1]$ on the right which is partitioned into subrectangles  $R_1=R_1^{\minus}\cup R_1^{\plus}$, $R_2$ and $R_3$. The decision trees $T_1$ and $T_2$ on the left compute the same function   $g=c_1\cdot\mathbbm{1}_{R_1}+c_2\cdot\mathbbm{1}_{R_2}+c_3\cdot\mathbbm{1}_{R_3}$; the leaves are colored based on the colors of corresponding subrectangles on the right. Shapley values for various games associated with these trees are computed in Example \ref{main example}. In particular, over each of the subrectangles, the Shapley values arising from the marginal game, or from TreeSHAP, are constant expressions in terms of probabilities of 
${\rm{P}}_\mathbf{X}(R_1^{\minus}),{\rm{P}}_\mathbf{X}(R_1^{\plus}),{\rm{P}}_\mathbf{X}(R_2),
{\rm{P}}_\mathbf{X}(R_3)$ and leaf values $c_1,c_2,c_3$; 
see Table \ref{Tab:main example table}. 
Although the former Shapley values depend only $g$, the latter turn out to be different for $T_1$ and $T_2$. In \eqref{parameters},  these parameters are chosen so that TreeSHAP ranks features $X_1$ and $X_2$ differently for any input from $R_2$, whereas the decision trees compute the same function and are almost indistinguishable in terms of impurity measures.}
\label{fig:failure1}
\end{figure}
Consider a simple regression problem with two predictors $\mathbf{X}=(X_1,X_2)$ 
that are supported in the square $\mathcal{B}=[-1,1]\times[-1,1]$, and the response variable $Y$ which takes value in $\{c_1,c_2,c_3\}$. 
We assume that $g(x)=\Bbb{E}[Y\mid \mathbf{X}=x]$ is captured by the function 
$g:=c_1\cdot\mathbbm{1}_{R_1}+c_2\cdot\mathbbm{1}_{R_2}+c_3\cdot\mathbbm{1}_{R_3}$ where 
$$
R_1:=[-1,1]\times[-1,0],\quad R_2:=[-1,0]\times[0,1], \quad R_3:=[0,1]\times[0,1]
$$
are subrectangles inside $\mathcal{B}$. 
Following our convention in \Sec \ref{subsec:ensembles}, we assume that events $X_1=0$ and $X_2=0$ are of probability zero.
Therefore, the pushforward probability measure ${\rm{P}}_\mathbf{X}$ on $\Bbb{R}^2$ is supported in $\mathcal{B}$; and $R_1$, $R_2$ and $R_3$ provide a measure-theoretic partition of $\mathcal{B}$.
The ground truth $g=c_1\cdot\mathbbm{1}_{R_1}+c_2\cdot\mathbbm{1}_{R_2}+c_3\cdot\mathbbm{1}_{R_3}$ can be captured via two different decision trees $T_1$ and $T_2$ where $T_1$ splits on $X_2$ at the root whereas the split at the root of $T_2$ is done with respect to $X_1$. Thus $T_2$ also partitions the bottom rectangle 
$R_1=[-1,1]\times[-1,0]$ into $R_1^{\minus}:=[-1,0]\times[-1,0]$ and $R_1^{\plus}:=[0,1]\times[-1,0]$.
These are all demonstrated in Figure \ref{fig:failure1}.\\
\indent
Next, we discuss the conditional, marginal and TreeSHAP feature attributions. 
To determine which of $X_1$ or $X_2$ is the most contributing feature at a point $\mathbf{x}=(x_1,x_2)$ of $\mathcal{B}$, the sign of the difference in Shapley values should be considered:
\begin{equation}\label{Delta}
\Delta\varphi[v](\mathbf{x}):=\varphi_1[v(\mathbf{x})]- \varphi_2[v(\mathbf{x})]
=\frac{1}{2}\big(v(\{1\})(\mathbf{x})-v(\{2\})(\mathbf{x})\big),   
\end{equation}
where the Shapley formula in the case of two players \eqref{two-player Shapley} was employed,
and
$$
v\in\left\{v^\CE(\cdot;\mathbf{X},g),v^\ME(\cdot;\mathbf{X},g),v^\Tree(\cdot;T_1),v^\Tree(\cdot;T_2)\right\}.
$$
For these choices of $v$, Table \ref{Tab:main example table} presents  
$2\Delta\varphi[v]=v(\{1\})-v(\{2\})$ 
as a piecewise function with respect to the partition of $\mathcal{B}$ into $R_1^{\minus},R_1^{\plus},R_2,R_3$ 
(see Figure \ref{fig:failure1}).
The first row of Table \ref{Tab:main example table} expresses $2\Delta\varphi\big[v^\CE\big]$ in terms of 
\begin{equation}\label{auxiliary1}
\alpha(x_1):=\Bbb{E}[1_{[0,1]}(X_2)\mid X_1=x_1], \quad 
\beta(x_2):=\Bbb{E}[1_{[0,1]}(X_1)\mid X_2=x_2]. 
\end{equation}
Such functions do not come up in the case of the marginal game where, on each of $R_1^{\minus},R_1^{\plus},R_2$ or $R_3$,
the function $2\Delta\varphi\big[v^\ME\big]$ is almost surely constant with a value which is an expression in terms of 
\begin{equation}\label{auxiliary2}
p_1^{\minus}:={\rm{P}}_\mathbf{X}(R_1^{\minus}),\quad p_1^{\plus}:={\rm{P}}_\mathbf{X}(R_1^{\plus}),\quad 
p_1:={\rm{P}}_\mathbf{X}(R_1)=p_1^{\minus}+p_1^{\plus},\quad
p_2:={\rm{P}}_\mathbf{X}(R_2), \quad p_3:={\rm{P}}_\mathbf{X}(R_3).
\end{equation}
We finally get to the TreeSHAP games (see Definition \ref{TreeSHAP definition}) 
for $T_1$ and $T_2$ where a  training set $D\subset\mathcal{B}$ (with response variables removed) comes into play. The proportions of training instances ending up in each of the subrectangles can be retrieved from trained decision trees $T_1$ and $T_2$; these estimate the probabilities appeared in \eqref{auxiliary2}. 
\begin{equation}\label{auxiliary3}
\hat{p}_1^{\minus}:=\frac{|D\cap R_1^{\minus}|}{|D|},\quad 
\hat{p}_1^{\plus}:=\frac{|D\cap R_1^{\plus}|}{|D|},\quad 
\hat{p}_1:=\frac{|D\cap R_1|}{|D|}=\hat{p}_1^{\minus}+\hat{p}_1^{\plus},\quad
\hat{p}_2:=\frac{|D\cap R_2|}{|D|}, \quad \hat{p}_3:=\frac{|D\cap R_3|}{|D|}.
\end{equation}
The last two rows of Table \ref{Tab:main example table} present the corresponding differences $2\Delta\varphi$ as piecewise constant functions. The constant value assumed by $2\Delta\varphi$ on each of the subsquares $R_1^{\minus}$, $R_1^{\plus}$, $R_2$ or $R_3$ is in terms of the outputs $c_1,c_2,c_3$ of the simple function $g$ and fractions from \eqref{auxiliary3} which converge to the corresponding probabilities from \eqref{auxiliary2} as $|D|\to\infty$ (assuming that $D$ is drawn i.i.d.). It is not hard to choose these parameters so that, on one of the top subsquares, say on $R_2$, $2\Delta\varphi$ becomes negative for $T_1$ and positive for $T_2$. In such a situation, for instances from $R_2$, TreeSHAP ranks $X_2$ as the most contributing feature to the output of $T_1$ while in the case of $T_2$, it sees $X_1$ as that kind of feature for the same instances. Given that $T_1$ and $T_2$ compute the exact same function. This demonstrates a stark violation of the implementation invariance axiom from  \cite{sundararajan2017axiomatic}. 
A practitioner may neglect this issue by arguing that the training algorithm picks the ``best'' decision tree, so only one of $T_1$ or $T_2$ is relevant. But it is indeed possible to choose the parameters so that $T_1$ and $T_2$ are very close in terms of the impurity measures which are usually employed in constructing classification/regression decision trees.    
As an example, set\footnote{Notice that $p_1+p_2+p_3=p_1^{\minus}+p_1^{\plus}+p_2+p_3$ should be $1$.} 
\begin{equation}\label{parameters}
\begin{split}
&\hat{p}_1^{\minus}\approx p_1^{\minus}=0.33, \quad \hat{p}_1^{\plus}\approx p_1^{\plus}=0.01,
\quad \hat{p}_1\approx p_1=p_1^{\minus}+p_1^{\plus}=0.34, \quad 
\hat{p}_2\approx p_2=0.27, \quad \hat{p}_3\approx p_3=0.39;\\
&c_1=2.03,\quad c_2=1,\quad c_3=2.
\end{split}
\end{equation}
For these parameters, the weighted Gini impurities of the data after the splits at the roots are very close for $T_1$ and $T_2$
\begin{equation}\label{impurity1}
\frac{2\hat{p}_2\hat{p}_3}{\hat{p}_2+\hat{p}_3}\approx \frac{2\hat{p}_1^{\minus}\hat{p}_2}{\hat{p}_1^{\minus}+\hat{p}_2}+
\frac{2\hat{p}_1^{\plus}\hat{p}_3}{\hat{p}_1^{\plus}+\hat{p}_3}.    
\end{equation}
The same is true for the weighted variances after the first splits
\begin{equation}\label{impurity2}
\frac{\hat{p}_2\hat{p}_3}{\hat{p}_2+\hat{p}_3}(c_2-c_3)^2\approx
\frac{\hat{p}_1^{\minus}\hat{p}_2}{\hat{p}_1^{\minus}+\hat{p}_2}(c_1-c_2)^2+
\frac{\hat{p}_1^{\plus}\hat{p}_3}{\hat{p}_1^{\plus}+\hat{p}_3}(c_1-c_3)^2.    
\end{equation}
Consequently, with these parameters, functionally equivalent decision trees $T_1$ and $T_2$ are also (almost) equally likely outcomes of the training algorithm. But at the same time, the corresponding values of $2\Delta\varphi$ on $R_2$ (cf. Table \ref{Tab:main example table}) have different signs,
meaning that the top feature in terms of TreeSHAP differs for these trees for at least $27\%$ of data points. Detailed computations for this example can be found in  Appendix \ref{subappendix:main example's details}.

\begin{table}[ht]
\begin{tabular}{|l|c|c|c|c|}
\hline
\backslashbox{$v$}{$2\Delta\varphi$}
&
on $R_1^{\minus}$
&
on $R_1^{\plus}$
&
on $R_2$
&
on $R_3$
\\
\Xhline{2pt}
$v^\CE(\cdot;\mathbf{X},g)(\mathbf{x})$
&
\tiny
$(c_2-c_1)\cdot\alpha(x_1)$
\normalsize
&
\tiny
$(c_3-c_1)\cdot\alpha(x_1)$
\normalsize
&
\tiny
$
\begin{matrix}
\\
(c_1-c_2)\cdot(1-\alpha(x_1))\\
+(c_2-c_3)\cdot\beta(x_2)\\
\\
\end{matrix}
$
\normalsize
&
\tiny
$
\begin{matrix}
(c_1-c_3)\cdot(1-\alpha(x_1))\\
+(c_3-c_2)\cdot(1-\beta(x_2))
\end{matrix}
$
\normalsize
\\
\hline
$v^\ME(\cdot;\mathbf{X},g)(\mathbf{x})$
&
\tiny
$
(c_2-c_1)(p_2+p_3)
$
\normalsize
&
\tiny
$
(c_3-c_1)(p_2+p_3)
$
\normalsize
&
\tiny
$
\begin{matrix}
\\
(c_1-c_2)p_1\\
+(c_2-c_3)(p_1^{\plus}+p_3)
\\
\phantom{a}
\end{matrix}
$
\normalsize
&
\tiny
$
\begin{matrix}
\\
(c_1-c_3)p_1\\
+(c_3-c_2)(p_1^{\minus}+p_2)
\\
\phantom{a}
\end{matrix}
$
\normalsize
\\
\hline
$v^\Tree(\cdot;T_1)(\mathbf{x})$
&
\tiny
$
(c_2-c_1)(\hat{p}_2+\hat{p}_3)
$
\normalsize
&
\tiny
$
(c_3-c_1)(\hat{p}_2+\hat{p}_3)
$
\normalsize
&
\tiny
$
\begin{matrix}
\\
c_1\hat{p}_1+c_2(\hat{p}_2+\hat{p}_3)\\
-\left(c_2\frac{\hat{p}_2}{\hat{p}_2+\hat{p}_3}+c_3\frac{\hat{p}_3}{\hat{p}_2+\hat{p}_3}\right)\\
\\
\end{matrix}
$
\normalsize
&
\tiny
$
\begin{matrix}
c_1\hat{p}_1+c_3(\hat{p}_2+\hat{p}_3)\\
-\left(c_2\frac{\hat{p}_2}{\hat{p}_2+\hat{p}_3}+c_3\frac{\hat{p}_3}{\hat{p}_2+\hat{p}_3}\right)
\end{matrix}
$
\normalsize
\\
\hline
$v^\Tree(\cdot;T_2)(\mathbf{x})$
&
\tiny
$
(c_2-c_1)\frac{\hat{p}_2}{\hat{p}_1^{\minus}+\hat{p}_2}
$
\normalsize
&
\tiny
$
(c_3-c_1)\frac{\hat{p}_3}{\hat{p}_1^{\plus}+\hat{p}_3}
$
\normalsize
&
\tiny
$
\begin{matrix}
\\
\left(c_1\frac{\hat{p}_1^{\minus}}{\hat{p}_1^{\minus}+\hat{p}_2}+c_2\frac{\hat{p}_2}{\hat{p}_1^{\minus}+\hat{p}_2}\right)\\
-\left(c_2(\hat{p}_1^{\minus}+\hat{p}_2)+c_3(\hat{p}_1^{\plus}+\hat{p}_3)\right)\\
\\
\end{matrix}
$
\normalsize
&
\tiny
$
\begin{matrix}
\left(c_1\frac{\hat{p}_1^{\plus}}{\hat{p}_1^{\plus}+\hat{p}_3}+c_3\frac{\hat{p}_3}{\hat{p}_1^{\plus}+\hat{p}_3}\right)\\
-\left(c_2(\hat{p}_1^{\minus}+\hat{p}_2)+c_3(\hat{p}_1^{\plus}+\hat{p}_3)\right)
\end{matrix}
$
\normalsize
\\
\hline
\end{tabular}
\caption{
The table for Example \ref{main example} where $\mathbf{X}=(X_1,X_2)$ are the predictors, and 
two decision trees $T_1$ and $T_2$ computing the same function $g$ are considered as in Figure \ref{fig:failure1}. For different games associated with $\mathbf{X}$, $g$, $T_1$ and $T_2$, the table captures the values of $2\Delta\varphi$ over various parts of the input space. Here, $\Delta\varphi$ is the difference $\varphi_1-\varphi_2$ of the Shapley values of the game under consideration (see \eqref{Delta}). It is observed that the first two rows only depend on $(\mathbf{X},g)$ while on the last two rows, for TreeSHAP games, $2\Delta\varphi$ differs for $T_1$ and $T_2$. For a choice of parameters such as \eqref{parameters}, on $R_2$ one has $\varphi_1<\varphi_2$ for $T_1$ and  $\varphi_1>\varphi_2$ for $T_2$; hence inconsistent rankings of features by TreeSHAP.}
\label{Tab:main example table}
\end{table}
\end{example}
Notice that, as expected, 
 the entries of  Table \ref{Tab:main example table} pertaining to the conditional game have conditional expectation terms such as \eqref{auxiliary1}. This is not the case for the marginal game where, on the second row, we see constants. The same is true for the TreeSHAP games on the last two rows as well. Indeed, 
$v^\ME(\cdot;\mathbf{X},g)$, $v^\Tree(\cdot;T_1)$ and $v^\Tree(\cdot;T_2)$ are piecewise constant with respect to the partition $\widetilde{\mathscr{P}(T_1)}=\widetilde{\mathscr{P}(T_2)}$ of $\mathcal{B}$ into subsquares $R_1^{\minus},R_1^{\plus},R_2,R_3$. This is an instance of Theorem \ref{main theorem} in the next 
section. 

\subsection{Marginal and TreeSHAP feature attributions are piecewise constant functions}\label{subsec:theorem}
\begin{figure}
\includegraphics[width=14cm]{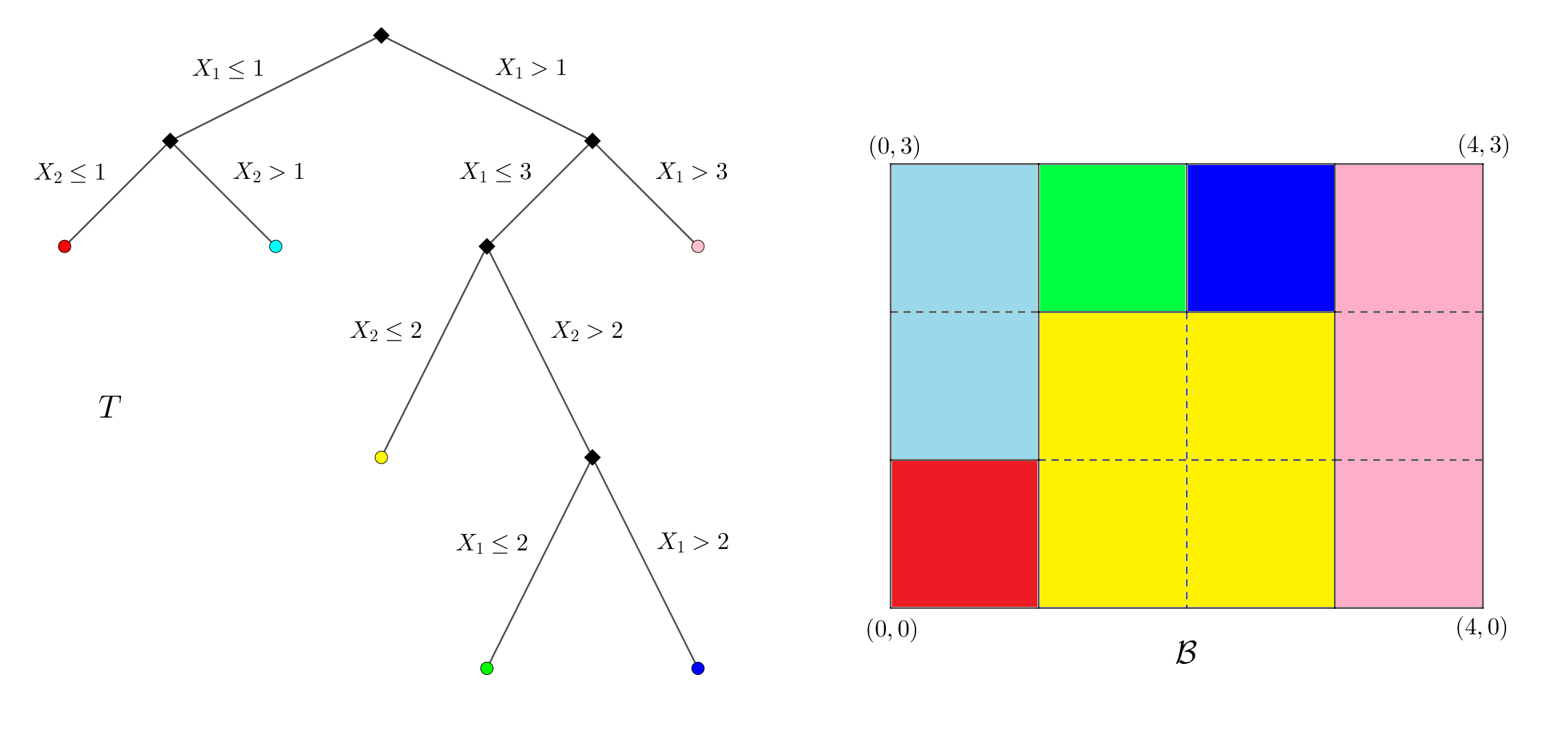}
\caption{The partition determined by a decision tree $T$ and its completion into a grid are demonstrated. 
The tree on the left implements a simple function $g=g(x_1,x_2)$ supported in $\mathcal{B}=[0,4]\times[0,3]$.
In the corresponding partition $\mathscr{P}(T)$ of $\mathcal{B}$ on the right, $g$ is constant on the smaller subrectangles each corresponding to the leaf of the same color. 
Here, the tree is not oblivious and $\mathscr{P}(T)$ is not a grid.
Adding the dotted lines refines $\mathscr{P}(T)$ to a three by four grid $\widetilde{\mathscr{P}(T)}$. 
On each piece of the grid, the associated marginal game (and thus its Shapley values) is almost surely constant with a value which is an expression in terms of probabilities 
${\rm{P}}_{\mathbf{X}}(\tilde{R})\,(\tilde{R}\in \widetilde{\mathscr{P}(T)})$ (cf. Theorem \ref{main theorem}); these in general cannot be retrieved from the trained tree unless it is oblivious.}
\label{fig:partition}
\end{figure}

We now arrive at one of our main results which generalizes observations made after Example \ref{main example} by showing that, for a decision tree, marginal and TreeSHAP feature attributions are 
constant on each piece of the grid partition determined by the tree (see Figure \ref{fig:partition}). 
This fails for the conditional feature attributions as observed in Example \ref{main example}.
\begin{theorem}\label{main theorem}
With notation as in  Sections \ref{subsec:convention}, \ref{subsec:ensembles}, let $T$ be a regression tree trained on data $D$ which implements a function $g=g_T:\Bbb{R}^n\rightarrow\Bbb{R}$. 
Consider the corresponding marginal 
$$v^\ME:S\mapsto v^\ME(S;\mathbf{X},g)(\mathbf{x})\quad (S\subseteq N)$$
and TreeSHAP 
$$v^\Tree:S\mapsto v^\Tree(S;T)(\mathbf{x})
\quad (S\subseteq N)$$ 
games. Then for any linear game value $h$ and any $i\in N$, 
$\mathbf{x}\mapsto h_i\big[v^\ME\big](\mathbf{x})$ 
and 
$\mathbf{x}\mapsto h_i\left[v^\Tree\right](\mathbf{x})$ are  simple functions with respect to the grid partition $\widetilde{\mathscr{P}(T)}$ away from the decision boundary\footnote{This is cut by all hyperplanes 
$X_i=\text{threshold}$ where $T$ splits on $(X_i,\text{threshold})$ at some node.} of $T$ whose ${\rm{P}}_{\mathbf{X}}$-measure is zero due to \eqref{assumption}. More precisely, away from the decision boundary:
\begin{enumerate}
\item  
$h_i\big[v^\ME\big]$ coincides with  a linear combination of indicator functions
$\left\{\mathbbm{1}_{\tilde{R}}\mid \tilde{R}\in \widetilde{\mathscr{P}(T)}\right\}$ 
with coefficients that are linear combination of elements of 
$\left\{{\rm{P}}_{\mathbf{X}}(\tilde{R})\mid \tilde{R}\in \widetilde{\mathscr{P}(T)}\right\}$;
\item     
$h_i\left[v^\Tree\right]$ coincides with
a linear combination of indicator functions
$\left\{\mathbbm{1}_{\tilde{R}}\mid \tilde{R}\in \widetilde{\mathscr{P}(T)}\right\}$ 
with coefficients that are rational expressions of elements of 
$\left\{\hat{{\rm{P}}}_{\mathbf{X}}(R)\mid R\in \mathscr{P}(T)\right\}$
where $\hat{{\rm{P}}}_{\mathbf{X}}(R)$ is the estimation based on data $D$ for the probability ${\rm{P}}_{\mathbf{X}}(R)$ of ending up at the closed rectangular region $R$:
\begin{equation}\label{estimator}
\hat{{\rm{P}}}_{\mathbf{X}}(R):=\frac{|D\cap R|}{|D|}.    
\end{equation}
\end{enumerate}
\end{theorem}
A proof will be presented in Appendix \ref{subappendix:simple}. 
The theorem also indicates that the feature attributions obtained from marginal and TreeSHAP games are very different functions. Indeed, for the former, one needs to compute/estimate probabilities 
${\rm{P}}_{\mathbf{X}}(\tilde{R}),\tilde{R}\in\widetilde{\mathscr{P}(T)}$,
which in general cannot be recovered from the trained model since $\tilde{R}$ comes from a partition finer than what the tree determines. In contrast, $h_i\left[v^\Tree\right]$ is in terms of estimated probabilities 
$\hat{{\rm{P}}}_{\mathbf{X}}(R), R\in\mathscr{P}(T)$.
They become close to ${\rm{P}}_{\mathbf{X}}(R)$ when the training set  (assuming it is drawn i.i.d.) is large; 
see Lemma \ref{error}. But even then, $h_i\left[v^\Tree\right]$ is different from $h_i\big[v^\ME\big]$ since it can contain rational non-linear expressions of these probabilities.

\subsection{Fewer features per tree: Implications to computing marginal values}\label{subsec:observations}
In an ensemble $\mathcal{T}$ of decision trees trained on features $\mathbf{X}=(X_1,\dots,X_n)$, 
only a portion of trees splits on a specific feature. By Lemma \ref{dummy players}, that feature is a null player of the marginal game associated with any of the other trees.   
Therefore, if one wants to quantify the contribution of $X_i$ as $h_i\big[v^\ME\big]$ 
via a linear game value $h$ such as Shapley or Banzhaf, 
then only those decision trees which split on $X_i$ matter (due to the null-player property of $h$); and for those trees, one essentially deals with a game whose players are the features appearing in that tree (due to the carrier-dependence property of $h$).\\
\indent
To make this more precise, as in \Sec \ref{subsec:ensembles}, write the trees in $\mathcal{T}$ as 
$T_1,\dots,T_t$. Let $f$ denote the function computed by $\mathcal{T}$ and 
$g_j$ the one computed by $T_j$;
thus $f=\sum_{j=1}^tg_j$. 
Denote the indices of features appearing in the decision tree $T_j$, i.e. the features on which $T_j$ splits, by $U_j\subseteq N$. Thus $g_j$ may be written as $g_j=\tilde{g}_j\circ\pi_{U_j}$ where $\pi_{U_j}:\mathbf{x}\mapsto\mathbf{x}_{U_j}$ is the projection operator.  
Invoking Lemma \ref{dummy players},  $U_j$ is a carrier for $v^\ME(\cdot;\mathbf{X},g_j)$; and  for a game value $h$ of form \eqref{linear general variant} satisfying the carrier-dependence property: 
\small
\begin{equation}\label{fewer features}
h_i\big[v^\ME(\cdot;\mathbf{X},f)\big](\mathbf{x})=
\sum_{\substack{1\leq j\leq t\\
T_j\in\mathcal{T}^{(i)}}}
\sum_{S\subseteq U_j\setminus\{i\}}
w(S;U_j,i)\left(v^\ME(S\cup\{i\};\mathbf{X}_{U_j},\tilde{g}_j)(\mathbf{x}_{U_j})
-v^\ME(S;\mathbf{X}_{U_j},\tilde{g}_j)(\mathbf{x}_{U_j})\right).
\end{equation}
\normalsize
The total number of summands in this double summation is no more than
$$(\text{number of trees in which } X_i \text{ appears})\cdot 2^{\max_{1\leq j\leq t}|U_j|-1}=\big|\mathcal{T}^{(i)}\big|\cdot 2^{\max_{1\leq j\leq t}|U_j|-1}.$$
However, applying $h$ directly to  $v^\ME(\cdot;\mathbf{X},f)$ without utilizing the fact that not all features appear in all trees results in 
$$
h_i\big[v^\ME(\cdot;\mathbf{X},f)\big](\mathbf{x})
=\sum_{S\subseteq N\setminus \{i\}}w(S;N,i)\left(v^\ME(S\cup\{i\};\mathbf{X},f)(\mathbf{x})
-v^\ME(S;\mathbf{X},f)(\mathbf{x})\right),
$$
which has $2^{n-1}$ summands. In the presence of a structural constraint which limits the number of distinct features per tree, $2^{n-1}$ can be much larger than the number of summands in \eqref{fewer features}.

\begin{example}\label{EBM}
Explainable Boosting Machines \cite{2019arXiv190909223N} are based on 
\textit{Generalized Additive Models plus Interactions} \cite{lou2013accurate}. In the boosting procedure, each tree is trained on one or two features. Hence the model computes a function of the form 
$$
f(\mathbf{x})=\sum_{i\in N}f_i(x_i)+\sum_{(i,j)\in\mathcal{U}}f_{ij}(x_i,x_j)
$$
where interaction terms are indexed by the subset $\mathcal{U}$ of $\{(i,j)\mid 1\leq i< j\leq n\}$.
With the knowledge of the constituent univariate and bivariate  parts of the model, the marginal Shapley values 
$\varphi_k\big[v^\ME\big]=\varphi_k\left[N,v^\ME(\cdot;\mathbf{X},f)\right]$ $(k\in N)$
can be computed easily:
\small
\begin{equation*}
\begin{split}
\varphi_k\big[v^\ME\big](\mathbf{x})
=&\left(f_k(x_k)-\Bbb{E}[f_k(X_k)]\right)
+\sum_{\substack{1\leq i<k\\
(i,k)\in\mathcal{U}}}\frac{1}{2}\left(f_{ik}(x_i,x_k)-\Bbb{E}[f_{ik}(X_i,X_k)]\right)
+\sum_{\substack{k<j\leq n\\
(k,j)\in\mathcal{U}}}\frac{1}{2}\left(f_{kj}(x_k,x_j)-\Bbb{E}[f_{kj}(X_k,X_j)]\right)
\\
&+\sum_{\substack{1\leq i<k\\
(i,k)\in\mathcal{U}}}
\frac{1}{2}\left(\Bbb{E}[f_{ik}(X_i,x_k)]-\Bbb{E}[f_{ik}(x_i,X_k)]\right)
+\sum_{\substack{k<j\leq n\\
(k,j)\in\mathcal{U}}}
\frac{1}{2}\left(\Bbb{E}[f_{kj}(x_k,X_j)]-\Bbb{E}[f_{kj}(X_k,x_j)]\right).
\end{split}    
\end{equation*}
\normalsize
This can be used to estimate the marginal Shapley values of an EBM model based on a background dataset.\footnote{The method 
\texttt{predict\_and\_contrib}
returns the univariate and bivariate terms of an EBM model.}
\end{example}

\begin{remark}
Equation \eqref{fewer features} remains valid for coalitional game values
with the null-player and coalitional carrier-dependence properties, e.g. the Owen value which is known to be helpful in explaining ML models \cite{2021arXiv210210878M}. Nevertheless, for EBM models (discussed above), the marginal Owen values do not differ from Shapley since for a game with two players Shapley and Owen values always coincide. But for more complicated yet still constrained (in terms of number of features per tree) models like CatBoost (see below),  partitioning features yields different marginal values (cf. Appendix \ref{appendix:generalization}).     
\end{remark}

\subsection{Marginal Shapley values for ensembles of oblivious trees}\label{subsec:CatBoost}
In previous sections, we observed that marginal feature attributions for a decision tree $T$ are piecewise constant functions, but with respect to the grid $\widetilde{\mathscr{P}(T)}$; they do not necessarily remain constant on a region belonging to the original partition $\mathscr{P}(T)$. The key observation of this section is that if $T$ is oblivious (cf. Definition \ref{oblivious definition}), then  $\mathscr{P}(T)=\widetilde{\mathscr{P}(T)}$. Invoking this and the symmetry of trees, we obtain a formula in terms of internal model parameters for marginal Shapley and Banzhaf values of ensembles of oblivious trees, e.g. CatBoost models.  
We first need some preliminary work before stating the theorem. 

\begin{definition-notation}\label{groundwork}
Let $T$ be an oblivious decision tree of depth $m$. Denote the distinct features appearing in $T$ from top to bottom by 
$X'_1,\dots,X'_k$, say a subset of the ambient set of features $\{X_1,\dots,X_n\}$. 
Starting from the root, suppose  $T$ splits with respect to $X'_{r_1},\dots,X'_{r_m}$ 
($r_1,\dots,r_m\in\{1,\dots,k\}$ not necessarily distinct);   
 and write the corresponding thresholds as 
$t'_1,\dots,t'_m$---i.e. splits at all $2^{s-1}$ internal nodes of the $s^{\rm{th}}$ level  occur based on comparing $X'_{r_s}$ with $t'_s$.
\begin{itemize}
\item 
We label each leaf with a binary code $\mathbf{a}\in\{0,1\}^m$ where, for the $s^{\rm{th}}$ bit, $a_s=0$ amounts to $X'_{r_s}\leq t'_s$ while $a_s=1$ corresponds to $X'_{r_s}>t'_s$. Any such code can also be thought of as a path from the root to a leaf.
\item  One has a partition $\mathsf{p}:=\{S_1,\dots,S_k\}$ of $M:=\{1,\dots,m\}$, 
indexed with elements of $K:=\{1,\dots,k\}$, 
 where $S_q:=\{s\in M\mid r_s=q\}$ for any $q\in K$.
\item 
When an oblivious tree splits on a feature more than once, some of the regions cut by the tree become vacuous; see Figure \ref{fig:repeated}.\footnote{For instance, if $T$ splits on $X'_i$ twice with thresholds $t_*$ and $\tilde{t}_*$, then all regions with 
$X'_i\in[\max(t_*,\tilde{t}_*),\min(t_*,\tilde{t}_*)]$ become vacuous (or of probability zero when 
$t_*=\tilde{t}_*$ because the event $X'_i=t_*$ is of probability zero due to our assumption in \eqref{assumption}.)}
A leaf is called \textit{realizable} if the path to it from the root does not encounter conflicting thresholds for a feature. The corresponding set of realizable binary codes is denoted by $\mathcal{R}\subseteq\{0,1\}^m$. Regions encoded by elements of $\{0,1\}^m\setminus\mathcal{R}$ are vacuous (or of probability zero).
\item We now define purely combinatorial functions parametrized by a partition $\mathsf{p}$ of a segment 
$M=\{1,\dots,m\}$ of natural numbers. First, for binary codes $\mathbf{e},\mathbf{e}'\in\{0,1\}^m$, 
set
\begin{equation}\label{auxiliary4'}
\mathcal{E}(\mathbf{e},\mathbf{e}';\mathsf{p}):=\big\{q\in K\mid \mathbf{e}_{S_q}=\mathbf{e}'_{S_q}\big\}.  
\end{equation}
Next, based on this, we define a function $\mathcal{E}^{-1}$ which takes in a  binary code 
$\mathbf{e}\in\{0,1\}^m$ and a subset $Q$ of $K=\{1,\dots,k=|\mathsf{p}|\}$ as inputs:
\begin{equation}\label{auxiliary5'}
\mathcal{E}^{-1}(\mathbf{e},Q;\mathsf{p}):=
\big\{\mathbf{e}'\in\{0,1\}^m\mid \mathcal{E}(\mathbf{e},\mathbf{e}';\mathsf{p})=Q\big\}. 
\end{equation}

\end{itemize}
\end{definition-notation}

Now consider $(\mathbf{X},f)$ where $\mathbf{X}=(X_1,\dots,X_n)$ are the predictors and $f:\Bbb{R}^n\rightarrow\Bbb{R}$ is a function implemented by an ensemble $\mathcal{T}$ of oblivious decision trees.
The notions discussed in Definition-Notation \ref{groundwork} can be considered for any arbitrary tree $T$ from the ensemble: We denote the depth by $m(T)$ and the number of distinct features on which $T$ splits by $k(T)$. The leaves of $T$ are in a bijection with binary codes in   $\{0,1\}^{m(T)}$; the subset of realizable ones is denoted by $\mathcal{R}(T)\subseteq\{0,1\}^{m(T)}$. Moreover, the levels on which different features from $\{X_1,\dots,X_n\}$ appear determine a partition of size $k(T)$ of 
$\{1,\dots,m(T)\}$ which we denote by $\mathsf{p}(T)$. The tree defines an inclusion 
\begin{equation}\label{inclusions}
\iota(\cdot;T):
K(T):=\{1,\dots,k(T)\}\hookrightarrow N=\{1,\dots,n\}    
\end{equation}
with the property that $X_{\iota(i;T)}$ is the $i^{\rm{th}}$ (enumerated from the root) distinct feature on which $T$ splits (i.e. $X'_i=X_{\iota(i;T)}$ in terms of the above notation).   
The corresponding element of $\mathsf{p}(T)$ captures all levels where $T$ splits with respect to $X_{\iota(i;T)}$.
Moreover, $T$ belongs to the subset $\mathcal{T}^{(i_*)}$ of trees splitting on a given feature $X_{i_*}$ if and only if 
$i_*$ is in the image of $\iota(\cdot;T)$. 
Finally, for any leaf of $T$ encoded with $\mathbf{e}\in\{0,1\}^{m(T)}$, we denote the value at the leaf with $c(\mathbf{e};T)$, the corresponding rectangular region in $\Bbb{R}^n$ with $R(\mathbf{e};T)$, and the probability ${\rm{P}}_{\mathbf{X}}(R(\mathbf{e};T))$ of a data point ending up in that region with $p(\mathbf{e};T)$.

\begin{theorem}\label{Catboost theorem}
With the notation as above, 
for any $i_*\in N=\{1,\dots,n\}$ and ${\rm{P}}_{\mathbf{X}}$-almost every explicand $\mathbf{x}\in\Bbb{R}^n$,
the marginal Shapley value $\varphi_{i_*}\big[v^\ME\big](\mathbf{x})$ of the regressor implemented by $\mathcal{T}$ is given by 
\begin{equation}\label{symmetric formula ensemble}
\varphi_{i_*}\big[v^\ME\big](\mathbf{x})=\sum_{\substack{T\in\mathcal{T}^{(i_*)}\\
i\in K(T), \iota(i;T)=i_*}} 
\big[\phi^+(\mathbf{a};i,T)-\phi^-(\mathbf{a};i,T)\big]
\quad (K(T):=\{1,\dots,k(T)\})
\end{equation}
where $\mathbf{a}=\mathbf{a}(\mathbf{x};T)$ is so that  $\mathbf{x}\in R(\mathbf{a};T)$ is satisfied.
The terms inside the brackets in \eqref{symmetric formula ensemble} are defined in terms of expressions\footnote{The quantity $\mathfrak{s}(\mathbf{e},Q;T)$ may be interpreted as an expectation: this is the expected leaf score taken over leaves $\mathbf{u}$ that ``partially'' coincide with $\mathbf{e}$ (in the sense of $\mathcal{E}(\mathbf{e},\mathbf{u};\mathsf{p}(T))=Q$) assuming that their scores are replaced with that of $\mathbf{e}$.} 
\begin{equation}\label{auxiliary5''}
\mathfrak{s}(\mathbf{e},Q;T):=
c(\mathbf{e};T)\cdot\Big(\sum_{\mathbf{u}\in\mathcal{E}^{-1}
(\mathbf{e},Q;\mathsf{p}(T))\cap\mathcal{R}(T)}p(\mathbf{u};T)\Big)   
\quad (Q\subseteq K(T):=\{1,\dots,k(T)\})
\end{equation}
\normalsize
as
\begin{equation}\label{contributions}
\begin{split}
& \phi^+(\mathbf{a};i,T):=
\sum_{\substack{Z\subseteq K(T)\\ i\in Z}}   
\sum_{\substack{W\subseteq K(T)\\
W\supseteq Z}}\omega^+(|W|,|Z|;k(T))\cdot 
\Big(\sum_{\mathbf{b}\in\mathcal{E}^{-1}(\mathbf{a},W;\mathsf{p}(T))\cap\mathcal{R}(T)}
\mathfrak{s}(\mathbf{b},-Z;T)\Big),\\
& \phi^-(\mathbf{a};i,T):=
\sum_{\substack{W\subseteq K(T)\\ i\notin W}}   
\sum_{\substack{Z\subseteq K(T)\\
Z\subseteq W}}\omega^-(|W|,|Z|;k(T))\cdot
\Big(\sum_{\mathbf{b}\in\mathcal{E}^{-1}(\mathbf{a},W;\mathsf{p}(T))\cap\mathcal{R}(T)}
\mathfrak{s}(\mathbf{b},-Z;T)\Big),
\end{split}    
\end{equation}
where $-Z:=K(T)\setminus Z$, and the coefficients $w^+$ and $w^-$ are defined as 
\begin{equation}\label{weights_Shapley}
\omega^+(w,z;k):=\frac{(z-1)!\,(k-w)!}{(k+z-w)!},\quad   
\omega^-(w,z;k):=\frac{z!\,(k-w-1)!}{(k+z-w)!}.
\end{equation}
\indent 
Via substituting \eqref{auxiliary5''} in \eqref{contributions}, 
each of $\phi^+(\mathbf{a};i,T)$ or $\phi^-(\mathbf{a};i,T)$
can be obtained, through less than 
\begin{equation}\label{complexity multiplication}
\left(2+\frac{m(T)}{k(T)}\right)^{k(T)}\leq 3^{m(T)}
\end{equation}
multiplication operations, 
as 
a summation with no more than  
\begin{equation}\label{complexity0}
3^{k(T)-1}\cdot\left(\frac{m(T)}{k(T)}\right)^{k(T)}\leq 3^{m(T)-1}
\end{equation}
summands, each of them a multiple of the product of two of parameters 
$\{c(\mathbf{b};T)\}_{\mathbf{b}\in\mathcal{R}(T)}$ and 
$\{p(\mathbf{u};T)\}_{\mathbf{u}\in\mathcal{R}(T)}$ which are associated with the leaves of  $T$. 
Thus formula \eqref{symmetric formula ensemble}  
for computing  $\varphi_{i_*}\big[v^\ME\big](\mathbf{x})$ has no more than
\begin{equation}\label{complexity1}
2\cdot \big|\mathcal{T}^{(i_*)}\big|\cdot 3^{(\max_{T\in\mathcal{T}}m(T))-1}
<|\mathcal{T}|\cdot\mathcal{L}^{\log_23}
\end{equation}
terms once expanded.  
\\
\indent
All these hold for the Banzhaf value $Bz_i\big[v^\ME\big](\mathbf{x})$ too 
after replacing coefficients $w^{\pm}(w,z;k)$ 
with 
\begin{equation}\label{weights_Banzhaf}
\tilde{\omega}(w,z;k)=\frac{2^{w-z}}{2^{k-1}}.
\end{equation}
\end{theorem}
A proof can be found in Appendix \ref{subappendix:CatBoost}.
The theorem can be formulated for game values 
introduced in Theorem \ref{classification} too, and also for coalitional game values such as the Owen value albeit it becomes tedious. So we skip it here and we 
refer the reader to Appendix \ref{appendix:generalization}.

\begin{remark}
Formula \eqref{symmetric formula ensemble} for marginal Shapley values of ensembles of symmetric trees cannot be simplified any further because,
once $\phi^+(\mathbf{a};i,T)$ and $\phi^-(\mathbf{a};i,T)$ are expanded as in \eqref{auxiliary5''} and \eqref{contributions}, the resulting pairs $(\mathbf{b},\mathbf{u})$  are distinct.
\end{remark}

\begin{remark}
Notice the two speedups employed in the theorem. The summation in 
\eqref{symmetric formula ensemble}
only considers the trees relevant to the feature under consideration; and numbers 
$\phi^+(\mathbf{a};i,T)$, $\phi^-(\mathbf{a};i,T)$ therein are described in summations 
\eqref{contributions} by focusing only on realizable binary codes $\mathbf{b},\mathbf{u}$ which causes the number of summands to decrease when the tree has many repeated features. 
\end{remark}

\begin{remark}\label{limited features 1}
The complexity of computing marginal Shapley values for a symmetric tree via the previous theorem is sublinear in terms of the number of leaves once a restriction is posed on the  number of distinct features on which the tree depends. 
This is because if $k(T)\leq k_*$ as $m(T)\to\infty$, then 
$$
3^{k(T)-1}\cdot\left(\frac{m(T)}{k(T)}\right)^{k(T)}=O\left(m(T)^{k_*}\right)
\quad (m(T)=\log_2(\text{\# of leaves})).
$$ 
\end{remark}

\begin{example}\label{Repeated feature example} 
The goal of this example is to verify Theorem \ref{Catboost theorem} in the case of the symmetric tree of depth three illustrated in Figure \ref{fig:repeated}. 
The features $(X_1,X_2)$ are supported in the rectangle $\mathcal{B}:=[0,3]\times[0,2]$ which is partitioned  into six subrectangles by the tree. 
Following Definition-Convention \ref{groundwork}, we label the leaves of $T$ and their corresponding subrectangles with binary codes in $\{0,1\}^3$. The set of realizable binary codes is of size six $\mathcal{R}=\{000,010,100,110,101,111\}$. The simple function computed by $T$ is 
\begin{equation}\label{auxiliary6}
g=c_{000}\cdot\mathbbm{1}_{R_{000}}+c_{010}\cdot\mathbbm{1}_{R_{010}}
+c_{100}\cdot\mathbbm{1}_{R_{100}}+c_{110}\cdot\mathbbm{1}_{R_{110}}
+c_{101}\cdot\mathbbm{1}_{R_{101}}+c_{111}\cdot\mathbbm{1}_{R_{111}}
\end{equation}
where 
\small
$$R_{000}=[0,1]\times[0,1],\,R_{010}=[0,1]\times[1,2],\,R_{100}=[1,2]\times[0,1],R_{110}=[1,2]\times[1,2],\,
R_{101}=[2,3]\times[0,1],\,R_{111}=[2,3]\times[1,2].$$
\normalsize
The feature $X_1$ appears on the first and third levels while $X_2$ appears only on the second level. This amounts to the partition
$$
\mathsf{p}(T)=\mathsf{p}=\{S_1=\{1,3\},S_2=\{2\}\}
$$
of $M=\{1,2,3\}$ which is indexed by $K=\{1,2\}$. 
To obtain the set $\mathcal{R}\subset\{0,1\}^3$ of realizable codes, all codes with the first bit $0$ and the third bit $1$, i.e. codes $001$ and $011$, should be excluded.  
Now suppose we want to compute the value that $\varphi_1\big[v^\ME\big]$ attains on 
one of the subrectangles, say on $R_{110}=[1,2]\times[1,2]$; that is, $\mathbf{a}=110$. 
In \eqref{contributions}, we are interested in pairs $(\mathbf{b},\mathbf{u})$ of elements of $\mathcal{R}$ with 
\begin{equation}\label{Z and W}
\begin{split}
&\mathcal{E}(\mathbf{a},\mathbf{b};\mathsf{p})=W,\quad
\mathcal{E}(\mathbf{b},\mathbf{u};\mathsf{p})=\{1,2\}\setminus Z,
\quad 1\in Z\subseteq W\subseteq\{1,2\};\\
&\mathcal{E}(\mathbf{a},\mathbf{b};\mathsf{p})=W,\quad
\mathcal{E}(\mathbf{b},\mathbf{u};\mathsf{p})=\{1,2\}\setminus Z,
\quad Z\subseteq W\subseteq\{2\}.
\end{split}    
\end{equation}
Since $S_1=\{1,3\}$, we have 
$b_1=a_1=1$ and $b_3=a_3=0$ on the first line, and 
$u_1=b_1$ and $u_3=b_3$ on the second. 
As for the second bits, given that $S_2=\{2\}$, they are determined based on if $2$ belongs to $Z$ and $W$ or not. 
All possibilities for bits $b_2$ and $u_2$ 
are summarized in Table \ref{Tab: Repeated Example} along with the corresponding weights for these pairs as defined in \eqref{weights_Shapley}.
We now invoke Theorem \ref{Catboost theorem} to compute $\varphi_1\big[v^\ME\big](\mathbf{x})$ where 
$\mathbf{x}\in R_{110}$. 
In the first summation $\phi^+(\mathbf{a};i,T)$ from \eqref{contributions}, 
one should have $\mathbf{b}\in\{100,110\}$, and  the summation becomes 
$$
A_1:=\frac{1}{2}\,c_{100}\left(p_{000}+p_{101}\right)+
\frac{1}{2}\,c_{110}\left(p_{000}+p_{101}\right)+
c_{110}\left(p_{010}+p_{111}\right).
$$
Notice that the three summands above correspond to the rows of  
Table \ref{Tab: Repeated Example}.
Next, in the summation $\phi^-(\mathbf{a};i,T)$ from \eqref{contributions}, to be subtracted from the former, one has $\mathbf{b}\in\{000,010,101,111\}$ and the summation becomes
$$
A_2:=\frac{1}{2}\left(c_{000}\,p_{000}+c_{101}\,p_{101}\right)
+\frac{1}{2}\left(c_{010}\,p_{000}+c_{111}\,p_{101}\right)
+\left(c_{010}\,p_{010}+c_{111}\,p_{111}\right).
$$
Therefore, $\varphi_1\big[v^\ME\big](\mathbf{x})=A_1-A_2$ whenever $\mathbf{x}\in R_{110}$. 
To verify this directly, notice that for the model \eqref{auxiliary6} the first marginal Shapley value over $R_{110}$
is given by:
\begin{equation}
\begin{split}
\varphi_1\big[v^\ME\big](\mathbf{x})=&
\frac{1}{2}\big[\Bbb{E}[g(x_1,X_2)]-\Bbb{E}[g(X_1,x_2)]\big]
+\frac{1}{2}\big[g(\mathbf{x})-\Bbb{E}[g(X_1,X_2)]\big]\\
=&\frac{1}{2}\big[(c_{110}(p_{010}+p_{110}+p_{111})+c_{100}(p_{000}+p_{100}+p_{101}))\\
&\quad -(c_{010}(p_{010}+p_{000})+c_{110}(p_{110}+p_{100})+c_{111}(p_{111}+p_{101}))\big]\\
&+\frac{1}{2}\big[c_{110}-(c_{000}p_{000}+c_{100}p_{100}+c_{101}p_{101}
+c_{010}p_{010}+c_{110}p_{110}+c_{111}p_{111})\big].
\end{split}
\end{equation}
It is not hard to check that the above expression coincides with $A_1-A_2$ after simplification. 
\begin{table}[ht]
\begin{tabular}{|c|c|c|c?c|c|c|c|}
\cline{2-7}
\multicolumn{1}{c|}{}& $Z$ & $W$ & $\omega^+$ & $Z$ & $W$ & $\omega^-$\\
\hline
$u_2=b_2,\, b_2\neq a_2$ & $\{1\}$ & $\{1\}$ & $\frac{1}{2}$
& $\varnothing$ & $\varnothing$ & $\frac{1}{2}$\\[0.3ex]
\hline
$u_2\neq b_2,\, b_2=a_2$ & $\{1,2\}$ & $\{1,2\}$ & $\frac{1}{2}$ 
& $\{2\}$ & $\{2\}$ & $\frac{1}{2}$\\[0.3ex]
\hline
$u_2=b_2,\, b_2=a_2$ & $\{1\}$ & $\{1,2\}$ & $1$ & $\varnothing$ & $\{2\}$ & $1$\\
\hline
\end{tabular}
\caption{ Table required for Example \ref{Repeated feature example} where the marginal Shapley values are computed for the tree illustrated in Figure \ref{fig:repeated}. 
All choices in \eqref{Z and W} for nested subsets $Z\subseteq W$  are outlined along their ramifications to the second bit of the binary codes, and the respective weights (cf. \eqref{weights_Shapley}) which appear in the formula.}
\label{Tab: Repeated Example}
\end{table}
\begin{figure}
\includegraphics[width=14cm]{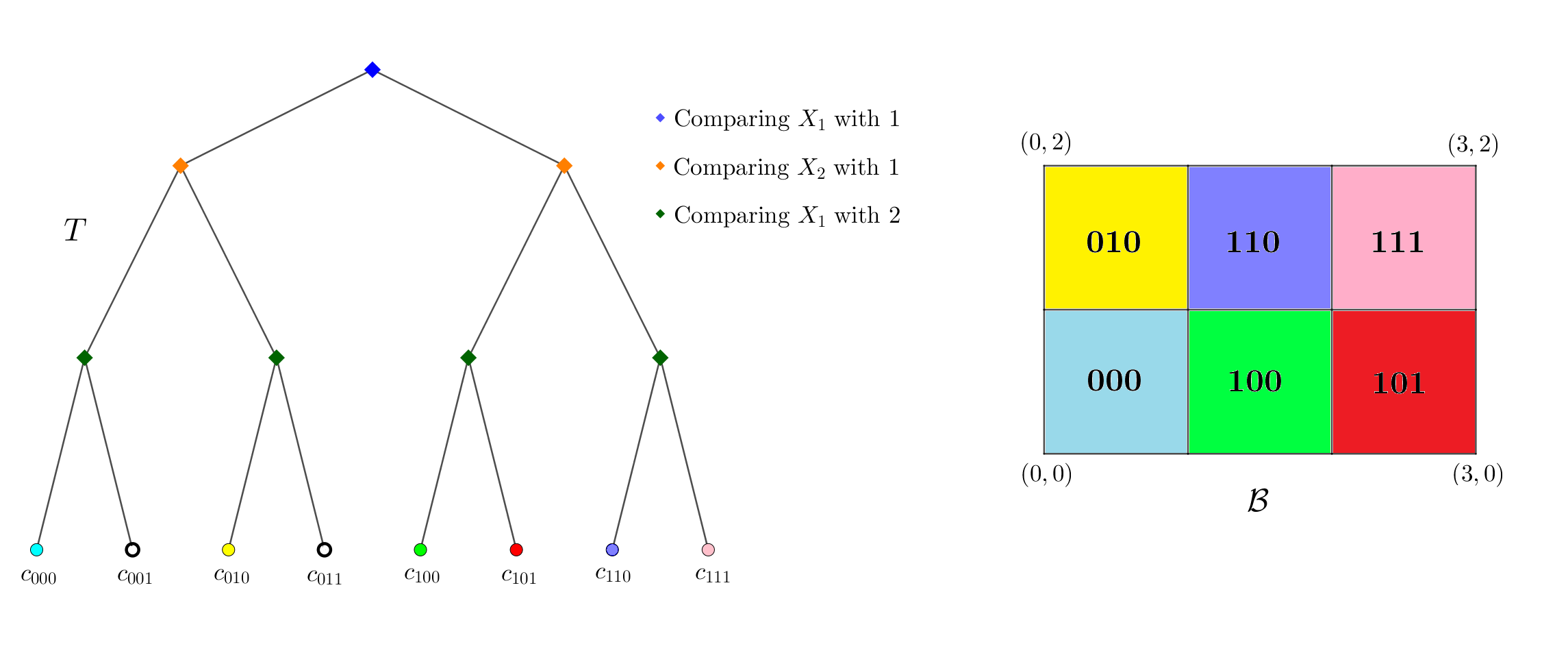}
\caption{The picture for Example \ref{Repeated feature example} 
which is concerned with the oblivious decision tree $T$ of depth three appearing on the left.
At each split, we go right if the feature is larger than the threshold and go left otherwise. The leaves can thus be encoded with elements of $\{0,1\}^3$; the same holds for the regions into which the tree, as on the right, partitions the rectangle $\mathcal{B}=[0,3]\times[0,2]$ where the features $(X_1,X_2)$ are supported. 
But two of the binary codes, $001$ and $011$, do not amount to any region since the paths from the root to their corresponding leaves encounter conflicting thresholds for feature $X_1$. Any other leaf of $T$ corresponds to the region of the same color on the right.}
\label{fig:repeated}
\end{figure}
\end{example}

\begin{example}\label{distinct features}
Let us apply Theorem \ref{Catboost theorem} to an oblivious decision tree $T$ of depth $m$ which does not split on any feature on more than one level. With the notation as in Definition-convention \ref{groundwork}, one has 
$m(T)=k(T)=m$; and 
the leaves, along with their values and their corresponding regions, can be encoded with elements of $\{0,1\}^m$---all of them are realizable, i.e. $\mathcal{R}=\mathcal{R}(T)$ coincides with  
$\{0,1\}^m$. Enumerating the features appearing in the tree from the top
as $X_1,\dots,X_m$ (all distinct), 
the resulting partition $\mathsf{p}=\mathsf{p}(T)$ of $M=\{1,\dots,m\}$  coincides with the partition 
$\{S_q=\{q\}\}_{q\in M}$
into singletons. It is easy to check that in this case 
$\mathcal{E}^{-1}(\mathbf{e},Q;\mathsf{p})$ has only one element, which, denoted by 
$\sigma(\mathbf{e};Q)=\left(\sigma(\mathbf{e};Q)_i\right)_{i\in M}$, is given by:
$$
\sigma(\mathbf{e};Q)_i:=
\begin{cases}
e_i & \text{ if } i\in Q,\\
1-e_i & \text{otherwise}.
\end{cases}
$$
Now for any $i\in M$ and $\mathbf{a}\in\{0,1\}^m$, the followings holds for ${\rm{P}}_{\mathbf{X}}$-a.e. point $\mathbf{x}$
from the region 
$R_{\mathbf{a}}:=R(\mathbf{a};T)$: 

\footnotesize
\begin{equation}\label{no repetition Shapley}
\varphi_i\big[v^\ME\big](\mathbf{x})=
\sum_{\substack{Z\subseteq M\\ i\in Z}}   
\sum_{\substack{W\subseteq M \\
W\supseteq Z}}\omega^+(|W|,|Z|;m)\cdot
c_{\sigma(\mathbf{a};W)}\cdot p_{\sigma(\sigma(\mathbf{a};W);-Z)}
-\sum_{\substack{W\subseteq M\\ i\notin W}}   
\sum_{\substack{Z\subseteq M\\
Z\subseteq W}}\omega^-(|W|,|Z|;m)
\cdot c_{\sigma(\mathbf{a};W)}\cdot p_{\sigma(\sigma(\mathbf{a};W);-Z)},    
\end{equation}
\begin{equation}\label{no repetition Banzhaf}
Bz_i\big[v^\ME\big](\mathbf{x})=
\sum_{\substack{Z\subseteq M\\ i\in Z}}   
\sum_{\substack{W\subseteq M \\
W\supseteq Z}}\tilde{\omega}(|W|,|Z|;m)\cdot
c_{\sigma(\mathbf{a};W)}\cdot p_{\sigma(\sigma(\mathbf{a};W);-Z)}
-\sum_{\substack{W\subseteq M\\ i\notin W}}   
\sum_{\substack{Z\subseteq M\\
Z\subseteq W}}\tilde{\omega}(|W|,|Z|;m)
\cdot c_{\sigma(\mathbf{a};W)}\cdot p_{\sigma(\sigma(\mathbf{a};W);-Z)},    
\end{equation}
\normalsize
where for any $\mathbf{e}\in\{0,1\}^m$, 
$p_\mathbf{e}:=p(\mathbf{e};T)$ and $c_\mathbf{e}:=c(\mathbf{e};T)$ are the probability and the score for the leaf encoded by $\mathbf{e}$.
Notice that there are precisely $2\cdot 3^{m-1}$ summands in either  
\eqref{no repetition Shapley} or \eqref{no repetition Banzhaf} because the number of pairs of nested subsets of $M\setminus\{i\}$ is $3^{|M\setminus\{i\}|}=3^{m-1}$. We conclude that the bound \eqref{complexity1} is sharp.
\end{example}

Theorem \ref{Catboost theorem} suggests a two-stage algorithm for approximating marginal Shapley values as the Shapley values of the empirical marginal game \eqref{empirical}; this amounts to replacing the true probabilities $p(\mathbf{u};T)$ 
associated with leaves with their estimations based on the training data:
\begin{equation}\label{probability estimation precomputation}
\hat{p}(\mathbf{u};T)=\frac{|D\cap R(\mathbf{u};T)|}{|D|}\quad 
(\mathbf{u}\in\mathcal{R}(T)\subseteq\{0,1\}^{m(T)}, T\in\mathcal{T}).
\end{equation}
\begin{algorithm}\label{algorithm}
With the notation as in Theorem \ref{Catboost theorem}, suppose $f:\Bbb{R}^n\rightarrow\Bbb{R}$ is a model computed by an ensemble $\mathcal{T}$ of oblivious decision trees. 
In the  precomputation phase, for each leaf of each tree, the empirical marginal Shapley values 
arising from that tree are computed for data points that end up at that leaf. These are recorded in look-up tables which are then used in the production phase  to estimate the vector of marginal Shapley values for an input explicand.\\
\indent 
First, $\mathcal{T}$ should be parsed to obtain the following quantities used in the precomputation algorithm:
\begin{itemize}
\item functions 
$\{\iota(\cdot;T)\}_{T\in\mathcal{T}}$ enumerating features appearing in each tree;
\item  depths $\{m(T)\}_{T\in\mathcal{T}}$, distinct feature counts  
$\{k(T)\}_{T\in\mathcal{T}}$ and partitions $\{\mathsf{p}(T)\}_{T\in\mathcal{T}}$ associated with the trees as well as leaf scores $\{\{c(\mathbf{b};T)\}_{\mathbf{b}\in\mathcal{R}(T)}\}_{T\in\mathcal{T}}$
and estimated probabilities 
$\{\{\hat{p}(\mathbf{u};T)\}_{\mathbf{u}\in\mathcal{R}(T)}\}_{T\in\mathcal{T}}$\footnote{
A code for ``unpacking'' a trained XGBoost, LightGBM or CatBoost model and retrieving quantities 
such as distinct features in each tree, depths, estimated leaf probabilities etc. is available on 
\url{https://github.com/FilomKhash/Tree-based-paper}.};
\item combinatorial data 
$\big\{\big\{\mathcal{E}^{-1}(\mathbf{e},Q;\mathsf{p}(T))\big\}_{\mathbf{e}\in\{0,1\}^{m(T)},Q\subseteq K(T)}\big\}_{T\in\mathcal{T}}$.
\end{itemize}
\begin{enumerate}
\setlength{\itemindent}{-.2in}
\item[$\blacktriangleright$] \textbf{Step 1 (precomputation).}   
\end{enumerate}
\small
\begin{algorithmic}
\State \textbf{Input:}
$\{\{c(\mathbf{b};T)\}_{\mathbf{b}\in\mathcal{R}(T)}\}_{T\in\mathcal{T}}$,
$\{\{\hat{p}(\mathbf{u};T)\}_{\mathbf{u}\in\mathcal{R}(T)}\}_{T\in\mathcal{T}}$, 
$\left\{\left\{\mathcal{E}^{-1}(\mathbf{e},Q;\mathsf{p}(T))\right\}_{\mathbf{e}\in\{0,1\}^{m(T)},Q\subseteq K(T)}\right\}_{T\in\mathcal{T}}$
\State \textbf{Output:} Shapley values
$\{\{{\boldsymbol{\hat{\phi}}}(\mathbf{a};T)\}_{\mathbf{a}\in\mathcal{R}(T)}\}_{T\in\mathcal{T}}$
\For {$T\in\mathcal{T}$}
\For {$Q\subseteq K(T)$} 
\Comment{First, we obtain $\mathfrak{s}(\mathbf{e},Q;T)$ as $\mathbf{e}$ and $Q$ vary.}
\For {$\mathbf{e}\in\mathcal{R}(T)$}
\State $s\gets 0$
\For {${\mathbf{u}\in\mathcal{E}^{-1}(\mathbf{e},Q;\mathsf{p}(T))\cap\mathcal{R}(T)}$}
\State $s\gets s+\hat{p}(\mathbf{u};T)$
\EndFor 
\State $\mathfrak{s}(\mathbf{e},Q;T)\gets c(\mathbf{e};T)\cdot s$
\EndFor 
\EndFor
\For {$\mathbf{a}\in\mathcal{R}(T)$} 
\Comment{Shapley values for $\mathbf{a}$ are computed as the difference of expressions in \eqref{contributions}.}
\State $\boldsymbol{\hat{\phi}}(\mathbf{a};T)=(\hat{\phi}_i(\mathbf{a};T))_{1\leq i\leq k(T)}\gets 
(0)_{1\leq i\leq k(T)}$ 
\For {$W\subseteq K(T)$}
\For {${\mathbf{b}\in\mathcal{E}^{-1}(\mathbf{a},W;\mathsf{p}(T))\cap\mathcal{R}(T)}$}
\For {$Z\subseteq W$}
\State $\text{increment}^{\plus}\gets\frac{(|Z|-1)!\,(k(T)-|W|)!}{(K(T)+|Z|-|W|)!}\cdot\mathfrak{s}(\mathbf{b},K(T)\setminus Z;T)$
\State $\text{increment}^{\minus}\gets\frac{|Z|!\,(k(T)-|W|-1)!}{(k(T)+|Z|-|W|)!}\cdot\mathfrak{s}(\mathbf{b},K(T)\setminus Z;T)$
\For {$i\in Z$}
\State $\hat{\phi}_i(\mathbf{a};T)\gets \hat{\phi}_i(\mathbf{a};T)+
\text{increment}^{\plus}$
\EndFor
\For {$i\in K(T)\setminus W$}
\State $\hat{\phi}_i(\mathbf{a};T)\gets \hat{\phi}_i(\mathbf{a};T)-
\text{increment}^{\minus}$
\EndFor
\EndFor
\EndFor
\EndFor
\EndFor
\EndFor
\State \Return 
$\{\{\boldsymbol{\hat{\phi}}(\mathbf{a};T)\}_{\mathbf{a}\in\mathcal{R}(T)}\}_{T\in\mathcal{T}}$
\end{algorithmic}
\normalsize
\begin{enumerate}
\setlength{\itemindent}{-.2in}
\item[$\blacktriangleright$]
\textbf{Step 2 (computation).}   
The outputs  $\{\{\boldsymbol{\hat{\phi}}(\mathbf{a};T)\}_{\mathbf{a}\in\mathcal{R}(T)}\}_{T\in\mathcal{T}}$ of the previous step are saved as look-up tables: one table for each tree $T\in\mathcal{T}$ whose rows are vectors
$\boldsymbol{\hat{\phi}}(\mathbf{a};T) (\mathbf{a}\in\mathcal{R}(T))$  of size $k(T)$.
They then can be employed  for explaining new data points on the fly as long as the model is in production.
\end{enumerate}  
\small
\begin{algorithmic}
\State \textbf{Input:} an explicand $\mathbf{x}$ along with functions 
$\{\iota(\cdot;T)\}_{T\in\mathcal{T}}$ 
and the precomputed Shapley values 
$\{\{\boldsymbol{\hat{\phi}}(\mathbf{a};T)\}_{\mathbf{a}\in\mathcal{R}(T)}\}_{T\in\mathcal{T}}$ 
\State \textbf{Output:} The vector of estimated marginal Shapley values 
$\boldsymbol{\hat{\varphi}}:=\Big(\widehat{\varphi_i\big[v^\ME\big]}(\mathbf{x})\Big)_{1\leq i\leq n}$ 
\State $\boldsymbol{\hat{\varphi}}\gets (0)_{1\leq i\leq n}$
\For {$T\in\mathcal{T}$}
\State $\mathbf{a}\gets\mathbf{a}(\mathbf{x};T)$ 
\Comment{The leaf at which $\mathbf{x}$ ends up may be found with binary search or library's native methods.}
\For {$i\in K(T)$}
\State $\hat{\varphi}_{\iota(i;T)}\gets \hat{\varphi}_{\iota(i;T)}+\hat{\phi}_i(\mathbf{a};T)$
\EndFor
\EndFor
\State \Return $\boldsymbol{\hat{\varphi}}$
\end{algorithmic}
\normalsize
\end{algorithm}

In the implementation, the algorithm  can be highly vectorized, and also parallelized across the trees in the ensemble (or even across the leaves of a single tree if necessary). Its complexity can be analyzed in light of  inequality \eqref{complexity1} for the number of terms in the formula for a marginal Shapley value of an oblivious tree. The aforementioned inequality yields the complexity 
$O\left(|\mathcal{T}|\cdot \mathcal{L}^{\log_2 3}\right)$ 
(in which the constant hidden in the $O$ notation is small).
But, for each tree, the precomputation algorithm goes through all leaves and all features on which the tree splits. This yields the total complexity       
$O\left(|\mathcal{T}|\cdot \mathcal{L}^{\log_2 6}\cdot\log(\mathcal{L})\right)$ for precomputation and  $O\left(|\mathcal{T}|\cdot\mathcal{L}\cdot\log(\mathcal{L})\right)$ for storage of the look-up tables  (again, with a small hidden constant).
To compare with path-dependent and interventional TreeSHAP algorithms (cf. \Sec \ref{subsec:TreeSHAP}) which do not conduct a precomputation, the complexity per leaf should be considered which is 
$O\left(|\mathcal{T}|\cdot \mathcal{L}^{\log_2 3}\cdot\log(\mathcal{L})\right)$.
All these facts are reflected in Table \ref{Tab: complexity}. 
The complexity terms appearing therein for Algorithm \ref{algorithm} are elaborated on in Appendix \ref{subappendix:CatBoost}.\\
\indent
At the heart of the precomputation algorithm is the computation of a preimage 
$\mathcal{E}^{-1}(\mathbf{e},Q;\mathsf{p}(T))$ 
for each tree $T$, where $\mathbf{e}\in \{0,1\}^{m(T)}$  and 
$Q \subseteq K(T)$. The computational complexity, and that of its memory storage, is $O(\mathcal{L}^2)$ for each tree. These values are then used to compute the values 
$\mathfrak{S}=\{\mathfrak{s}(\mathbf{e},Q;T)\}_{\mathbf{e},Q}$ obtained in the first part of Step 1 of Algorithm 3.12, whose complexity is also $O(\mathcal{L}^2)$ for each tree, rather than 
$O\left(\mathcal{L}^{\log_2 6} \cdot \log(\mathcal{L}) \approx \mathcal{L}^{2.585} \cdot \log(\mathcal{L})\right)$. The elements of $\mathfrak{S}$ are then used to compute the game value at every leaf and for every player; re-using the values $\mathfrak{S}$ multiple times is the main aspect of the precomputation algorithm, which helps to reduce overall complexity by sacrificing the memory storage.

\begin{remark}
Suppose during the period that the ensemble $\mathcal{T}$ is in production, the data distribution changes qualitatively. 
In such situations, assuming the same tree structures are kept, 
one can first update the leaf weights, and then redo the precomputation step of Algorithm \ref{algorithm} to update the look-up tables accordingly. 
\end{remark}

\begin{remark}\label{categorical}
By incorporating generalizations of Theorem \ref{Catboost theorem} from
Appendix \ref{appendix:generalization},
Algorithm \ref{algorithm} can be modified to estimate those marginal feature attributions of CatBoost models that stem from a larger class of game values, or from coalitional game values such as the Owen value. The Owen values are indeed relevant since they can be used to recover the Shapley values of  categorical features that were one-hot encoded before the training. We elaborate on this in  
Proposition \ref{encoded Proposition}. This is one of the rare cases where Shapley values can be recovered after a non-linear transformation of features. It is infeasible to do something similar for the other highly complex built-in encodings of categorical features in the CatBoost 
library.     
\end{remark}

The last theorem of the section bounds the $L^2$-error of the algorithm when it is applied to a dataset $\mathbf{D}$ drawn i.i.d. from $\Bbb{R}^n$ according to ${\rm{P}}_{\mathbf{X}}$. 
In that setting, the estimated probabilities $\hat{p}(\mathbf{u};T)$, now treated as random variables, converge to the true probabilities $p(\mathbf{u};T)$ as $|\mathbf{D}|\to\infty$ (see Lemma \ref{error}).
Another quantity related to these probabilities that comes up in our error term is $1$ minus the Gini impurity; that is, the following quantity between zero and one associated with a tree:
\begin{equation}\label{Gini}
{\rm{Gini}}(\mathbf{X},T):=\sum_{\mathbf{u}\in\mathcal{R}(T)}p(\mathbf{u};T)^2.  
\end{equation}
In other words, ${\rm{Gini}}(\mathbf{X},T)$  is the sum of squares of probabilities that ${\rm{P}}_{\mathbf{X}}$ 
assigns to the regions cut by the decision tree $T$.

\begin{theorem}\label{error analysis}
Consider $(\mathbf{X},f)$ where $\mathbf{X}=(X_1,\dots,X_n)$ are the predictors and the model $f:\Bbb{R}^n\rightarrow\Bbb{R}$ is implemented via an ensemble $\mathcal{T}$ of oblivious trees. 
Fix $i\in N$ and denote the subset of trees that split on $X_i$ by $\mathcal{T}^{(i)}$. 
Suppose the dataset in Algorithm \ref{algorithm} is a random sample 
${\bf D}:=\left\{\mathcal{X}^{(1)},\dots,\mathcal{X}^{(\mathscr{D})}\right\}$ drawn i.i.d. from $\Bbb{R}^n$ according to ${\rm{P}}_{\mathbf{X}}$; hence the (unbiased) estimator 
$\widehat{ \varphi_i\big[v^\ME\big]}({\bf x}; {\bf D})$ of the $i^{\rm{th}}$ marginal Shapley value becomes a random variable. 
Then for ${\rm{P}}_{\mathbf{X}}$-almost every $\mathbf{x}\in\Bbb{R}^n$ one has 
\begin{equation}\label{inequality}
\sqrt{\Bbb{E}\left[\left|\widehat{\varphi_i\big[v^\ME\big]}(\mathbf{x}; {\bf D})
-\varphi_i\big[v^\ME\big](\mathbf{x})\right|^2\right]}
\leq \frac{C}{\sqrt{|\mathbf{D}|}}\cdot\max_{T\in\mathcal{T}^{(i)}} 
\sqrt{\sum_{\mathbf{b}\in\mathcal{R}(T)\subseteq\{0,1\}^{m(T)}}c(\mathbf{b};T)^2}
\end{equation}
where, on the right-hand side, the second term is the maximum possible value for the
$L^2$-norm of the vector formed by the leaf scores\footnote{It is more precise to say ``relevant'' leaf scores because the leaves corresponding to degenerate regions are ignored.} of a tree from $\mathcal{T}$ which splits on $X_i$;
and  in the first term, we have the constant 
\begin{equation}\label{Lipschitz constant}
C:=4\,\left|\mathcal{T}^{(i)}\right|
\cdot\max_{T\in\mathcal{T}^{(i)}}
\sqrt[4]{{\rm{Gini}}(\mathbf{X},T)}\cdot
\sqrt[4]{\frac{1.5}{k(T)}\left(1+\frac{m(T)}{k(T)}\right)^{k(T)}}
\end{equation}
which is no larger than 
$4\,|\mathcal{T}|\cdot \sqrt[4]{\frac{3\mathcal{L}}{\log_2(\mathcal{L})}}$ 
with $\mathcal{L}$ being the maximum possible number of leaves.
In particular, for a fixed $f$,
$\widehat{\varphi_i\big[v^\ME\big]}(\mathbf{x};{\bf D}) \to \varphi_i\big[v^\ME\big](\mathbf{x})$ in $L^2(\Omega, \mathcal{F},\Bbb{P})$ as the size  of $\mathbf{D}$ tends to infinity.
\end{theorem}

See Appendix \ref{subappendix:CatBoost} for a proof.


\begin{remark}
Recall that the estimations  of Shapley values generated by the algorithm above coincide 
with the Shapley values of the empirical marginal game defined based on the training set, i.e. $\widehat{\varphi_i\big[v^\ME\big]}({\bf x};{\bf D})=\varphi_i\big[\hat{v}^\ME(\cdot; {\bf D}, f)({\bf x})\big]$ in the above setting. 
Thus, due to \eqref{empirical},  the estimator is an average of i.i.d. random variables, 
making Theorem \ref{error analysis} essentially a consequence of the Strong Law of Large Numbers \cite{koralov2007}.
\end{remark}

\begin{remark}
Theorem \ref{error analysis} establishes strong convergence of the estimator, which implies convergence in probability \cite{koralov2007}. In particular, \eqref{inequality} and the Chebyshev inequality \cite{koralov2007} imply for any threshold $t>0$ that
\begin{equation}\label{cheb}
\Bbb{P} \Big( |\widehat{\varphi_i\big[v^\ME\big]}({\bf x};{\bf D})-\varphi_i[v^\ME]({\bf x})| \geq t \Big) \leq \frac{
{\rm{Var}}(\widehat{\varphi_i\big[v^\ME\big]}({\bf x},{\bf D}))}{t^2} \leq \frac{1}{t^2} \cdot \frac{c}{k},
\end{equation}
where $c>0$ is the multiplier of $1/\sqrt{|{\bf D}|}$ on the right-hand side of \eqref{inequality}. Thus, given any  $\epsilon>0$, one can always pick the number $k_{\epsilon}>0$ using the bound \eqref{cheb} such that the estimator based on ${\bf D}$ falls in the interval 
$(\varphi_i({\bf x})-\epsilon,\varphi_i({\bf x})+\epsilon)$  
with probability at least $1-\epsilon$ whenever $|{\bf D}|\geq k_{\epsilon}$.
\end{remark}

\begin{remark}
To get from the Lipschitz constant \eqref{Lipschitz constant} to the cruder one $4\,|\mathcal{T}|\cdot \sqrt[4]{\frac{3\mathcal{L}}{\log_2(\mathcal{L})}}$, the term $\sqrt[4]{{\rm{Gini}}(\mathbf{X},T)}$ was replaced with $1$. But the reader should keep in mind that, in principle, ${\rm{Gini}}(\mathbf{X},T)$ can be as small as $\frac{1}{\mathcal{L}}$. 
\end{remark}

\begin{remark}
The inequality \eqref{inequality} holds for ${\rm{P}}_{\mathbf{X}}$-a.e. $\mathbf{x}\in\Bbb{R}^n$. 
It thus yields an upper bound for the 
\textit{mean integrated squared error} (MISE) of the estimator $\widehat{\varphi_i\big[v^\ME\big]}$ as well.  
\end{remark}


\section{Numerical experiments}\label{sec:experiments}
This last part is devoted to our experiments with data. 
The claimed complexity of the algorithm (see Table \ref{Tab: complexity}) is verified in 
\Sec \ref{subsec:experiments3} through working with a synthetic dataset; and its performance is moreover benchmarked with TreeSHAP (cf. Figure \ref{fig::basic_info} and Figure \ref{fig::complexity_comparison}).
Next, in \Sec \ref{subsec:experiments1}, we train CatBoost, XGBoost and LightGBM models on four public datasets. We compute and compare various quantities pertinent  to the structure of trees 
(cf. Table \ref{Tab: topology}) to showcase some of the points previously made. 
Next, in \Sec \ref{subsec:experiments2}, we apply Algorithm \ref{algorithm} to the four CatBoost models trained in the preceding  section to compute their marginal Shapley values, and we record the time 
(cf. Table \ref{Tab: time}).     

We refer the reader to Appendix \ref{appendix:code and data} for more on the data and codes used for this section.

\subsection{Complexity  of Algorithm \ref{algorithm}}\label{subsec:experiments3}
\begin{figure}
  \centering
  \begin{subfigure}[t]{0.32\textwidth}
    \centering
    \includegraphics[width=\textwidth]{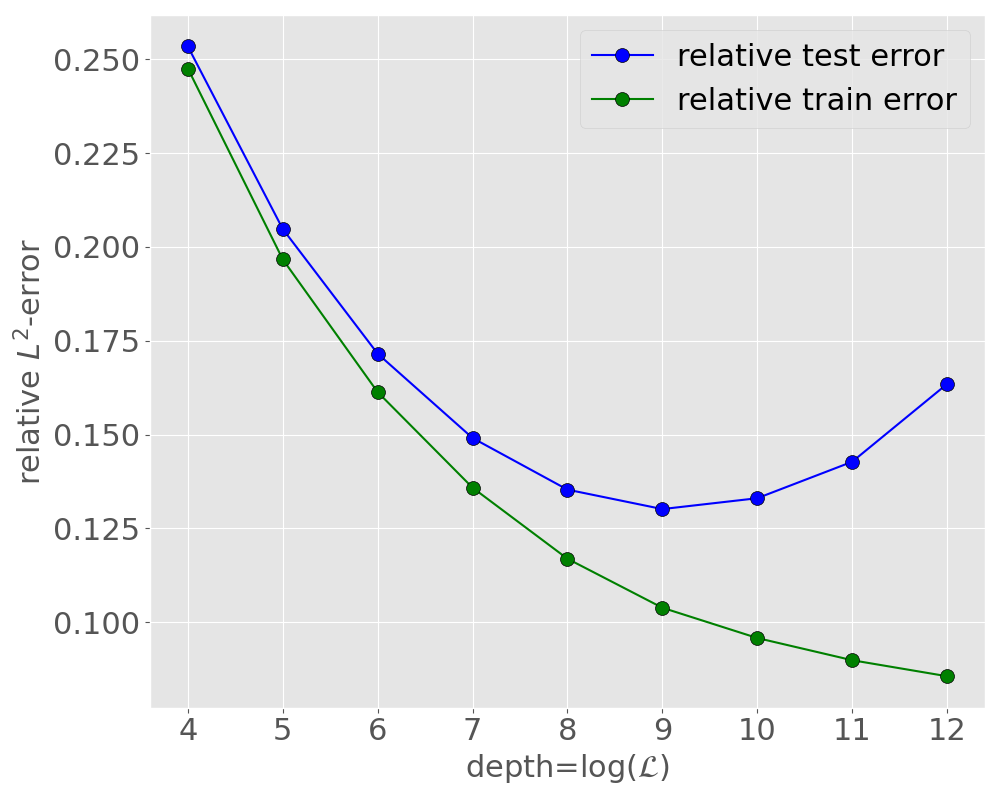} \caption{Training and test errors}\label{fig::ml_errors}
  \end{subfigure}
  \begin{subfigure}[t]{0.32\textwidth}
    \centering
    \includegraphics[width=\textwidth]{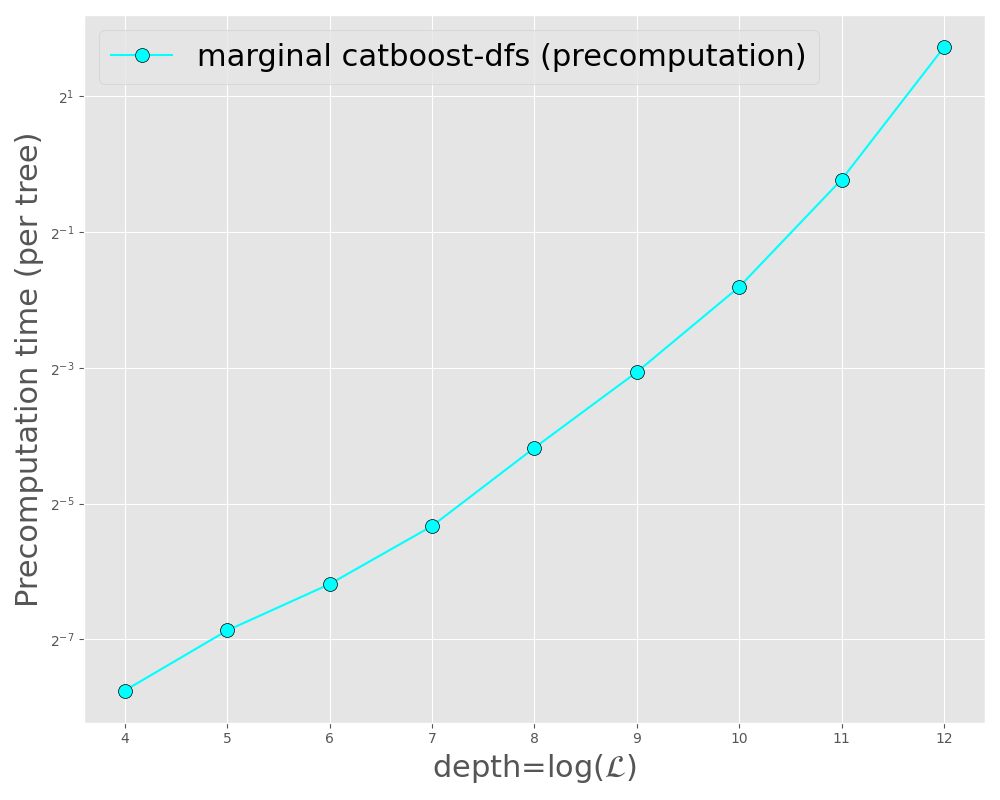}\caption{Precomputation time per tree}\label{fig::precomp_time} 
  \end{subfigure}
  \begin{subfigure}[t]{0.32\textwidth}
    \centering
    \includegraphics[width=\textwidth]{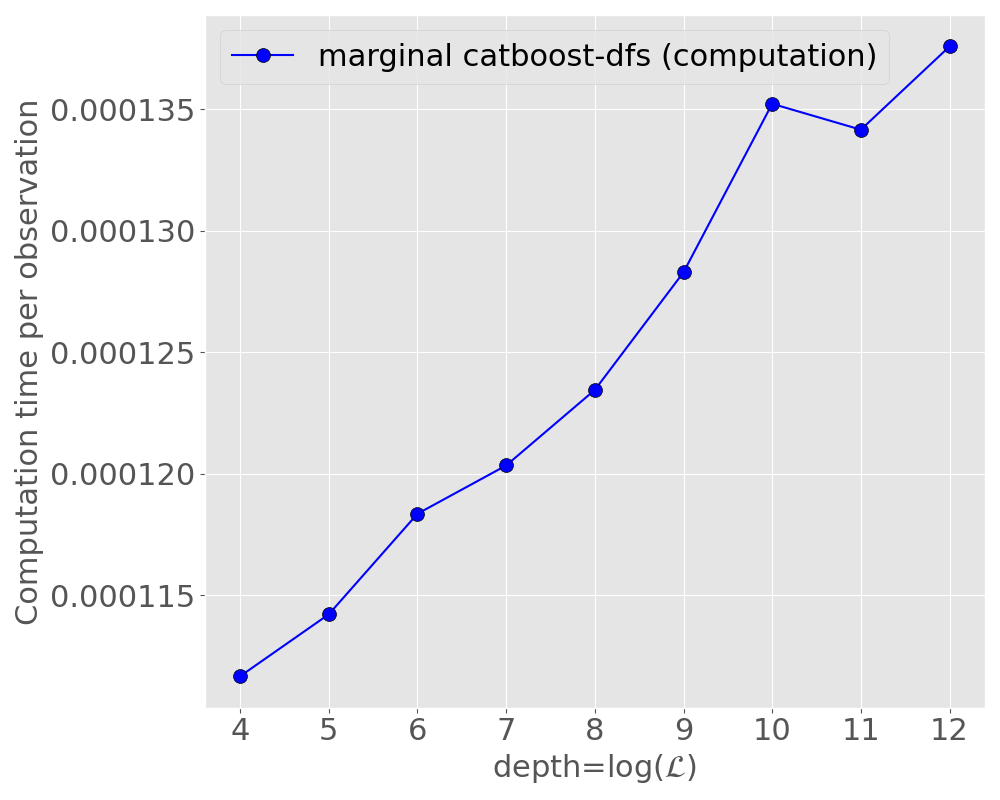} \caption{Computation time per observation} \label{fig::comp_time} 
  \end{subfigure}
  \caption{The execution times for the two steps of Algorithm \ref{algorithm} are depicted for CatBoost models of various depths which were trained on synthetic data \eqref{synth_data_model} for our experiment in \Sec \ref{subsec:experiments3}. 
  The plot on the left illustrates the training and test errors for these models. The one in the middle shows the average time it took to precompute the Shapley values for a tree from the ensemble in the logarithmic scale. 
  Finally, the last plot captures the on-the-fly computation time for obtaining the Shapley values of 1,000 random data point based on the precomputed tables.}\label{fig::basic_info} 
\end{figure}
\begin{figure}
  \centering
  \begin{subfigure}[t]{0.45\textwidth}
    \centering
    \includegraphics[width=\textwidth]{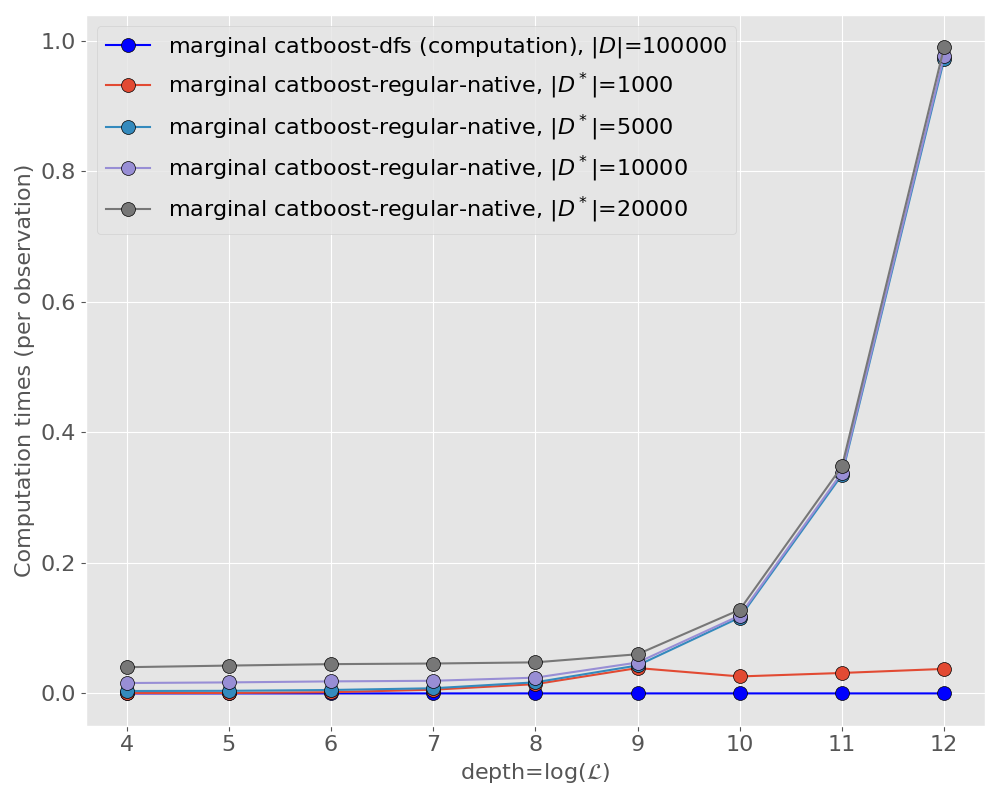} \caption{Computation time comparison}\label{fig::comp_time_all}
  \end{subfigure}
  \begin{subfigure}[t]{0.45\textwidth}
    \centering
    \includegraphics[width=\textwidth]{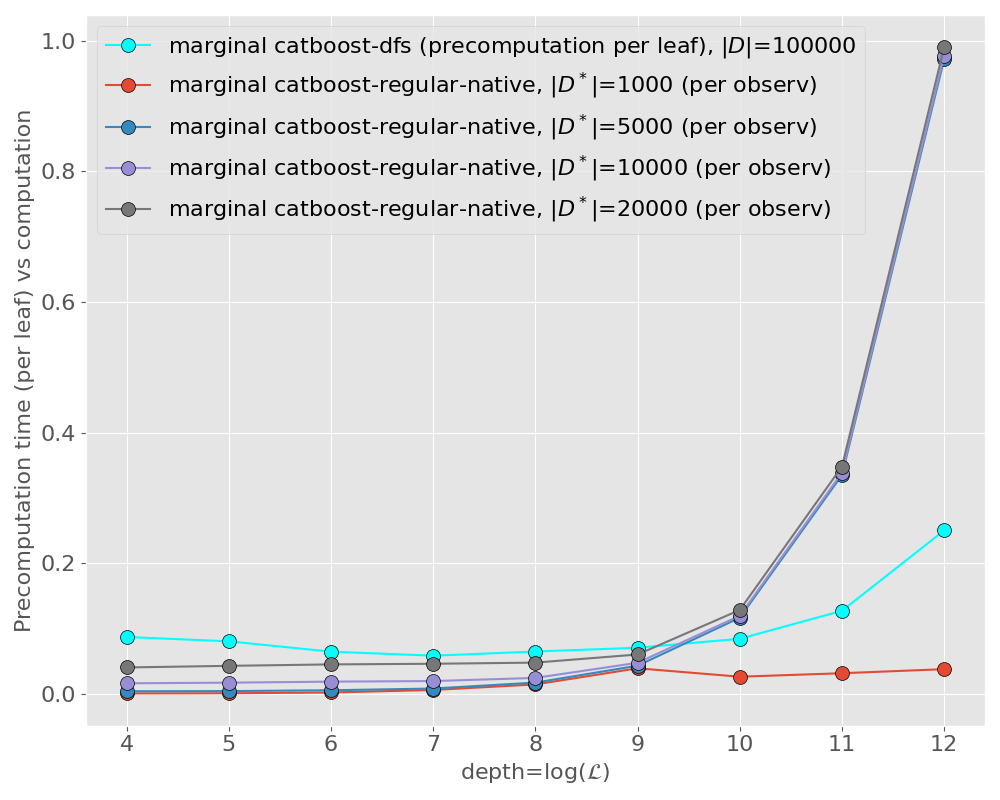}\caption{Precomputation time per leaf vs. computation time}\label{fig::precomp_time_vs_computation} 
  \end{subfigure}
  \caption{The execution times are plotted for  the interventional TreeSHAP---using the CatBoost's built-in method \texttt{get\_feature\_importance(shap\_calc\_type="Regular")} and background datasets $D_*$ of various sizes---and the two steps of our proposed algorithm \ref{algorithm}---which requires no background dataset and its accuracy is dictated by the training set $D$---once applied to CatBoost models trained for \Sec \ref{subsec:experiments3}. The plots confirm the complexity analysis outlined in Table 
  \ref{Tab: complexity}.}\label{fig::complexity_comparison} 
\end{figure}
In this section, we use a synthetic dataset to investigate the complexity of Algorithm \ref{algorithm}, both precomputation and computation steps, in terms of the maximum number of leaves in a CatBoost ensemble. \\
\indent 
Suppose  the predictors $\mathbf{X}=(X_1,X_2,\dots, X_n)$ obey the normal distribution     
$\mathcal{N}(0,I_n)$ where $I_n$  is the $n\times n$ identity matrix. The model for the response variable is assumed to be:
\begin{equation}\label{synth_data_model}
Y = \sum_{i=1}^{n}a_i \cdot X_i + \sum_{1\leq i<j\leq n} b_{ij}X_i \cdot X_j + \epsilon
\end{equation}
where $a_i$  and $b_{ij}$  are coefficients in intervals $(1,5)$ and $(-0.5,0.5)$ chosen randomly (once); and $\epsilon \sim \mathcal{N}(0,\delta)$, with $\delta=0.05$, is a noise term.\\
\indent
For our numerical experiment, we generated $100,000$ samples drawn from the data generating model \eqref{synth_data_model} with $n=40$. We then trained a collection of CatBoost regressors $\hat{f}_1,\hat{f}_2,\dots,\hat{f}_9$   where the \texttt{max\_depth} parameter is set to be $4,5,\dots,12$ respectively (which in turn yields ensembles with the maximum number of leaves, $\mathcal{L}$, being $2^4,2^5,\dots,2^{12}$ respectively). Other tuned training hyperparameters are: \texttt{n\_estimators}$=300$, \texttt{subsample}$=0.8$, \texttt{learning\_rate}$=0.1$. A subset of $50,000$ samples drawn from \eqref{synth_data_model} was used as the test set.
The relative $L^2$-errors for training and test sets are depicted on Figure \ref{fig::ml_errors}.\\
\indent
We then carried out Step 1 of Algorithm \ref{algorithm} on an AWS machine with 32 cores and  256GB of memory. This step precomputes marginal Shapley values at every realizable leaf of each tree, relying only on the internal parameters of trained models, and without taking in any background dataset. For these nine models, the average precomputation time, per tree in the ensemble, is plotted in Figure \ref{fig::precomp_time} in the logarithmic scale. Recall from \Sec \ref{subsec:CatBoost} that the theoretical complexity of the precomputation algorithm should be $O\left(\mathcal{L}^{\log_2 6}\cdot\log(\mathcal{L})\right)$; compare with Table \ref{Tab: complexity}. The true complexity observed on Figure \ref{fig::comparison_native}, however, can be expressed in the form $c \cdot \mathcal{L}^p \cdot \log(\mathcal{L})$ where $c$ depends on the depth of the tree and $p$ varies from $1.6$ for small depths to $1.8$ for depth equal to $12$. As the depth increases, we expect that the complexity curve would tend to a straight line with the slope $\log_2 6 \approx 2.5849$.\\
\indent
We remind the reader that the precomputation occurs only once. Its results, the (local) game values associated with every realizable leaf of each  tree in the ensemble, are saved in look-up tables. 
These are then used in Step 2 (on-the-fly computation step) of Algorithm \ref{algorithm} to output the vector of Shapley values for any input data point. 
We performed this step for $1,000$  randomly picked samples of the training set for each of the models.  The average computation time per observation for these nine models is depicted on Figure \ref{fig::comp_time}. As expected from Table \ref{Tab: complexity}, the complexity is linear in terms of  $\log(\mathcal{L})$.\\
\indent 
We next compare  Algorithm \ref{algorithm} with the interventional TreeSHAP algorithm which requires a background dataset as an input.  Background datasets $D_*$ of sizes $1,000$, $5,000$, $10,000$, and $20,000$ were used for this part. Unlike \cite{lundberg2020local}, we avoid using small background datasets with a few hundred samples since, for them, the  relative statistical error can be quite high (for details, see \cite{2023arXiv230310216K}).
To estimate marginal (interventional) Shapley values for them, the \texttt{get\_feature\_importance} method of CatBoost, with the default setting \texttt{shap\_calc\_type="Regular"}, was used. 
According to the documentation, this implements the interventional TreeSHAP algorithm \cite{lundberg2020local}.\footnote{Check \url{https://catboost.ai/en/docs/concepts/python-reference_catboost_get_feature_importance}.}
Figure \ref{fig::comp_time_all} illustrates the average computation time of our inherently interpretable algorithm (Step 2 of Algorithm \ref{algorithm}) compared to the interventional TreeSHAP with the aforementioned background datasets.
As expected from Table \ref{Tab: complexity}, the latter are exponential in terms of $\log(\mathcal{L})$ whereas the former is linear, and thus much faster. \\
\indent
As for the precomputation time of our algorithm (Step 1 of Algorithm \ref{algorithm}), comparison with the execution times of interventional TreeSHAP is, in general, difficult. However, renormalizing the precomputation time  on the average number of leaves of trees in the ensemble is, in principle, comparable to the computation times per observation of the on-the-fly algorithms. 
For ensembles $\hat{f}_1,\hat{f}_2,\dots,\hat{f}_9$, Figure \ref{fig::precomp_time_vs_computation} compares the precomputation times per leaf (that is, the precomputation time divided by the average number of leaves in any given tree) with the computation times of algorithms that require background datasets. We note that the accuracy of Algorithm \ref{algorithm} corresponds to the empirical marginal Shapley value based on the training set which has $100,000$ samples, much larger than the background datasets. Meanwhile, our precomputation  per leaf is only a few times slower than the on-the-fly computation by the interventional TreeSHAP.\\
\indent 
Finally, after the first draft of this paper appeared on arXiv, we discovered that the CatBoost method \texttt{get\_feature\_importance} allows for the argument \texttt{shap\_calc\_type} to be set to \texttt{"Exact"} and \texttt{shap\_mode} to be set to \texttt{"UsePreCalc"}. It appears that 
these options allow for the precomputation at every leaf of each tree, and subsequently allow for the evaluation of explanations at every given observation based on the precomputed information.
No detailed documentation on what these options are, or error analysis (such as Theorem \ref{error analysis}) is available at the time of writing. Furthermore, the two steps cannot be run independently, which  unfortunately limits the use of the \texttt{"UsePreCalc"} option.
Nevertheless, to make our treatment comprehensive, we compared the precomputation time of our algorithm with the native CatBoost one (where we simply provided a single explicand). 
Figure \ref{fig::comparison_native} illustrates the time comparison in terms of the number of leaves, both with and without the logarithmic scale. Notice that ours (Step 1 of Algorithm \ref{algorithm}) is faster; and its complexity is linear in the logarithmic scale in accord with our theoretical analysis (cf. Table \ref{Tab: complexity}). 
The plot hence shows that the two methods are different.
At last, it must be mentioned that our proposed algorithm can be generalized to any other game value characterized by Theorem \ref{classification}, or even to certain coalitional game values such as the Owen value; see the results of Appendix \ref{appendix:generalization}.
To the best of our knowledge, no competitor algorithm presents a similar option.
Indeed, we repeated our experiments for generating the marginal Shapley values for CatBoost models trained on synthetic data \eqref{synth_data_model}, this time for the marginal Owen values, based on a proprietary code that computes these values in a two-stage procedure similar to Algorithm \eqref{algorithm}. Figure \ref{fig::basic_info_Owen} illustrates computation and precomputation times.  
\begin{figure}
  \centering
  \begin{subfigure}[t]{0.45\textwidth}
    \centering
    \includegraphics[width=\textwidth]{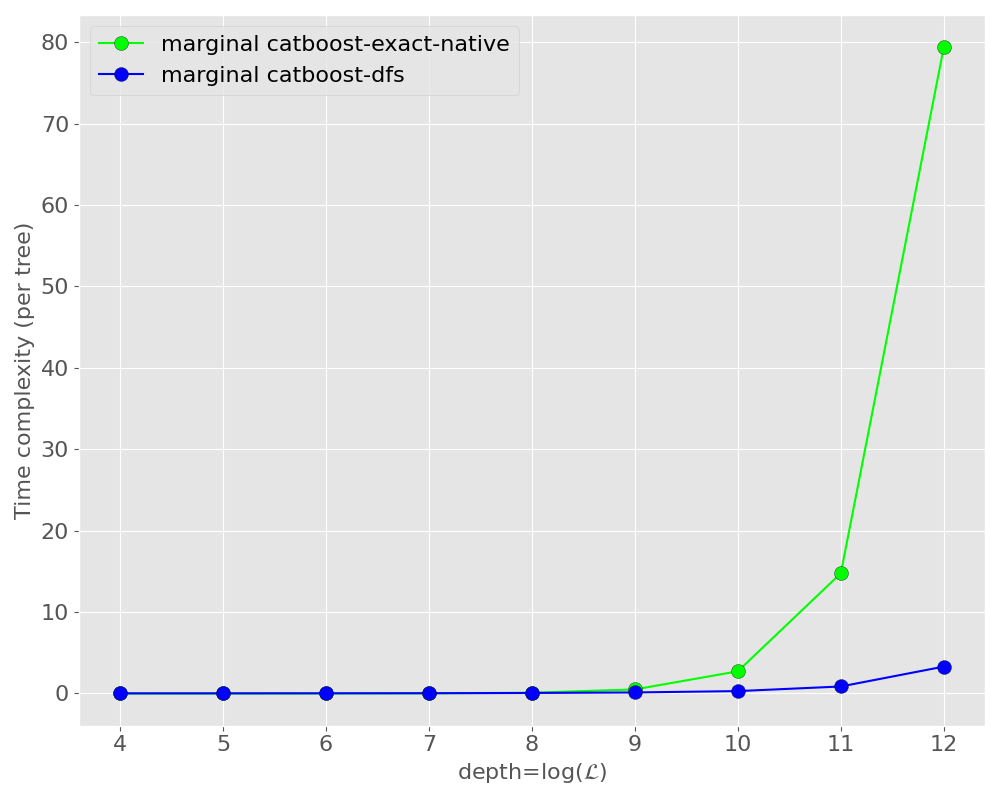} \caption{Average time required for a tree: precomputation vs.  a native method of CatBoost}
  \end{subfigure}  
  \begin{subfigure}[t]{0.45\textwidth}
    \centering
    \includegraphics[width=\textwidth]{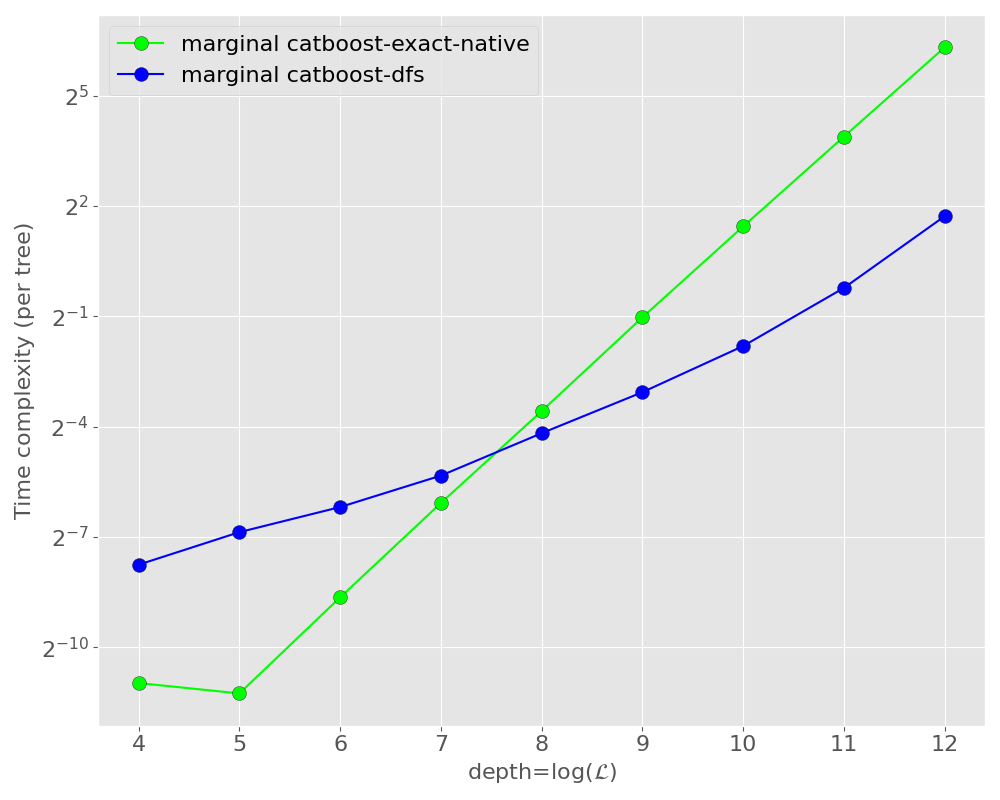} \caption{Average time required for a tree: precomputation vs.  a native method of CatBoost (logarithmic scale)}
  \end{subfigure}  
  \caption{The figure benchmarks the precomputation time per tree for our algorithm (Step 1 of Algorithm \ref{algorithm}) with the per tree execution time for the CatBoost's built-in method \texttt{get\_feature\_importance(shap\_calc\_type="Exact",shap\_mode="UsePreCalc")} in both standard and logarithmic scales.}\label{fig::comparison_native} 
\end{figure}

\begin{figure}
  \centering
  \begin{subfigure}[t]{0.45\textwidth}
    \centering
    \includegraphics[width=\textwidth]{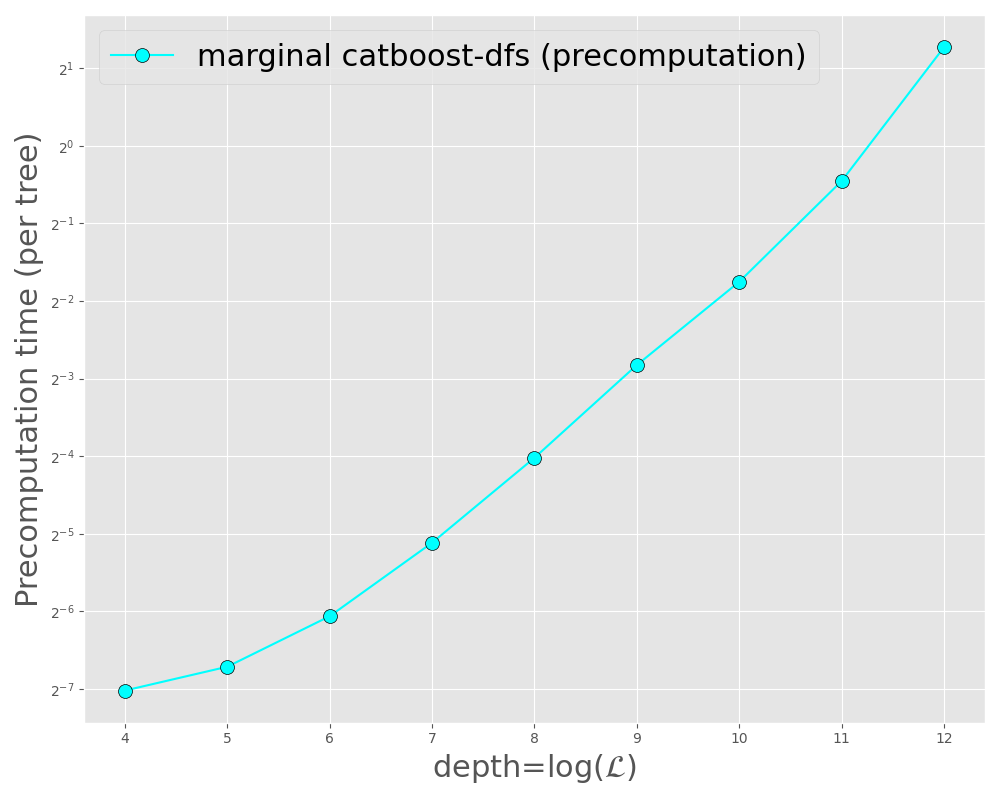}\caption{Precomputation time per tree for Owen values}\label{fig::precomp_time_Owen} 
  \end{subfigure}
  \begin{subfigure}[t]{0.45\textwidth}
    \centering
    \includegraphics[width=\textwidth]{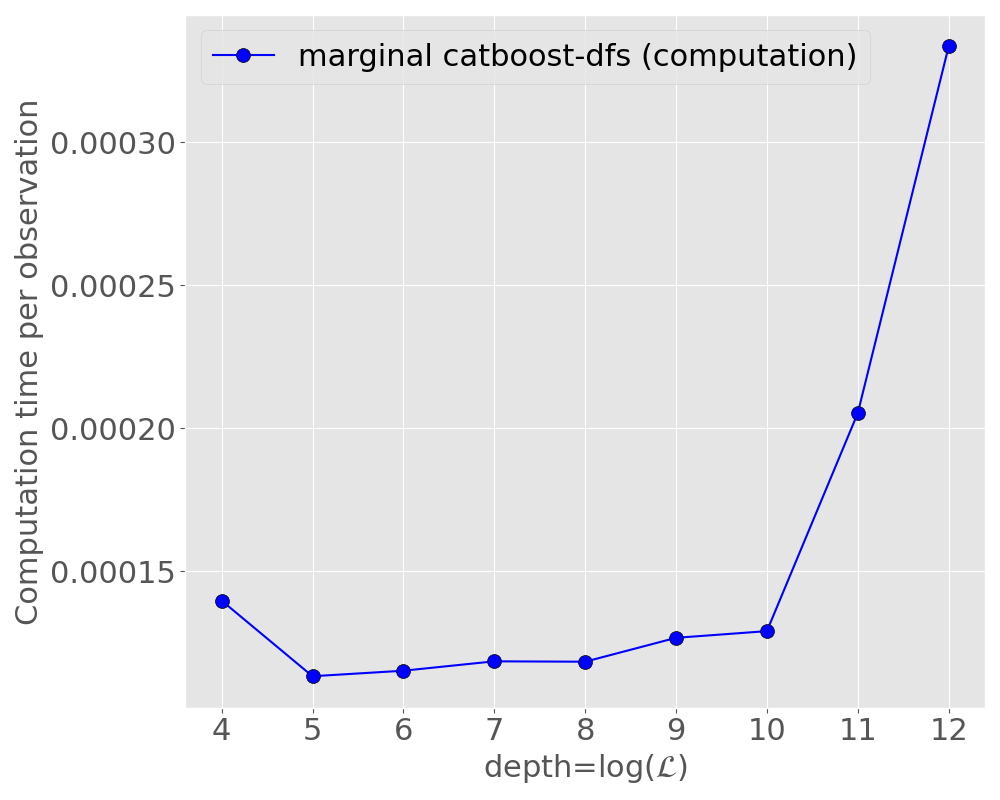} \caption{Computation time per observation for Owen values} \label{fig::comp_time_owen} 
  \end{subfigure}
  \caption{An analog of Figure \ref{fig::basic_info} for a generalization of Algorithm \ref{algorithm} to the case of marginal Owen values based on formula \eqref{symmetric formula tree Owen}.
   The plots show Owen values precomputation and computation times for CatBoost ensembles with \texttt{max\_depth} varying from  $4$ to $12$ trained on synthetic data \eqref{synth_data_model} with a fixed underlying partition of the $n=40$ predictors into $8$ groups of size $5$.
   The first plot is in the logarithmic scale.}\label{fig::basic_info_Owen} 
\end{figure}

\subsection{Experiments with public datasets}\label{subsec:experiments1}
For this section, we trained models on public datasets through libraries that construct trees very differently. 
The datasets used are the Superconductivity dataset \cite{misc_superconductivty_data_464}, 
the Ailerons dataset \cite{misc_ailerons_data}, the Online News Popularity dataset \cite{misc_online_news_popularity_332}, 
and the Higgs dataset \cite{misc_higgs_280}. Table \ref{Tab: Datasets} outlines the task, the number of variables and the sizes for training, validation and test sets for each of them. The features are non-categorical in all cases. See Appendix \ref{appendix:code and data} for descriptions of these datasets. \\
\indent
For each dataset, we trained a CatBoost, a LightGBM and an XGBoost model with the libraries' default growth policies. 
Their performances over the test set, recorded in Table \ref{Tab: metrics}, indicate that 
 each model enjoys a reasonable predictive power;
and their outputs are much closer to each other than to the ground truth.
But the models' internal structures differ considerably in the sense that their constituent trees are topologically very different. This is observed in Table \ref{Tab: topology} where we disentangle each ensemble to obtain quantities such as the average number of leaves, the average number of distinct features per tree etc.  
Based on the table, and following the same notation, we observe that:
\begin{enumerate}
\setlength\itemsep{0.3em}
\item Comparing the number of leaves $\ell(T)$ with its maximum possible value $2^{m(T)}$ indicates that, overall, 
the trees in LightGBM ensembles are sparse while those in CatBoost ensembles are prefect binary (as expected);
XGBoost ensembles lie some place in between.

\item On average,  the number of distinct features $k(T)$ is smaller than the maximum possible number of features appearing in the tree $T$, i.e. the depth $m(T)$ in the case of CatBoost models and the number of internal nodes $\ell(T)-1$ for  
trees from XGBoost and LightGBM ensembles.\footnote{In training, the number of features used for constructing a tree was not limited by any hyperparameter (i.e. a hyperparameter such as  \texttt{colsample\_bytree} in XGBoost and LightGBM was not used).}
This  indicates that  trees occasionally split more than once with respect to some feature (or some feature appears on more than one level in the case of CatBoost).

\item Following the discussion in \Sec \ref{subsec:observations}, for game values with null-player and carrier-dependence properties such as Shapley and Banzhaf, only the trees that split on a given feature matter for computing its marginal attribution. Furthermore, the contribution of each of those trees to the feature attribution under consideration depends only on the distinct features appearing in that tree. Now, revisiting Table  
\ref{Tab: topology}, we first observe that $\overline{\left|\mathcal{T}^{(i)}\right|}$ is often much smaller than $|\mathcal{T}|$; namely a feature appears only in a portion of trees from the ensemble. Secondly, the total number of variables, $n$, is much larger than the average number of distinct features per tree $\overline{k(T)}$.  
All in all, we observe that $\overline{\left|\mathcal{T}^{(i)}\right|}\cdot 2^{\overline{k(T)}}\ll 2^n$, 
where $2^n$ is the complexity of naively computing a game value such as Shapley for $n$ players.  
This reinforces one of this paper's key points: 
For tree ensembles, the complexity of computing marginal feature attributions such as Shapley, Banzhaf or Owen  can be lowered through taking individual trees and their structures into account; see \eqref{fewer features}.

\item  Recall that Theorem \ref{main theorem} puts forward a method for estimating the vector of marginal Shapley values of a CatBoost model at an explicand. These are generated based on the whole training set and with complexity $O\left(|\mathcal{T}|\cdot \mathcal{L}^{\log_2 3}\cdot\log(\mathcal{L})\right)$ 
where $\mathcal{L}$ is the maximum possible number of leaves (compare with Table 
\ref{Tab: complexity}). 
To get estimations with a comparable accuracy from the interventional TreeSHAP, take the 
training set as the background dataset. According to \eqref{interventional complexity}, 
the complexity would then become $O(|\mathcal{T}|\cdot \mathcal{L} \cdot |D|)$. 
This is a much higher complexity:  Table \ref{Tab: topology} indicates that,
for CatBoost models we trained on datasets from Table \ref{Tab: Datasets}, 
the ratio $\frac{|D|}{\mathcal{L}^{\log_2 3-1}\cdot\log(\mathcal{L})}$ is always larger than $40$.
\end{enumerate}

\begin{table}[ht]
\begin{tabular}{|c|c|c|c|c|c|}
\hline
 Dataset & Task & Features &  Train & Validation & Test\\
\Xhline{2pt}
 \href{https://archive.ics.uci.edu/ml/datasets/superconductivty+data}{Superconductivity} \cite{misc_superconductivty_data_464}& Regression &
81 & 12,757 & 4,253 & 4,253\\
\hline 
 \href{https://www.openml.org/search?type=data&status=active&id=296}{Ailerons} 
\cite{misc_ailerons_data} & Regression &
40 & 5,723 & 1,431 &  6,596\\
\hline 
 \href{https://archive.ics.uci.edu/ml/datasets/online+news+popularity}{Online News} 
\cite{misc_online_news_popularity_332} & Classification &
58 & 23,786 & 7,929 &  7,929\\
\hline 
 \href{https://archive.ics.uci.edu/ml/datasets/HIGGS}{Higgs} 
\cite{misc_higgs_280} & Classification &
28 & 10,000,000 & 500,000 &  500,000\\
\hline
\end{tabular}
\caption{Datasets used for experiments in \Sec \ref{sec:experiments}.}
\label{Tab: Datasets}
\end{table}

\begin{table}[ht]

\begin{tabular}{|c|ccc|caccc|c|c|}
\cline{2-9}
\multicolumn{1}{c|}{}&\multicolumn{4}{ca}{Superconductivity} & \multicolumn{4}{c|}{Ailerons} 
\\ \clineB{2-9}{5}
\multicolumn{1}{c|}{}&$\mathbf{y}_{\rm{cat}}$ & $\mathbf{y}_{\rm{lgbm}}$ & $\mathbf{y}_{\rm{xgb}}$ & $\mathbf{y}_{\rm{test}}$ &
$\mathbf{y}_{\rm{cat}}$ & $\mathbf{y}_{\rm{lgbm}}$ & $\mathbf{y}_{\rm{xgb}}$ & $\mathbf{y}_{\rm{test}}$
\\ \hline
$\mathbf{y}_{\rm{cat}}$ & \cellcolor{teal} 1 &  \cellcolor{teal} 0.9905 & \cellcolor{teal} 0.9788 & \cellcolor{teal} 0.9206 & \cellcolor{teal} 1 & \cellcolor{teal} 0.9678 &  \cellcolor{teal} 0.9347 & \cellcolor{teal} 0.8297\\ 
$\mathbf{y}_{\rm{lgbm}}$ & \cellcolor{teal} 0.9904 &  \cellcolor{teal} 1 &  \cellcolor{teal} 0.9814 & \cellcolor{teal} 0.9212 & \cellcolor{teal} 0.9694 &  \cellcolor{teal} 1 &  \cellcolor{teal} 0.9499 & \cellcolor{teal} 0.8336\\
$\mathbf{y}_{\rm{xgb}}$ & \cellcolor{teal} 0.9780 &  \cellcolor{teal} 0.9809 & \cellcolor{teal} 1 & \cellcolor{teal} 0.9166 & \cellcolor{teal} 0.9409 &  \cellcolor{teal} 0.9524 & \cellcolor{teal} 1 & \cellcolor{teal} 0.8003 \\
\hline
\end{tabular}

\vspace*{0.5 cm}

\begin{tabular}{|c|ccc|caccc|c|c|}
\cline{2-9}
\multicolumn{1}{c|}{}&\multicolumn{4}{ca}{Online News} & \multicolumn{4}{c|}{Higgs} 
\\ \clineB{2-9}{5}
\multicolumn{1}{c|}{}&$\mathbf{y}_{\rm{cat}}$ & $\mathbf{y}_{\rm{lgbm}}$ & $\mathbf{y}_{\rm{xgb}}$ & $\mathbf{y}_{\rm{test}}$ &
$\mathbf{y}_{\rm{cat}}$ & $\mathbf{y}_{\rm{lgbm}}$ & $\mathbf{y}_{\rm{xgb}}$ & $\mathbf{y}_{\rm{test}}$
\\ \hline
$\mathbf{y}_{\rm{cat}}$ & \cellcolor{teal} 1 &  \cellcolor{teal} 0.8695  &  \cellcolor{teal} 0.8324  &  \cellcolor{yellow} 0.7324 & \cellcolor{teal} 1 & \cellcolor{teal} 0.9420  & \cellcolor{teal} 0.9460   & \cellcolor{yellow} 0.8465 \\ 
$\mathbf{y}_{\rm{lgbm}}$ & \cellcolor{teal} 0.8786  &  \cellcolor{teal} 1 &  \cellcolor{teal} 0.8668  &  \cellcolor{yellow} 0.7328 & \cellcolor{teal} 0.9461  &  \cellcolor{teal} 1 &  \cellcolor{teal} 0.9533  &  \cellcolor{yellow} 0.8402 \\
$\mathbf{y}_{\rm{xgb}}$ & \cellcolor{teal} 0.8259  & \cellcolor{teal} 0.8512   & \cellcolor{teal} 1 & \cellcolor{yellow} 0.7238  & \cellcolor{teal} 0.9473  &  \cellcolor{teal} 0.9511  & \cellcolor{teal} 1 &  \cellcolor{yellow} 0.8409 \\
\hline
\end{tabular}

\caption{
The performance over the test set for the CatBoost, LightGBM and XGBoost models trained for regression tasks (top) and 
for classification tasks (bottom) are presented here. (See Table \ref{Tab: Datasets} for the underlying datasets.)  
Here, $\mathbf{y}_{\rm{test}}$ denotes the correct target values/labels and 
$\mathbf{y}_{\rm{cat}}$, $\mathbf{y}_{\rm{lgbm}}$, $\mathbf{y}_{\rm{xgb}}$ 
are the predicted vectors over the test set (predicted probabilities in the case of classification). 
Each cell captures a metric between the vectors associated with the corresponding row and column. The cells in teal denote $R^2$ scores while those in yellow denote AUC scores (only relevant for the classification tasks---the table on the bottom). In both tables, the predicted vectors $\mathbf{y}_{\rm{cat}}$, $\mathbf{y}_{\rm{lgbm}}$, $\mathbf{y}_{\rm{xgb}}$ 
are compared based on the $R^2$ score.
}
\label{Tab: metrics}

\end{table}

\begin{table}[ht]
\renewcommand{\arraystretch}{1.5}
\begin{tabular}{|c|c|c|c|c|c|c|c|c|c|}
\hline
Dataset & model type & $|\mathcal{T}|$ & $\overline{\left|\mathcal{T}^{(i)}\right|}$ & $\overline{\ell(T)}$ 
& $\overline{m(T)}$ 
& $\overline{k(T)}$  & $n$
\\ \Xhline{2pt}
\multirow{3}{*}{Superconductivity} &
CatBoost & 300 & 27.70 & 256$^*$  & 8$^*$ &   7.48  & \multirow{3}{*}{81}\\
\cline{2-7}
&LightGBM & 300 & 85.48 & 31$^*$  & 11.48  & 23.08  &  \\
\cline{2-7}
&XGBoost & 300 & 92.69 & 42.21  & 6$^*$&  25.03  & \\ \Xhline{1pt}
\multirow{3}{*}{Ailerons} &
CatBoost & 50 & 7.50 & 126.72 & 6.98   & 6.00  & \multirow{3}{*}{40}\\
\cline{2-7}
&LightGBM & 50 & 14.00 & 25$^*$ & 8.42   & 11.20  &  \\
\cline{2-7}
&XGBoost & 40 & 1.98 & 3.78 & 1.85   & 1.98   & \\ \Xhline{1pt}
\multirow{3}{*}{Online News} &
CatBoost &  167 & 19.86  &  128$^*$  &  7$^*$ &  6.90  & \multirow{3}{*}{58}\\
\cline{2-7}
&LightGBM & 88  & 32.78  &  31$^*$  & 10.41   &  21.60  &  \\
\cline{2-7}
&XGBoost & 88 & 26.98  &  27.09 & 5$^*$  &  17.78 & \\ \Xhline{1pt}
\multirow{3}{*}{Higgs} &
CatBoost &  1000 &  251.32 & 256$^*$ &  8$^*$ & 7.04 & \multirow{3}{*}{28}\\
\cline{2-7}
&LightGBM & 1000 & 636.68 & 60$^*$ & 13.18 & 17.83 &  \\
\cline{2-7}
&XGBoost & 1000 & 524.93 & 31.89 & 5$^*$ & 14.70 & \\ \hline

\end{tabular}
\captionsetup{singlelinecheck=off}
\caption[foo]{Various quantities associated with the ensembles trained on datasets from Table \ref{Tab: Datasets}. 
\begin{itemize}
\setlength\itemsep{0.5em}
\item $|\mathcal{T}|$: number of trees in the ensemble,\hspace{0.6cm} $\bullet$ $n$: the total number of features,
\item $\overline{\left|\mathcal{T}^{(i)}\right|}$: the average number of trees in the ensemble that split on a feature,
\item $\overline{m(T)}$: the average depth of trees,\hspace{1cm} $\bullet$ $\overline{\ell(T)}$: the average number of leaves,
\item $\overline{k(T)}$: the average number of distinct features per tree.
\end{itemize}
Asterisk indicates that the quantity whose average is under consideration does not vary across the ensemble.}
\label{Tab: topology}
\end{table}



\subsection{Applying Algorithm \ref{algorithm} to models from \Sec \ref{subsec:experiments1}}\label{subsec:experiments2}
For this section, an  implementation of Algorithm \ref{algorithm} was applied to the four CatBoost models from \Sec \ref{subsec:experiments1}. 
To this end, a proprietary, internal library based on Algorithm \ref{algorithm} (and its generalizations; cf. Appendix \ref{appendix:generalization}) was used. 
This library generates accurate estimations of marginal Shapley (as well as Banzhaf, Owen etc.)
feature attributions for ensembles of oblivious trees without reliance on any background dataset.\\
\indent
The execution time for the precomputation step, as well as  the time required for generating Shapley values for 100 random test samples,
are recorded in Table \ref{Tab: time}. Observe that the algorithm is very fast. 
Similar execution times for the Owen value are available from Table \ref{Tab: time Owen}.\\
\indent
The look-up tables constructed for these CatBoost models are available on \url{https://github.com/FilomKhash/Tree-based-paper}
along with a code which provides a sanity check via verifying the efficiency property of Shapley 
and Owen values. We refer the reader to Appendix \ref{appendix:code and data} for more on  our experiments and the relevant material.

\begin{table}[ht]
\small
\begin{tabular}{|c|c|c|c|c|c|c|}
\hline
 \shortstack{Dataset\\\phantom{a}} &  \shortstack{Features\\\phantom{a}} &  \shortstack{Trees\\\phantom{a}} & 
 \shortstack{\\Precomputation\\ (no multithreading)}  &  \shortstack{\\Precomputation\\ (with multithreading)}  
 & \shortstack{Total Size\\\phantom{a}} & \shortstack{Computation\\(100 samples)} \\
\Xhline{2pt}
 Superconductivity & 81 & 300 & 20.118s & 6.197s & 10.6MB & 0.030s \\
\hline 
Ailerons  & 40 &  50 & 0.600s &  0.367s & 808KB & 0.005s \\
\hline 
Online News  & 58 &  167 & 3.518s & 1.888s & 3.19MB & 0.017s\\
\hline 
Higgs & 28 & 1000 & 51.140s & 19.152s & 32.3MB & 0.068s \\
\hline
\end{tabular}
\normalsize
\caption{The execution times of an  optimized implementation of Algorithm \ref{algorithm} once applied to explain CatBoost models from \Sec \ref{subsec:experiments1} on an AWS machine with 16 cores and 128GB of memory. The algorithm first precomputes a look-up table for each decision tree. In each case, the total time required to construct all of the tables, both with or without multithreading, is recorded  along with the total size of these tables. The last column captures the time it took to compute marginal Shapley values for 100 randomly chosen test instances based on these look-up tables.}
\label{Tab: time}
\end{table}

{
\begin{table}[ht]
\footnotesize
\begin{tabular}{|c|c|c|c|c|c|c|c|}
\hline
 \shortstack{Dataset\\\phantom{a}} &  \shortstack{Features\\\phantom{a}} &  \shortstack{Trees\\\phantom{a}} & 
 \shortstack{Partition\\Size} &
 \shortstack{\\Precomputation\\ (no multithreading)}  &  \shortstack{\\Precomputation\\ (with multithreading)} 
 & \shortstack{Total \\Size} & \shortstack{Computation\\(100 samples)} \\
\Xhline{2pt}
 Superconductivity & 81 & 300 & 11 & 32.022s & 10.138s & 10.6MB  & 0.034s \\
\hline 
Ailerons  & 40 &  50 & 24 & 2.059s &  1.025s & 808KB  & 0.006s \\
\hline 
Online News  & 58 &  167 & 37 & 12.827s & 3.715s & 3.19MB  & 0.015s\\
\hline 
Higgs & 28 & 1000 & 27 & 252.005s & 61.908s & 32.2MB  & 0.068s \\
\hline
\end{tabular}
\normalsize
\caption{The analog of Table \ref{Tab: time} for computing marginal Owen values of CatBoost ensembles from Section \ref{subsec:experiments1} with a proprietary code based on formula \eqref{symmetric formula tree Owen}. The code was run on an AWS machine with 16 cores and 128GB of memory. For each dataset, the features were partitioned through a hierarchical clustering process based on a sophisticated measure of dependence developed in \cite{Reshef2015MeasuringDP}.}
\label{Tab: time Owen}
\end{table}
}

\appendix
\section{Availability of data and code}\label{appendix:code and data}
Here, we elaborate on the experiments that we carried out for this paper, and on the supplementary material which is available from \url{https://github.com/FilomKhash/Tree-based-paper}. 

The computations in Example \ref{main example} concerning path-dependent and interventional TreeSHAP for decision trees from Figure \ref{fig:failure1}  are confirmed in the notebook \texttt{TreeSHAP\_Sanity\_Check.ipynb}.  

In Sections \ref{subsec:experiments1},\ref{subsec:experiments2}, we deal with models trained on the following public datasets; also see Table \ref{Tab: Datasets}. 
\begin{enumerate}
\item \href{https://archive.ics.uci.edu/ml/datasets/superconductivty+data}{The Superconductivity dataset} \cite{misc_superconductivty_data_464}: 
a regression dataset where the 
superconductivity critical temperature should be predicted based on 81 features extracted from the superconductor’s chemical formula. The original dataset has 21,263 instances. We randomly split it into training, validation and test sets in 60:20:20 proportions.

\item \href{https://www.dcc.fc.up.pt/~ltorgo/Regression/ailerons.html}{The Ailerons dataset} \cite{misc_ailerons_data}: 
a regression dataset originating from a control problem for an F-16 aircraft where the control action on the ailerons of the aircraft should be predicted based on 40 features describing the status of the airplane. The original data comes in a test set with 6,596 instances and a training set with 7,154 instances. The latter was split into a smaller training set and a validation set in 80:20 proportions.

\item \href{https://archive.ics.uci.edu/ml/datasets/online+news+popularity}{The Online News Popularity dataset} \cite{misc_online_news_popularity_332}:
a dataset with features about articles published in a period of two years on the news website Mashable. We used it for a binary classification task where one should predict if an article has been shared at least  1,400 times in social networks or not. The prediction should be done based on 58 features (two of the features from the original dataset are non-predictive).
The data we obtained from the data source had 39,644 instances. We randomly split it into training, validation and test sets in 60:20:20 proportions.

\item \href{https://archive.ics.uci.edu/ml/datasets/HIGGS}{The Higgs dataset} \cite{misc_higgs_280}: 
a binary classification dataset where, based on 28 features that are functions of the kinematic properties, one should distinguish between a signal process which produces Higgs bosons and a background process which does not. 
The original dataset has 11,000,000 instances; and the data source states that the last 500,000 instances should be used as the test set. We randomly picked 500,000 instances from the rest as the validation set.
\end{enumerate}

As mentioned in \Sec \ref{subsec:experiments1}, for each of the datasets above, we trained a 
CatBoost, a LightGBM and an XGBoost model. These models are available on the GitHub repository along with  the notebook 
\texttt{r2\_score.ipynb}
which replicates the metrics for them as outlined in Table \ref{Tab: metrics}. Moreover, the quantities in Table \ref{Tab: topology}, which are related to the structure of trees in these models, were generated with the code in   
\texttt{Retrieve\_splits.ipynb}.
The notebook contains three functions 
\texttt{retrieve\_catboost}, 
\texttt{retrieve\_lgbm}
and
\texttt{retrieve\_xgb}
which respectively parse a trained CatBoost, LightGBM or XGBoost decision tree 
to extract information such as depth, number of leaves, distinct features on which the tree splits, thresholds for splits, and the rectangular regions corresponding to leaves along with their associated probabilities and values. 
Notice that such information is a prerequisite for applying Algorithm \ref{algorithm} to CatBoost models. 
We have furthermore included a script \texttt{EnsembleParser.py} on the GitHub repository which can be used to obtain such statistics for an arbitrary trained CatBoost, LightGBM or XGBoost model.

For \Sec \ref{subsec:experiments2}, a proprietary, optimized implementation of Algorithm \ref{algorithm} was applied to the four CatBoost models from \Sec \ref{subsec:experiments1} (trained on public datasets) on an AWS machine with 16 cores and 128GB of memory. 
The code for the precomputation step uses the package Numba for just-in-time (JIT) compilation, and can also leverage multithreading. 
The execution times for the precomputation step of Algorithm \ref{algorithm} are recorded in Table \ref{Tab: time}, both with and without multithreading (the JIT compilation time, which was less than seven seconds, is not included). 
The table also includes the time it took to compute Shapley values for a random subsample of size 100 in the second step of the algorithm  via utilizing the look-up tables generated in the first step. The total size of these tables on disk can also be found there. Since our internal library is not available on GitHub, we provide a sanity check to confirm that the look-up tables indeed contain Shapley values. This is done through verifying the efficiency property of Shapley values of a machine learning model (see Example \ref{efficiency-ML}) in \texttt{explanations.ipynb}.   
The code in this notebook randomly picks a tree from the CatBoost ensemble and a leaf of that tree, and then verifies that, at that leaf, the sum of Shapley values of features appearing in the tree  agrees with the leaf's output minus the average of all leaf values (i.e. the average of raw predictions of that decision tree over the training set). 
For the same four CatBoost models, we also present execution times for computing the Owen values in Table \ref{Tab: time Owen} of \Sec \ref{subsec:experiments2}. For each of the four public datasets, the partition of variables which underlies the Owen value generation was obtained through hierarchical clustering of variables based on \textit{maximal information coefficient} \cite{Reshef2015MeasuringDP}. The notebook \texttt{grouping.ipynb} details this process. This is also available on our GitHub repository along with the partitions used for Table \ref{Tab: time Owen} as \texttt{json} files. The look-up tables constructed for Owen values are available as well, and again the efficiency property may be verified for them as a sanity check.\footnote{The second part of Proposition \ref{simplification coalitional} along with the efficiency property of the Shapley value immediately imply that the identity appearing in Example \ref{efficiency-ML} carries over to the marginal or conditional Owen values of a model.}

The experiments in \Sec \ref{subsec:experiments3} were designed to benchmark the complexity of our internal explanation library with other explanation methods.
The comparison was done for  CatBoost models of varying depths trained on a synthetic dataset generated by \eqref{synth_data_model}. 
The code was run on an AWS machine with 32 cores and 256GB of memory.
The reader can find the execution times of our proprietary code on the GitHub repository along with scripts \texttt{plotter.py} and \texttt{plotter\_owen.py} for recreating figures of \Sec \ref{subsec:experiments3}.

\section{Background from game theory}\label{appendix:game}
\subsection{Properties of game values}\label{subappendix:properties}
The goal of the current section is to briefly review basic concepts from cooperative game theory and various properties of game values.\\
\indent 
Let  $N$ be a  finite non-empty subset of $\Bbb{N}$ and denote its cardinality by $n$. 
An $n$-person game with $N$ as its set of players is a set function $v:2^N\rightarrow\Bbb{R}$; 
$v$ is called a \textit{cooperative} game if $v(\varnothing)=0$. A game value $h$ is a function that to any cooperative game $(N,v)$ assigns a vector $\left(h_i[N,v]\right)_{i\in N}$. 
We first review some important properties of game values. 
\begin{definition}\label{terminology}
Let $h$ be a game value; that is, an assignment $(N,v)\mapsto\left(h_i[N,v]\right)_{i\in N}$ 
where $(N,v)$ varies among cooperative games. 
\begin{itemize}
\item The game value $h$ is \textit{linear} if 
$h[N,v_1+r\cdot v_2]=h[N,v_1]+r\cdot h[N,v_2]$ 
for any scalar $r\in\Bbb{R}$ and any two cooperative games $v_1$ and $v_2$ with the same set of players $N$.
\item An element $i\in N$ is called a null (dummy) player for $(N,v)$ if $v(S\cup\{i\})=v(S)$ for any $S\subseteq N\setminus\{i\}$. We say $h$ satisfies the \textit{null-player property} if $h_i[N,v]=0$ whenever $i$ is a null player.  
\item A \textit{carrier} for a game $(N,v)$ is a subset $U\subseteq N$ with the property that $v(S)=v(S\cap U)$ for any $S\subseteq N$; namely, a subset whose complement consists of null players. The game value $h$ satisfies the  \textit{carrier-dependence property} if for any carrier $U$ of a game $(N,v)$ one has $h_i[N,v]=h_i[U,v]$ for all $i\in U$. (On the right-hand side, $h$ is applied to the restrictions of the original game $v:2^N\rightarrow\Bbb{R}$ to $2^U$.)   
\item The game value $h$ satisfies the \textit{efficiency property} if 
\begin{equation}\label{efficiency}
\sum_{i\in N}h_i[N,v]=v(N)    
\end{equation}
for any cooperative game $(N,v)$.  
\item The game value $h$ satisfies the \textit{symmetry property} if for any game $(N,v)$ and any permutation $\sigma$ of $N$ one has $h_{\sigma(i)}[N,\sigma^*v]=h_i[N,v]$ for all $i\in N$ where  $\sigma^*v$ is defined as 
$\sigma^*v(S):=v\left(\sigma^{-1}(S)\right)$. Equivalently, $h$ is symmetric if and only if for any game $(N,v)$ and any $i,j\in N$:
\begin{equation}\label{symmetry}
v(S\cup\{i\})=v(S\cup\{j\})\, \forall S\subseteq N\setminus\{i,j\} 
\Rightarrow h_i[N,v]=h_j[N,v].
\end{equation}
\item The game value $h$ satisfies the \textit{strong monotonicity property} if for any two games $(N,v_1), (N,v_2)$ and any $i\in N$:
\begin{equation}\label{monotonicity}
v_1(S\cup\{i\})-v_1(S)\geq v_2(S\cup\{i\})-v_2(S)\, \forall S\subseteq N\setminus\{i\}
\Rightarrow h_i[N,v_1]\geq h_i[N,v_2].
\end{equation}
\end{itemize}
\end{definition}

As mentioned before, this article mostly deals with games $(N,v)$ where $N=\{1,\dots,n\}$ and with game values defined explicitly by a formula of the form \eqref{linear general} (where $h_i[N,v]$ was replaced with $h_i[v]$ to simplify the notation). The next lemma justifies our convention 
(also see Remark \ref{arbitrary N}).
\begin{lemma}\label{properties}
Linear game values with the null-player property are precisely those that can be written as  
\begin{equation}\label{linear general variant}
h_i[N,v]=\sum_{S\subseteq N\setminus \{i\}}w(S;N,i)\left(v(S\cup\{i\})-v(S)\right)\quad (i\in N)    
\end{equation}
where $(N,v)$ is a cooperative game and $w(\cdot;N,i)$ is a real-valued function defined on the set of subsets of $N\setminus\{i\}$.
Furthermore, for a game value $h$ of this form, properties from Definition \ref{terminology} can be rephrased  as:
\begin{enumerate}
\item carrier dependence  $\Leftrightarrow$ 
$\sum_{S\subseteq N\setminus\{i\},S\cap U=S'}w(S;N,i)=w(S';U,i)$ for all $i\in U\subsetneq N$ and $S'\subseteq U\setminus\{i\}$;
\item efficiency $\Leftrightarrow$ $\sum_{S=N\setminus\{i\}}w(S;N,i)=1$ \&  
$\sum_{i\in S}w(S\setminus\{i\};N,i)=\sum_{i\in N\setminus S}w(S;N,i)$ $\forall\varnothing\neq S\subsetneq N$;
\item symmetry $\Leftrightarrow$ there exists a function $\alpha(\cdot;N)$ 
with $w(S;N,i)=\alpha(|S|;N)$ for all $i\in N$ and $S\subseteq N\setminus\{i\}$;
\item strong monotonicity $\Leftrightarrow w(S;N,i)\geq 0$ $\forall S\subseteq N\setminus\{i\}$.
\end{enumerate}
\end{lemma}

\begin{proof}
If $h$ is linear, then, given any finite subset $N\subset\Bbb{N}$, there exist constants 
$\left\{\gamma(S;N,i)\right\}_{i\in N,S\subseteq N\setminus\{i\}}$   
in terms of which the outputs of $h$ for any cooperative game 
$v:2^N\rightarrow\Bbb{R}$ may be written as
$$
h_i[N,v]=\sum_{S\subseteq N}\gamma(S;N,i)\,v(S).
$$
The equation above may be rewritten as 
\begin{equation}\label{auxiliary7}
h_i[N,v]=
\sum_{S\subseteq N\setminus\{i\}}\gamma(S\cup\{i\};N,i)\left(v(S\cup\{i\})-v(S)\right)
+\sum_{\varnothing\neq S\subseteq N\setminus\{i\}}\left(\gamma(S\cup\{i\};N,i)+\gamma(S;N,i)\right)v(S).
\end{equation}
(Keep in mind that $v(\varnothing)=0$.) 
If $i\in N$ is a null player for $v$, then the first summation vanishes. Hence $h$ satisfies the null-player property if and only if 
$\sum_{\varnothing\neq S\subseteq N\setminus\{i\}}\left(\gamma(S\cup\{i\};N,i)+\gamma(S;N,i)\right)v(S)=0$
for any cooperative game $v:2^N\rightarrow\Bbb{R}$ for which $i$ is a null player. 
But this happens exactly when all coefficients $\gamma(S\cup\{i\};N,i)+\gamma(S;N,i)$ are zero since, in the former sum, each term $v(S)$ can be any arbitrary real number. 
Substituting $\gamma(S\cup\{i\};N,i)+\gamma(S;N,i)=0$ in \eqref{auxiliary7} and setting 
$w(S;N,i):=\gamma(S\cup\{i\};N,i)$, we arrive at the formula 
$$
h_i[N,v]=\sum_{S\subseteq N\setminus \{i\}}w(S;N,i)\left(v(S\cup\{i\})-v(S)\right)\quad (i\in N) 
$$
for linear game values $h$ that satisfy the null-player property. 
Conversely, any game value of the form above clearly has the null-player property. 
This establishes the first claim. 
So we assume that $h$ is of the form above for the rest of the proof. 
Now let $U\subseteq N$ be a carrier for $(N,v)$ and $i\in U$.
In the formula above, $v(S)$ and $v(S\cup\{i\})$ can be replaced with $v(S\cap U)$ and 
$v\big((S\cap U)\cup \{i\}\big)$ respectively. Thus
\small
$$
h_i[N,v]=\sum_{S\subseteq N\setminus \{i\}}w(S;N,i)\left(v\big((S\cap U)\cup \{i\}\big)-v(S\cap U)\right) 
=\sum_{S'\subseteq U\setminus\{i\}}\Big(\sum_{S\subseteq N\setminus\{i\},S\cap U=S'}w(S;N,i)\Big)
(v(S'\cup\{i\})-v(S')).
$$
\normalsize
For $h$ to satisfy the carrier dependence, the above expression must always coincide with 
$$h_i[U,v]=\sum_{S'\subseteq U\setminus \{i\}}w(S';U,i)\left(v(S'\cup\{i\})-v(S')\right),$$
thus the equality of corresponding coefficients: 
$\sum_{S\subseteq N\setminus\{i\},S\cap U=S'}w(S;N,i)=w(S';U,i)$ (which always holds when $U=N$).\\
\indent
Next, we turn into part (2). For a game value $h$ of the form \eqref{linear general variant}, 
the efficiency property \eqref{efficiency} turns into
$$
\sum_{i\in N}\sum_{S\subseteq N\setminus \{i\}}w(S;N,i)\left(v(S\cup\{i\})-v(S)\right)
=v(N)
$$
where $(N,v)$ is a cooperative game, i.e. $v(\varnothing)=0$. 
The left-hand side may be rewritten as 
$$
\Big(\sum_{S=N\setminus\{i\}}w(S;N,i)\Big)v(N)
+\sum_{\varnothing\neq S\subsetneq N}
\Big(\sum_{i\in S}w(S\setminus\{i\};N,i)-\sum_{i\in N\setminus S}w(S;N,i)\Big)v(S).
$$
This must agree with $v(N)$ for any cooperative game $(N,v)$. 
But, as $v$ varies among cooperative games $2^N\rightarrow\Bbb{R}$,
$\left(v(S)\right)_{\varnothing\neq S\subseteq N}$ can be any element of $\Bbb{R}^{2^n-1}$. Therefore, the efficiency holds precisely when in the last expression the coefficient of $v(N)$ is $1$ and the rest of the coefficients are $0$:
$$
\sum_{S=N\setminus\{i\}}w(S;N,i)=1, \quad \sum_{i\in S}w(S\setminus\{i\};N,i)=\sum_{i\in N\setminus S}w(S;N,i)
\quad(S\subsetneq N, S\neq\varnothing).
$$
As for the symmetry property, for any permutation $\sigma$ of the set of players $N$, one has
\begin{equation*}
\begin{split}
h_{\sigma(i)}[N,\sigma^*v]
&=\sum_{S\subseteq N\setminus \{\sigma(i)\}}w(S;N,\sigma(i))
\big(\sigma^*v(S\cup\{\sigma(i)\})-\sigma^*v(S)\big)\\    
&=\sum_{S\subseteq N\setminus \{\sigma(i)\}}w(S;N,\sigma(i))
\left(v\left(\sigma^{-1}(S)\cup\{i\}\right)-v\left(\sigma^{-1}(S)\right)\right)\\
&=\sum_{S\subseteq N\setminus \{i\}}w(\sigma(S);N,\sigma(i))\left(v(S\cup\{i\})-v(S)\right).
\end{split}
\end{equation*}
Comparing with the formula \eqref{linear general variant} for $h_i[N,v]$, the symmetry is equivalent to the equality of corresponding coefficients, i.e. 
$w(S;N,i)=w(\sigma(S);N,\sigma(i))$ for any permutation $\sigma:N\rightarrow N$. This happens exactly when 
$w(S;N,i)$ depends only on the cardinality of the proper subset $S$ of $N$ 
because for any two pairs 
$$
(S_1,i_1) \,(S_1\subseteq N\setminus\{i_1\}) \quad \& \quad (S_2,i_2) \,(S_2\subseteq N\setminus\{i_2\})
$$
with $|S_1|=|S_2|$,
there exists a permutation $\sigma$ of $N$ with $S_2=\sigma(S_1)$ and $i_2=\sigma(i_1)$.\\
\indent
At last, we characterize the strong monotonicity property for game values of form \eqref{linear general variant}.
Due to linearity, \eqref{monotonicity} amounts to  
$$
h_i[N,v]=\sum_{S\subseteq N\setminus \{i\}}w(S;N,i)\left(v(S\cup\{i\})-v(S)\right)\geq 0
$$
whenever all differences $v(S\cup\{i\})-v(S)$ are non-negative. This amounts to $\sum_{S\subseteq N\setminus \{i\}}w(S;N,i)\cdot r_{S}\geq 0$
for any $(r_S)_{S\subseteq N\setminus \{i\}}\in[0,\infty)^{2^{n-1}}$ since any such vector can be realized as 
$\left(v(S\cup\{i\})-v(S)\right)_{S\subseteq N\setminus \{i\}}$
for a suitable cooperative game $v:2^N\rightarrow\Bbb{R}$. Clearly, the former inequality holds for all 
vectors with non-negative entries
if and only if all the coefficients $w(S;N,i)$ are non-negative. 
\end{proof}

\begin{remark}\label{non-cooperative}
There is a systematic way of extending linear game values so that they can be applied to non-cooperative games as well \cite[\Sec 3.5]{2021arXiv210210878M}. Indeed, conditional and marginal games associated with $(\mathbf{X},f)$ as in \eqref{games} are not cooperative since they assign $\Bbb{E}[f(\mathbf{X})]$ to $\varnothing$. But, a benefit of working with a formula such as \eqref{linear general variant} is that it remains unchanged after replacing a game 
$v:S\mapsto v(S)$ with the cooperative one  $S\mapsto v(S)-v(\varnothing)$. 
In other words, \eqref{linear general variant} automatically extends\footnote{Extensions of the form $v\mapsto h[N,S\mapsto v(S)-v(\varnothing)]$ to non-cooperative games are called \textit{centered}.} to non-cooperative games; and we shall freely apply such game values to marginal and conditional games.
All properties appeared in Definition \ref{properties} generalize to the case of non-cooperative games in an obvious way except the efficiency property 
which should be replaced with
\begin{equation}\label{efficiency-variant}
\sum_{i\in N}h_i[N,v]=v(N)-v(\varnothing).    
\end{equation}
\end{remark}

\begin{remark}\label{arbitrary N}
The carrier-dependence property from Definition \ref{properties} naturally brings up games whose sets of players are not necessarily segments of $\Bbb{N}$. On the other hand, it is tempting to define game values only for games $(N,v)$ with $N=\{1,\dots,n\}$ as in \eqref{linear general} rather than working with the more precise form \eqref{linear general variant}. To alleviate this problem, notice that if a game value of form \eqref{linear general}---which is defined only in the case of $N=\{1,\dots,n\}$---is symmetric, then by Lemma \ref{properties} its weights $w(S;n,i)$ depend only on the cardinality of $S$. It can thus be unambiguously applied to any game $(N,v)$ by setting   $w(S;N,i):=w(\sigma^{-1}(S);n,\sigma^{-1}(i))$ where $\sigma:\{1,\dots,n\}\rightarrow N$ is an arbitrary bijection.\footnote{This extension to games $(N,v)$ with arbitrary $N$ is unique if one imposes an \textit{isomorphism invariance} axiom which is stronger than the symmetry axiom and  requires $h_{\sigma(i)}[N_1,\sigma^*v]=h_i[N_2,v]$ for any game $(N_2,v)$ and any bijection 
$\sigma:N_1\rightarrow N_2$.}
\end{remark}

\begin{remark}
The famous characterization of the Shapley value in 
\cite{shapley1953value}
as the unique linear game value satisfying null-player, symmetry and efficiency properties 
can also be recovered from Lemma \ref{properties}: substituting $w(S;N,i)=\alpha(|S|;N)$ from part (3) in the equations from part (2) implies that
$$
|N|\cdot\alpha(|N|-1;N)=1 \quad\&\quad 
|S|\cdot\alpha(|S|-1;N)=(|N|-|S|)\cdot\alpha(|S|;N).
$$
From these, one can inductively show that $\alpha(|S|;N)=w(S;N,i)$ is equal to 
$\frac{|S|!\,(|N|-|S|-1)!}{|N|!}$; consequently, $h_i[N,v]$ coincides with the right-hand side of \eqref{ShapleyFormula}.
\end{remark}

\subsection{Marginal and conditional games}\label{subappendix:games}
The marginal and conditional games associated with predictors $\mathbf{X}=(X_1,\dots,X_n)$ 
and a Borel-measurable $f:\Bbb{R}^n\rightarrow\Bbb{R}$ are defined as pointwise games in \eqref{games}. 
Of course, they (along with feature attributions arising from them) can also be treated as random variables
which, by abuse of notation, we denote as
\begin{equation}\label{games-variant}
v^\ME(S;\mathbf{X},f)(\omega):=
\Bbb{E}_{\mathbf{X}_{-S}}[f(\mathbf{X}_S(\omega),\mathbf{X}_{-S})], \quad 
v^\CE(S;\mathbf{X},f)(\omega):=
\Bbb{E}[f(\mathbf{X})\mid \mathbf{X}_S](\omega) 
\quad\quad (S\subseteq N);
\end{equation}  
where $\omega$ belongs to the sample space $\Omega$.\footnote{Any slice of a Borel-measurable function is Borel measurable too. Thus, fixing $\omega\in\Omega$,  $f(\mathbf{X}_S(\omega),\mathbf{X}_{-S})$ is a random variable on $(\Omega,\mathcal{F},\Bbb{P})$ and it thus makes sense to speak of its expectation in \eqref{games-variant}.}  
In order for the expectations  in \eqref{games-variant} (conditional or unconditional) to exist, we always assume that the predictors belong to $L^2(\Omega,\mathcal{F},\Bbb{P})$ and 
$f$ is uniformly bounded on $\Bbb{R}^n$.\footnote{Lipschitz-continuous functions work too although $f$ is not continuous in our context.}

\begin{example}\label{efficiency-ML}
For these games, the efficiency property \eqref{efficiency-variant} turns into
\begin{equation}
\sum_{i\in N}\varphi_i\big[v^\ME\big]=\sum_{i\in N}\varphi_i\big[v^\CE\big]
=f-\Bbb{E}\left[f(\mathbf{X})\right]   
\end{equation}
for the marginal and conditional Shapley values. 
\end{example}

\begin{example}\label{regression}
For a linear model 
$f(\mathbf{X})=c_1X_1+\dots+c_nX_n$,
marginal and conditional Shapley values (treated as random variables here) are given by  
\begin{equation*}
\varphi_i\big[v^\ME\big]=c_i\left(X_i-\Bbb{E}[X_i]\right)\quad (i\in N)
\end{equation*}
and
\begin{equation*}
\begin{split}
\varphi_i\big[v^\CE\big]=
&c_i\Big(X_i-\sum_{S\subseteq N\setminus\{i\}}\frac{|S|!\,(|N|-|S|-1)!}{|N|!}
\,\Bbb{E}[X_i\mid \mathbf{X}_S]\Big)\\
&+\sum_{j\in N\setminus\{i\}}
c_j\Big(\sum_{S\subseteq N\setminus\{i,j\}}\frac{|S|!\,(|N|-|S|-1)!}{|N|!}
\left(\Bbb{E}[X_j\mid\mathbf{X}_{S\cup\{i\}}]-\Bbb{E}[X_j\mid \mathbf{X}_S]\right)\Big) \quad (i\in N).
\end{split}    
\end{equation*}
Notice that the latter contains conditional expectation terms while the former is purely in terms of coefficients of the model $f$. Moreover, if the predictors are independent, the two kinds of explanations coincide since all terms $\Bbb{E}[X_i\mid \mathbf{X}_S]$, 
$\Bbb{E}[X_j\mid \mathbf{X}_{S\cup\{i\}}]$ and
$\Bbb{E}[X_j\mid \mathbf{X}_S]$ become $\Bbb{E}[X_i]$ or $\Bbb{E}[X_j]$. Finally, 
the identities above showcase why in Lemma \ref{dummy players}
certain assumptions on predictors are required in the case of the conditional game.
When $c_i=0$, the variable $x_i$ is absent from $f(\mathbf{x})=c_1x_1+\dots+c_nx_n$, and we see that the $i^{\rm{th}}$ marginal Shapley value vanishes. This does not hold for its conditional counterpart unless $X_i$ is independent of the rest of variables. On the other extreme, when $c_i$ is the only non-zero coefficient, 
$\{i\}$ does not become a carrier for $v^\CE$, i.e. $\varphi_i\big[v^\CE\big]\neq c_i\left(X_i-\Bbb{E}[X_i]\right)$ in general. This is because the conditional expectations $\Bbb{E}[X_i\mid \mathbf{X}_S]$ are not necessarily the same as $\Bbb{E}[X_i]$ unless a further assumption, 
e.g. the independence of $X_i$ from $\mathbf{X}_{-i}$, is imposed. 
\end{example}

\begin{proof}[Proof of Lemma \ref{dummy players}]
If $f$ is independent of variables indexed by elements outside of $U\subset N$, then
\begin{equation*}
\begin{split}
&v^\ME(S;\mathbf{X},f)(\mathbf{x})=\Bbb{E}[f(\mathbf{x}_S,\mathbf{X}_{-S})]
=\Bbb{E}[f(\mathbf{x}_{S\cap U},\mathbf{x}_{S\setminus U},\mathbf{X}_{-S})]
=\Bbb{E}[f(\mathbf{x}_{S\cap U},\mathbf{X}_{S\setminus U},\mathbf{X}_{-S})]\\
&=\Bbb{E}[f(\mathbf{x}_{S\cap U},\mathbf{X}_{-(S\cap U)})]=v^\ME(S\cap U;\mathbf{X},f)(\mathbf{x})
\end{split}    
\end{equation*}
which implies that $U$ is a carrier. In particular, elements of $N\setminus U$ are null players. 
As for the other claim of part (1), notice that the term $\Bbb{E}[f(\mathbf{x}_S,\mathbf{X}_{-S})]$
above may be written as 
$$\Bbb{E}\big[\tilde{f}(\mathbf{x}_{S\cap U},\mathbf{X}_{U\setminus S})\big]
=v^\ME(S\cap U;\mathbf{X}_U,f)(\mathbf{x}_U).$$
\\
\indent
For the second part, we shall need the following fact: Given random vectors $\mathbf{W},\mathbf{Y}$ and $\mathbf{Z}$, one has 
$\Bbb{E}[\mathbf{W}\mid \mathbf{Y}=\mathbf{y},\mathbf{Z}=\mathbf{z}]
=\Bbb{E}[\mathbf{W}\mid\mathbf{Y}=\mathbf{y}]$ 
provided that $(\mathbf{W},\mathbf{Y})$ is independent from $\mathbf{Z}$.
Writing $f$ as $f=\tilde{f}\circ\pi_U$ again, if $\mathbf{X}_U$ is independent of $\mathbf{X}_{-U}$, then one has  
\begin{equation*}
\begin{split}
&v^\CE(S;\mathbf{X},f)(\mathbf{x})
=\Bbb{E}[f(\mathbf{X})\mid \mathbf{X}_S=\mathbf{x}_S]  
=\Bbb{E}\big[\tilde{f}(\mathbf{X}_U)\mid \mathbf{X}_{S\cap U}=\mathbf{x}_{S\cap U},
\mathbf{X}_{S\setminus U}=\mathbf{x}_{S\setminus U}\big]\\  
&=\Bbb{E}\big[\tilde{f}(\mathbf{X}_U)\mid \mathbf{X}_{S\cap U}=\mathbf{x}_{S\cap U}\big]
=\Bbb{E}[f(\mathbf{X})\mid \mathbf{X}_{S\cap U}=\mathbf{x}_{S\cap U}]
=v^\CE(S\cap U;\mathbf{X},f)(\mathbf{x}).
\end{split} 
\end{equation*}
The remaining claim follows immediately by setting $U=N\setminus\{i\}$: If $f$ does not depend on $x_i$ and $X_i$ is independent of 
$\mathbf{X}_{-i}=\mathbf{X}_{N\setminus \{i\}}$, then $U$ is a carrier for the conditional game, and hence $i\in N\setminus U$ should be a null player.
\end{proof}

\begin{remark}
The marginal framework for feature attribution is sometimes called interventional because it can be thought of as enforcing independence between two complementary subsets of features $\mathbf{X}_S$ and $\mathbf{X}_{-S}$ through causal intervention \cite{janzing2020feature}. This amounts to a flat causal graph, not a general 
DAG (Directed Acyclic Graph), so the name can be misleading 
\cite[p. 7]{2022arXiv220707605C}. 
\end{remark}

\subsection{Coalitional game values}\label{subappendix:coalitional game values}
A coalitional game value $\mathfrak{h}$ takes a cooperative game $(N,v)$ along with a partition $\mathfrak{P}$ of $N$ as inputs and returns a vector 
$\left(\mathfrak{h}_i[N,v,\mathfrak{P}]\right)_{i\in N}$.
Just like \eqref{linear general variant}, it is convenient to assume that  $\mathfrak{h}$ is in the form 
\begin{equation}\label{linear general variant coalitional}
\mathfrak{h}_i[N,v,\mathfrak{P}]=\sum_{S\subseteq N\setminus \{i\}}w(S;N,i,\mathfrak{P})\left(v(S\cup\{i\})-v(S)\right)\quad (i\in N)    
\end{equation}
which is linear and satisfies the null-player property in a clear sense.
The most important example for us is the Owen value \eqref{OwenFormula} which is of this form. \\
\indent 
The efficiency property from Definition \ref{properties} immediately generalizes to coalitional game values, and it is not hard to see that it holds for the Owen value. 
On the other hand, formulating symmetry or carrier-dependence properties for coalitional game values is slightly more subtle.
\begin{definition}\label{coalitional definition}
Let $\mathfrak{h}$ be a coalitional game value. 
\begin{itemize}
\item We say that $\mathfrak{h}$ has the \textit{coalitional symmetry property} if 
given a game $(N,v)$ and a partition $\mathfrak{P}:=\{S_1,\dots,S_m\}$ of $N$ one has the followings:\footnote{These properties also come up in a characterization of the Owen value in \cite{owen1977values}.}
\begin{enumerate}
\item $\mathfrak{h}\big[N,v,\{S_j\}_{j=1}^m\big]=\mathfrak{h}\big[N,v,\{S_{\tau(j)}\}_{j=1}^m\big]$
for any permutation $\tau$ of $M=\{1,\dots,m\}$ (i.e. the order of members of $\mathfrak{P}$ does not matter);
\item 
$\mathfrak{h}_{\sigma(i)}\big[N,\sigma^*v,\{\sigma(S_j)\}_{j=1}^m\big]
=\mathfrak{h}_i\big[N,v,\{S_j\}_{j=1}^m\big]$ for any permutation $\sigma$ of $N=\{1,\dots,n\}$.
\end{enumerate}
\item We say that $\mathfrak{h}$ has the \textit{coalitional carrier-dependence property} if given a game $(N,v)$, 
a carrier $U\subseteq N$, and a partition $\mathfrak{P}$ of $N$, one has 
$\mathfrak{h}_i[N,v,\mathfrak{P}]=\mathfrak{h}_i[U,v,\mathfrak{P}']$ for all $i\in U$ 
where 
$\mathfrak{P}':=\{S\cap U\mid S\in\mathfrak{P}, S\cap U\neq\varnothing\}$
is a partition of $U$.
\end{itemize}
\end{definition}

\begin{lemma}\label{carrier}
The Shapley value and the Banzhaf value satisfy the carrier dependence property; and the Owen value satisfies the coalitional carrier dependence property. 
\end{lemma}
\begin{proof}
According to part (1) of Lemma \ref{properties}, to establish the carrier dependence for the Shapley value \eqref{ShapleyFormula}, it suffices to show that 
$$
\sum_{S\subseteq N\setminus\{i\},S\cap U=S'}\frac{|S|!\,(|N|-|S|-1)!}{|N|!}
=\frac{|S'|!\,(|U|-|S'|-1)!}{|U|!}
$$
for any finite subset $U$ of $N$ containing $i\in N$ and any $S'\subseteq U\setminus\{i\}$. 
The left-hand side may be rewritten as 
$\sum_{S'\subseteq S\subseteq S'\sqcup (N-U)}\frac{|S|!\,(|N|-|S|-1)!}{|N|!}$; 
the result now follows from Lemma \ref{hockey-stick}.\\
\indent
Similarly, for the Banzhaf value, the identity from part (1) of Lemma  \ref{properties} holds:
$$
\sum_{S\subseteq N\setminus\{i\},S\cap U=S'}\frac{1}{2^{n-1}}=2^{|N|-|U|}\cdot\frac{1}{2^{n-1}}=\frac{1}{2^{|U|-1}}.
$$
For the carrier-dependence property, we refer the reader to Proposition \ref{simplification coalitional} where the coalitional carrier-dependence is established for a family of coalitional values generalizing the Owen value. 
\end{proof}

We finish this section by discussing the quotient game and the \textit{quotient game property}.
In the context of Explainable AI, these game-theoretic concepts are used in \cite{2021arXiv210210878M} to unify marginal and conditional feature attributions of a model through  appropriate groupings of predictors, and also to address the instability of the former type of attributions.  
For a game $(N,v)$ and a partition $\mathfrak{P}:=\{S_1,\dots,S_m\}$ of its set of players $N$, the quotient game $v^{\mathfrak{P}}$ has $M=\{1,\dots,m\}$ as its set of players and is defined as 
\begin{equation}\label{quotinet game}
v^{\mathfrak{P}}(R):=v\left(\cup_{r\in R}S_r\right) \quad (R\subseteq M).
\end{equation}
A coalitional game value $\mathfrak{h}$ is said to admit the quotient game property if one has 
\begin{equation}\label{quotinet game property}
\sum_{i\in S_j}\mathfrak{h}_i[N,v,\mathfrak{P}]=
\mathfrak{h}_j\big[M,v^{\mathfrak{P}},\bar{M}\big]
\quad (j\in M \text{ and }\bar{M}:=\{\{i\}\mid i\in M\})
\end{equation}
for all $(N,v)$ and $\mathfrak{P}$. 

\begin{example}\label{two-step}
The two-step Shapley value defined in \cite{kamijo2009two} 
\begin{equation}\label{two-step Shapley}
TSh_i[N,v,\mathfrak{P}]:=
\varphi_i[S_j,v]+\frac{1}{|S_j|}\big(\varphi_j\big[M,v^{\mathfrak{P}}\big]-v(S_j)\big)    
\quad (j\in M, i\in S_j)
\end{equation}
has  the coalitional symmetry and quotient game properties, but it does not satisfy the coalitional carrier-dependence or the null-player properties.  
\end{example}

\subsection{An application to categorical variables}\label{subappendix:OneHot}
In this short section, we argue that concepts such as a coalitional game value or quotient game 
naturally emerge in the context of machine learning explanations when   
it comes to one-hot encoding.\\
\indent 
Let us first formalize the setting. Write the predictors as 
\begin{equation}\label{categorical predictors}
\mathbf{X}=(X_1,\dots,X_n)=(\mathbf{X}_C,\mathbf{X}_{-C})     
\end{equation}
where $C\subseteq N=\{1,\dots,n\}$, a subset of size $c\leq n$, captures the indices of categorical feature. Denote the range of the categorical feature $X_i$, $i\in C$, by the finite set 
$\{u_{i,1},\dots,u_{i,\kappa_i}\}$ 
which can be indexed by $\mathcal{K}_i:=\{1,\dots,\kappa_i\}$.
The one-hot encoding of $X_i$ yields a vector of random variables taking values in $\{0,1\}$:
\begin{equation}\label{encoded}
\tilde{\mathbf{X}}^{(i)}:=(\tilde{X}^{(i)}_{1},\dots,\tilde{X}^{(i)}_{\kappa_i})  
\quad (i\in C)\quad 
\text{ where }  \tilde{X}^{(i)}_{j}:=\mathbbm{1}_{X_i=u_{i,j}} \quad (1\leq j\leq\kappa_i).
\end{equation}
Therefore, the set of predictors after the encoding is 
\begin{equation}\label{encoded all}
\tilde{\mathbf{X}}:=\big((\tilde{\mathbf{X}}^{(i)})_{i\in C},\mathbf{X}_{N\setminus C}\big).    
\end{equation}
These are defined on the same probability space $(\Omega,\mathcal{F},\Bbb{P})$, 
and can be indexed with the elements of the disjoint union
\begin{equation}\label{encoded players}
\tilde{N}:=(\sqcup_{i\in C}\mathcal{K}_i)\sqcup (N\setminus C).    
\end{equation}
Given a model $f=f(x_1,\dots.x_n)=f(\mathbf{x}_C,\mathbf{x}_{N\setminus C})$, 
its transformed version $\tilde{f}$ takes points
$\tilde{\mathbf{x}}:=\big((\tilde{\mathbf{x}}^{(i)})_{i\in C},\mathbf{x}_{N\setminus C}\big)$
as inputs where $\tilde{\mathbf{x}}^{(i)}\in\Bbb{R}^{\kappa_i}$ and 
$\mathbf{x}_{N\setminus C}\in\Bbb{R}^{n-c}$.
It should yield the same output as the original model $f$:
\begin{equation}\label{encoded function}
\tilde{f}\big((\mathbf{e}^{(i,j_i)})_{i\in C},\mathbf{x}_{N\setminus C}\big)
=f\big(\mathbf{x}_C=(u_{i,j_i})_{i\in C},\mathbf{x}_{N\setminus C}\big)
\quad (j_i\in\{1,\dots,\kappa_i\}),
\end{equation}
where each $\mathbf{e}^{(i,j_i)}\in\Bbb{R}^{\kappa_i}$ is the standard basis vector whose only 
non-zero entry is $1$, located at the dimension $j_i$.
Notice that, above, the one-hot encoding was done as
\begin{equation}\label{encoding}
\mathbf{x}=\big(\mathbf{x}_C=(u_{i,j_i})_{i\in C},\mathbf{x}_{N\setminus C}\big)\mapsto
\tilde{\mathbf{x}}=\big((\mathbf{e}^{(i,j_i)})_{i\in C},\mathbf{x}_{N\setminus C}\big).
\end{equation}
We next relate the marginal game $v^\ME(\cdot;\mathbf{X},f)$ defined based on the original predictors and model to the marginal game  $v^\ME(\cdot;\tilde{\mathbf{X}},\tilde{f})$ associated with the encoded predictors and the transformed model. The set of players for the former is $N$ while the set of players for the latter is the set $\tilde{N}$ from \eqref{encoded players}. This set admits a natural partition $\mathfrak{P}:=\{\mathcal{K}_i\mid i\in C\}\cup\{\{i\}\mid i\in N\setminus C\}$. 
Its elements are in a one-to-one correspondence with elements of $N$. Consequently, the 
set of players for both $v^\ME(\cdot;\mathbf{X},f)$ and the quotient game 
$v^{\ME,\mathfrak{P}}(\cdot;\tilde{\mathbf{X}},\tilde{f})$
can be identified with $N$.
\begin{lemma}\label{encoded Lemma}
With the notation and convention as above, for any $S\subseteq N$ we have 
$v^\ME(S;\mathbf{X},f)=v^{\ME,\mathfrak{P}}(S;\tilde{\mathbf{X}},\tilde{f})$
almost surely. 
\end{lemma}
\begin{proof}
Let us treat $v^\ME(S;\mathbf{X},f)$ and $v^{\ME,\mathfrak{P}}(S;\tilde{\mathbf{X}},\tilde{f})$ as random variables on $(\Omega,\mathcal{F},\Bbb{P})$ as in \eqref{games-variant}. 
Thus for any $\omega\in\Omega$ one has 
$v^\ME(S;\mathbf{X},f)(\omega)=
\Bbb{E}_{\mathbf{X}_{-S}}\big[f(\mathbf{X}_S(\omega),\mathbf{X}_{-S})\big]
=\Bbb{E}_{\mathbf{X}_{N\setminus S}}\big[f(\mathbf{X}_S(\omega),\mathbf{X}_{N\setminus S})\big]$, and 
\begin{equation}\label{auxiliary23}
v^{\ME,\mathfrak{P}}(S;\tilde{\mathbf{X}},\tilde{f})(\omega)
=\Bbb{E}_{\mathbf{X}_{N\setminus (C\cup S)}}\Bbb{E}_{(\tilde{\mathbf{X}}_i)_{i\in C\setminus S}}
\big[
\tilde{f}\big((\tilde{\mathbf{X}}^{(i)}(\omega))_{i\in C\cap S},
(\tilde{\mathbf{X}}^{(i)})_{i\in C\setminus S},
\mathbf{X}_{S\setminus C}(\omega),
\mathbf{X}_{N\setminus (C\cup S)}\big)
\big], 
\end{equation}
where the definition of the quotient game, \eqref{quotinet game}, was invoked.
Due to the fact that each $\mathbf{X}^{(i)}$ is obtained from one-hot encoding of the categorical feature $X_i$ as in \eqref{encoded},  $\mathbf{X}^{(i)}$, with probability $1$, is equal to a standard basis vector in $\Bbb{R}^{\kappa_i}$. Equation \eqref{encoded function} now implies that the right-hand side of \eqref{auxiliary23}  is the same as 
$\Bbb{E}_{\mathbf{X}_{N\setminus S}}\big[f(\mathbf{X}_S(\omega),\mathbf{X}_{N\setminus S})\big]$
with probability $1$.  
\end{proof}

We are finally in a position to present an application to feature attributions in machine learning. 
Suppose a model is trained using the encoded features $(\tilde{\mathbf{X}}^{(i)})_{i\in C}$, but we want to recover the marginal Shapley values of the original categorical features 
$(X_i)_{i\in C}$. A common misconception is that, to this end, the Shapley values of the encoded features should be added up. The correct answer is to ``enrich'' the Shapley value to its coalitional counterpart, the Owen value. This can be generalized to the game values $h^{\mathcal{A}}$ from 
\eqref{good game value} which mimic the Shapley value: To obtain marginal attributions based on $h^{\mathcal{A}}$ for a categorical feature from the attributions of its encodings, one should employ coalitional analogs $\mathfrak{h}^{\mathcal{A}_1,\mathcal{A}_2}$ of $h^{\mathcal{A}}$---see Appendix \ref{subappendix:classification coalitional}---that imitate the definition of the Owen value.  

\begin{proposition}\label{encoded Proposition}
With the notation and convention as above, we have
$$
\varphi_i\big[v^\ME(\cdot;\mathbf{X},f)\big](\mathbf{x})
=\sum_{j\in\mathcal{K}_i}Ow_j\big[v^\ME(\cdot;\tilde{\mathbf{X}},\tilde{f}),\mathfrak{P}\big](\tilde{\mathbf{x}}) \quad
\text{for }{\rm{P}}_\mathbf{X}\text{-a.e. }\mathbf{x} \text{ and any }i\in N
$$
where $\tilde{\mathbf{x}}$ is obtained from $\mathbf{x}$ after the one-hot encoding, as in \eqref{encoding}.\\
\indent
More generally, consider a game value $h^{\mathcal{A}}$,  
defined as in Theorem \ref{classification} based on a collection 
$\mathcal{A}=\{\alpha(s,n)\}_{\substack{n\in\Bbb{N}\\ 0\leq s<n}}$
of numbers satisfying \eqref{backward Pascal}. One then has 
$$
h^{\mathcal{A}}_i\big[v^\ME(\cdot;\mathbf{X},f)\big](\mathbf{x})
=\sum_{j\in\mathcal{K}_i}\mathfrak{h}^{\mathcal{A}_1,\mathcal{A}_2}_j
\big[v^\ME(\cdot;\tilde{\mathbf{X}},\tilde{f}),\mathfrak{P}\big](\tilde{\mathbf{x}}) 
\quad
\text{for }{\rm{P}}_\mathbf{X}\text{-a.e. }\mathbf{x} \text{ and any }i\in N
$$
where $\mathfrak{h}^{\mathcal{A}_1,\mathcal{A}_2}$ is the coalitional game value defined as in 
\eqref{coalitional generalization} with $\mathcal{A}_1=\mathcal{A}$ 
and $\mathcal{A}_2:=\left\{\alpha_2(s,n):=\frac{s!(n-s-1)!}{n!}\right\}_{\substack{n\in\Bbb{N}\\ 0\leq s<n}}$.
\end{proposition}
\begin{proof}
Invoking part (3) of Proposition \ref{simplification coalitional}, coalitional game values  $\mathfrak{h}^{\mathcal{A}_1,\mathcal{A}_2}$ with $\mathcal{A}_2$ as above satisfy the quotient game property \eqref{quotinet game property}. Lemma \ref{encoded Lemma} now concludes the proof.
\end{proof}

\section{Computations related to TreeSHAP}
\subsection{Computations for Example \ref{main example}}\label{subappendix:main example's details}

Let us elaborate by first discussing the case of conditional and marginal games which are defined as in \eqref{games}. It is convenient to write the function implemented by $T_1$ and $T_2$ as  
\begin{equation}\label{main example function}
\begin{split}
g(\mathbf{X})&=c_1\cdot\mathbbm{1}_{R_1}(\mathbf{X})+c_2\cdot\mathbbm{1}_{R_2}(\mathbf{X})+c_3\cdot\mathbbm{1}_{R_3}(\mathbf{X})\\
&=c_1\cdot\mathbbm{1}_{[-1,0]}(X_2)+c_2\cdot\mathbbm{1}_{[-1,0]}(X_1)\cdot\mathbbm{1}_{[0,1]}(X_2)
+c_3\cdot\mathbbm{1}_{[0,1]}(X_1)\cdot\mathbbm{1}_{[0,1]}(X_2).
\end{split}
\end{equation}
One then has 
\small
\begin{equation}\label{main example conditional}
\begin{split}
v^\CE(\{1\};\mathbf{X},g)(x_1,x_2)=&c_1\cdot\Bbb{E}[\mathbbm{1}_{[-1,0]}(X_2)\mid X_1=x_1]\\
&+c_2\cdot\mathbbm{1}_{[-1,0]}(x_1)\cdot\Bbb{E}[\mathbbm{1}_{[0,1]}(X_2)\mid X_1=x_1]
+c_3\cdot\mathbbm{1}_{[0,1]}(x_1)\cdot\Bbb{E}[\mathbbm{1}_{[0,1]}(X_2)\mid X_1=x_1],\\
v^\CE(\{2\};\mathbf{X},g)(x_1,x_2)=&c_1\cdot\mathbbm{1}_{[-1,0]}(x_2)\\
&+c_2\cdot\mathbbm{1}_{[0,1]}(x_2)\cdot\Bbb{E}[\mathbbm{1}_{[-1,0]}(X_1)\mid X_2=x_2]
+c_3\cdot\mathbbm{1}_{[0,1]}(x_2)\cdot\Bbb{E}[\mathbbm{1}_{[0,1]}(X_1)\mid X_2=x_2].
\end{split}
\end{equation}
\normalsize
Conditioning on which of $R_1^{\minus}$, $R_1^{\plus}$, $R_2$ or $R_3$ the point $\mathbf{x}$ belongs to, the difference of expressions above can be written in terms of $c_1,c_2,c_3$ and functions from \eqref{auxiliary1}, hence the first row Table \ref{Tab:main example table}.
As for the marginal game, plugging \eqref{main example function} in the definition yields 
\small
\begin{equation}\label{main example marginal}
\begin{split}
&v^\ME(\{1\};\mathbf{X},g)(x_1,x_2)=c_1\cdot\Bbb{P}(X_2\in [-1,0])
+c_2\cdot\mathbbm{1}_{[-1,0]}(x_1)\cdot\Bbb{P}(X_2\in [0,1])
+c_3\cdot\mathbbm{1}_{[0,1]}(x_1)\cdot\Bbb{P}(X_2\in [0,1]),\\
&v^\ME(\{2\};\mathbf{X},g)(x_1,x_2)=c_1\cdot\mathbbm{1}_{[-1,0]}(x_2)
+c_2\cdot \mathbbm{1}_{[0,1]}(x_2)\cdot\Bbb{P}(X_1\in [-1,0])
+c_3\cdot\mathbbm{1}_{[0,1]}(x_2)\cdot\Bbb{P}(X_1\in [0,1]).
\end{split}    
\end{equation}
\normalsize
Probabilities appearing above can be written in terms of probabilities \eqref{auxiliary2} assigned to subrectangles visible in Figure \ref{fig:failure1}.
Simplifying the difference of expressions from \eqref{main example marginal} then yields the second row of 
Table \ref{Tab:main example table}.\\
\indent
Based on Definition \ref{TreeSHAP definition}, the TreeSHAP games associated with $T_1$ and $T_2$ are given by
\begin{equation}\label{main example TreeSHAP}
\begin{split}
&v^\Tree(\{1\};T_1)=\hat{p}_1c_1+(\hat{p}_2+\hat{p}_3)(c_2\cdot\mathbbm{1}_{x_1\leq 0}+c_3\cdot\mathbbm{1}_{x_1>0}),\\
&v^\Tree(\{2\};T_1)=c_1\cdot\mathbbm{1}_{x_2\leq 0}+
\big(\frac{\hat{p}_2}{\hat{p}_2+\hat{p}_3}c_2+\frac{\hat{p}_3}{\hat{p}_2+\hat{p}_3}c_3\big)\cdot
\mathbbm{1}_{x_2>0};\\
&v^\Tree(\{1\};T_2)=
\big(\frac{\hat{p}_1^{\minus}}{\hat{p}_1^{\minus}+\hat{p}_2}c_1+\frac{\hat{p}_2}{\hat{p}_1^{\minus}+\hat{p}_2}c_2\big)
\cdot\mathbbm{1}_{x_1\leq 0}
+\big(\frac{\hat{p}_1^{\plus}}{\hat{p}_1^{\plus}+\hat{p}_3}c_1+\frac{\hat{p}_3}{\hat{p}_1^{\plus}+\hat{p}_3}c_3\big)
\cdot\mathbbm{1}_{x_1>0},\\
&v^\Tree(\{2\};T_2)=
(\hat{p}_1^{\minus}+\hat{p}_2)(c_1\cdot\mathbbm{1}_{x_2\leq 0}+c_2\cdot\mathbbm{1}_{x_2>0})
+(\hat{p}_1^{\plus}+\hat{p}_3)(c_1\cdot\mathbbm{1}_{x_2\leq 0}+c_3\cdot\mathbbm{1}_{x_2>0}).
\end{split}
\end{equation}
\normalsize
Their differences are recorded in the last two rows of Table \ref{Tab:main example table}.
In particular, over $R_2$, where $x_1\leq 0$ and $x_2\geq 0$, the differences become 
$$
v^\Tree(\{1\};T_1)-v^\Tree(\{2\};T_1)=
c_1\hat{p}_1+c_3(\hat{p}_2+\hat{p}_3)-\big(c_2\frac{\hat{p}_2}{\hat{p}_2+\hat{p}_3}+c_3\frac{\hat{p}_3}{\hat{p}_2+\hat{p}_3}\big),
$$
and 
$$
v^\Tree(\{1\};T_2)-v^\Tree(\{2\};T_2)=
\big(c_1\frac{\hat{p}_1^{\minus}}{\hat{p}_1^{\minus}+\hat{p}_2}+c_2\frac{\hat{p}_2}{\hat{p}_1^{\minus}+\hat{p}_2}\big)-\big(c_2(\hat{p}_1^{\minus}+\hat{p}_2)+c_3(\hat{p}_1^{\plus}+\hat{p}_3)\big).
$$
For the parameters in \eqref{parameters}, the former is negative while the latter becomes positive. At the same time, the impurity measures are very close for the two trees; see \eqref{impurity1} and \eqref{impurity2}.

\subsection{Eject TreeSHAP}\label{subappendix:eject}
Here, we discuss the ``eject'' variant of TreeSHAP  \cite[Algorithm 3]{campbell2022exact}, and 
exhibit an example which demonstrates that this method also
suffers from dependence on model make-up.

\begin{definition}\label{Eject definition}
With the notation as in Definition \ref{TreeSHAP definition}, consider a trained ensemble $\mathcal{T}$ of regressor trees.
One can assume that, during the training process, values are assigned to non-terminal nodes of trees from $\mathcal{T}$ the same way values at terminal nodes are determined. The most natural example is to take the value of a node to be the average of response values of the training instances ended up there.
We define the associated eject TreeSHAP game 
$v^\TreeEj(\cdot;\mathcal{T})$ as 
$v^\TreeEj(\cdot;\mathcal{T})
=\sum_{T\in\mathcal{T}}v^\TreeEj(\cdot;T)$
where, for decision trees, games $v^\TreeEj(\cdot;T)$ are defined recursively 
in the following manner. 
In case that $T$ has no splits (so $T$ is a single leaf), $v^\TreeEj(\cdot;T)$ assigns the value at the unique leaf of $T$ to every subset of $N=\{1,\dots,n\}$.   
Next suppose the split at the root of $T$ takes place with respect  to feature  
$X_{i_*}$ and threshold $t_*$. Then, with $T^<$, $T^>$, $D^<$ and $D^>$ as in Definition \ref{TreeSHAP definition}, we set 
\begin{equation}\label{Eject game}
    v^\TreeEj(S;T)(\mathbf{x}):=
    \begin{cases}
    v^\TreeEj(S;T^>)(\mathbf{x}) &\text{if }i_*\in S \text{ and } x_{i_*}>t_*\\
    v^\TreeEj(S;T^<)(\mathbf{x}) &\text{if }i_*\in S \text{ and } x_{i_*}<t_*\\
    \text{the value assigned to the root of }T & \text{if }i_*\notin S
    \end{cases}
    \quad (\mathbf{x}\in\Bbb{R}^n, S\subseteq N).
\end{equation}
\end{definition}
\noindent 
As pointed out in \cite{campbell2022exact}, ``local dummy players'' emerge in this setting: 
$X_i$ may appear in a tree $T\in\mathcal{T}$ but $i\in N$ can be a dummy (null) player of $v^\TreeEj(\cdot;T)(\mathbf{x})$ if $\mathbf{x}$ ends up at a leaf of $T$ whose path to the root does not encounter any split on $X_i$.

\begin{example}\label{Eject example}
Again, we work with a problem with two predictors $(X_1,X_2)$ supported in the rectangle $\mathcal{B}=[-1,1]\times[-1,1]$. 
Consider the partition of $\mathcal{B}$ illustrated in Figure \ref{fig:failure2} into subsquares $[-1,0]\times[-1,0]$, $[-1,0]\times[0,1]$, $[0,1]\times[-1,0]$ and $[0,1]\times[0,1]$; we assume the probability assigned to each of them is $0.25$. The decision trees $T_1$ and $T_2$ appearing in that picture both compute the simple function 
$$g:=c_1\cdot\mathbbm{1}_{[-1,0]\times[-1,0]}
+c_2\cdot\mathbbm{1}_{[-1,0]\times[0,1]}
+c_3\cdot\mathbbm{1}_{[0,1]\times[-1,0]}
+c_1\cdot\mathbbm{1}_{[0,1]\times[0,1]}.$$  
Indeed, these trees differ only in terms of the order of levels. The reader can easily check that they also have identical impurity measures (the same Gini impurity and the same weighted variance after the splits at roots).
As in Definition \ref{Eject game}, one can define games $v^\TreeEj(\cdot;T_1)$ and $v^\TreeEj(\cdot;T_2)$ based on these trees. Here, we assume that the value assigned to a non-terminal node is the average of response variables for training instances ended up there. Assuming that the training set is very large, the value at a node becomes very close to the average of values of its descendant leaves (keep in mind that, here, for each split the two alternatives  are equally likely). Over the top-left subsquare one has:
\begin{equation}\label{eject games}
\begin{split}
&v^\TreeEj(\{1\};T_1)(\mathbf{x})\approx\frac{c_1+c_2}{2}, 
\quad v^\TreeEj(\{2\};T_1)(\mathbf{x})\approx\frac{2c_1+c_2+c_3}{4},\\
&v^\TreeEj(\{1\};T_2)(\mathbf{x})\approx\frac{2c_1+c_2+c_3}{4},
\quad v^\TreeEj(\{2\};T_2)(\mathbf{x})\approx\frac{c_1+c_2}{2},
\end{split}    
\end{equation}
where $\mathbf{x}\in[-1,0]\times[0,1]$. Notice how the values of the eject game at $\{1\}$ and $\{2\}$ are swapped when we switch from $T_1$ to $T_2$. Therefore, assuming $c_2\neq c_3$, the differences $\Delta\varphi=\varphi_1-\varphi_2$ of the Shapley values (cf. \eqref{Delta}) for these two games have opposite signs. 
This means that, over the top-left square, i.e. for roughly $25\%$ of data points, the eject TreeSHAP applied to $T_1$ ranks features differently than the eject TreeSHAP applied to $T_2$.
\end{example}

\begin{figure}
\includegraphics[width=14cm]{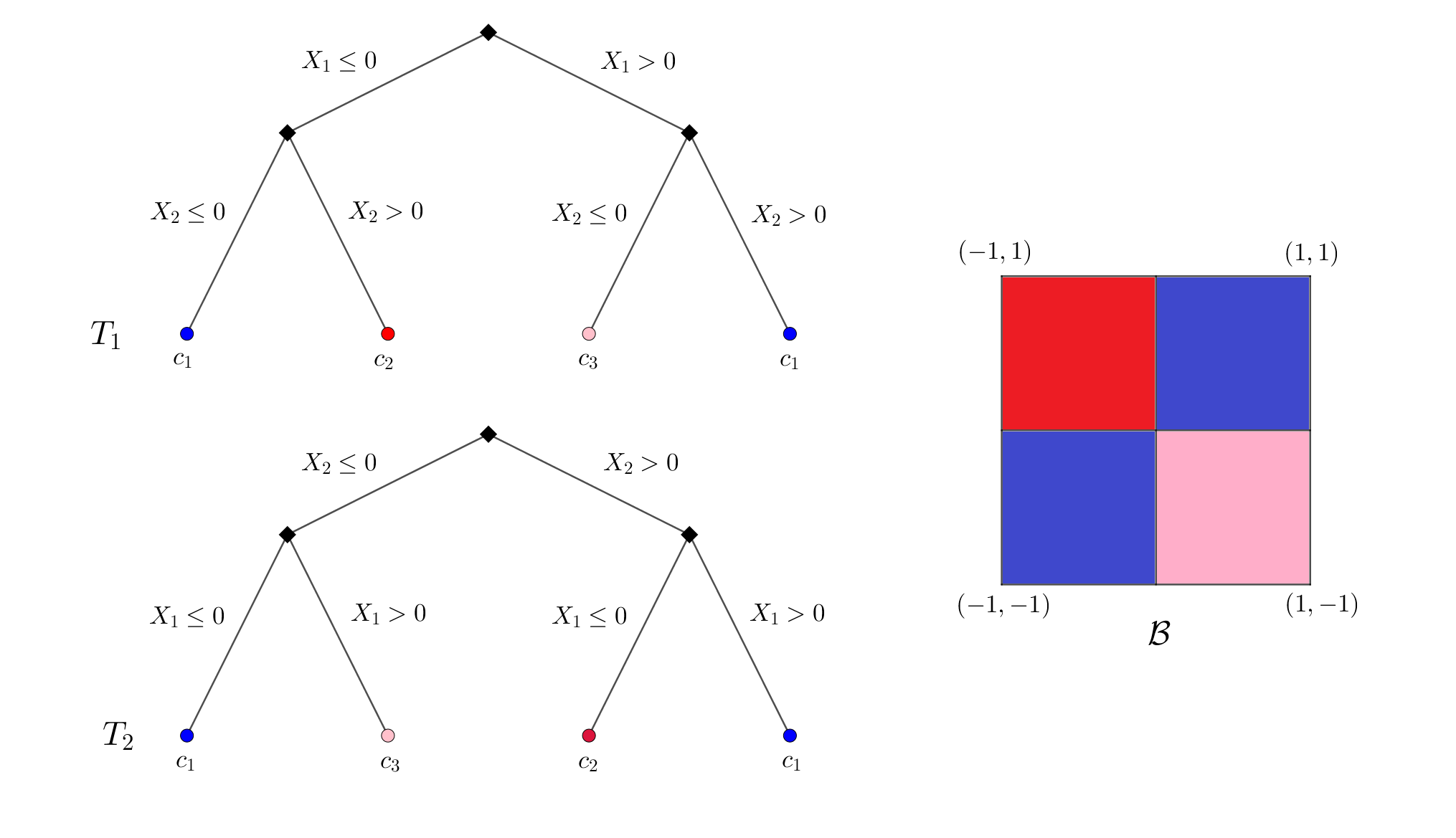}
\caption{The picture for Example \ref{Eject example} demonstrating that eject TreeSHAP (\cite{campbell2022exact}) can depend on the model make-up. For the two decision trees on the left, the splits are the same but occur in different orders.
The trees compute the same function 
$g=c_1\cdot\mathbbm{1}_{[-1,0]\times[-1,0]}
+c_2\cdot\mathbbm{1}_{[-1,0]\times[0,1]}
+c_3\cdot\mathbbm{1}_{[0,1]\times[-1,0]}
+c_1\cdot\mathbbm{1}_{[0,1]\times[0,1]}$
and determine the partition on the right of $\mathcal{B}=[-1,1]\times[-1,1]$, where the features are supported, into four subsquares. Under the assumption that these subsquares are equally probable, the associated eject TreeSHAP games (cf. Definition \ref{Eject definition}) 
are presented in \eqref{eject games}. The rankings of features based on the Shapley values of these games are never the same for inputs from the top-left subsquare (unless $c_2=c_3$).} 
\label{fig:failure2}
\end{figure}

\section{Technical proofs}\label{appendix:proofs}

\subsection{Proofs for \Sec \ref{subsec:theorem}}\label{subappendix:simple}
\begin{proof}[Proof of Theorem \ref{main theorem}]
Let us first consider the case of the marginal game. Writing the function $g$ computed by $T$ as 
$g=c_1\cdot\mathbbm{1}_{R_1}+\dots+c_\ell\cdot\mathbbm{1}_{R_\ell}$, 
the  partition $\mathscr{P}(T)$ of $\mathcal{B}$ is given by $\{R_1,\dots,R_\ell\}$.
The marginal game is linear:
$$
v^\ME(S;\mathbf{X},g)=\sum_{i=1}^\ell c_i\cdot v^\ME(S;\mathbf{X},\mathbbm{1}_{R_i}).
$$
Therefore, it suffices to prove the first part in the case of a simple function $\mathbbm{1}_R$ where 
$R\in\mathscr{P}(T)$. We shall show that for any $S\subseteq N$ and any grid element 
$\tilde{R}\in\widetilde{\mathscr{P}(T)}$, 
the function $\mathbf{x}\mapsto v^\ME(S;\mathbf{X},\mathbbm{1}_R)(\mathbf{x})$ 
is constant on the interior of $\tilde{R}$. This will imply that $v^\ME(S;\mathbf{X},\mathbbm{1}_R)$ is ${\rm{P}}_\mathbf{X}$-a.s. constant since, due to hypothesis \eqref{assumption}, the boundary of each grid element is of measure zero, i.e. 
${\rm{P}}_\mathbf{X}(\partial\tilde{R})=0$. 
Using the definition \eqref{games} of the marginal game, one has
\begin{equation}\label{auxiliary8}
v^\ME(S;\mathbf{X},\mathbbm{1}_R)(\mathbf{x})=\Bbb{E}[\mathbbm{1}_R(x_S,\mathbf{X}_{-S})]
=\mathbbm{1}_{R_S}(x_S)\cdot\Bbb{E}[\mathbbm{1}_{R_{-S}}(\mathbf{X}_{-S})]
\end{equation}
where, following the notation in \Sec \ref{subsec:convention}, $R_S$ and $R_{-S}$ are projections of the rectangle $R\subset\Bbb{R}^n$ onto coordinates $(x_i)_{i\in S}$ and $(x_i)_{i\in N\setminus S}$ respectively. 
A key point to notice is that for 
any two rectangles $R'\in\mathscr{P}(T)$ and $\tilde{R}'\in\widetilde{\mathscr{P}(T)}$, and any $S'\subseteq N$, 
either $\tilde{R}'_{S'}\subseteq R'_{S'}$ or the interiors of $R'_{S'}$ and $\tilde{R}'_{S'}$ are disjoint. 
This is due to the fact that the grid $\widetilde{\mathscr{P}(T)}$ is the product of the partitions determined by projections of $\mathscr{P}(T)$ along various dimensions; see Figure \ref{fig:partition}.
We conclude that when $\mathbf{x}\in{\rm{int}}(\tilde{R})$, \eqref{auxiliary8} may be written as
$$
v^\ME(S;\mathbf{X},\mathbbm{1}_R)(\mathbf{x})
=\begin{cases}
\sum_{\left\{\tilde{R}'\in\widetilde{\mathscr{P}(T)}\mid \tilde{R}'_{-S}\subseteq R_{-S}\right\}}{\rm{P}}_{\mathbf{X}}(\tilde{R}')& 
\text{if } \tilde{R}_S\subseteq R_S,\\
0&\text{otherwise.}
\end{cases}
$$
Consequently, $\mathbf{x}\in {\rm{int}}(\tilde{R})\mapsto v^\ME(S;\mathbf{X},\mathbbm{1}_R)(\mathbf{x})$
is a constant function whose value is a linear combination of elements of
$\left\{{\rm{P}}_{\mathbf{X}}(\tilde{R}')\mid \tilde{R}'\in\widetilde{\mathscr{P}(T)}\right\}$.\\
\indent
Next, we establish the claims made in part (2) on the TreeSHAP game $v^\Tree(\cdot;T)$. 
As before, we denote the regions corresponding to the leaves of $T$ by $R_1,\dots,R_\ell$ and the leaf scores by 
$c_1,\dots,c_\ell$. Moreover, the paths from the root of $T$ to the leaves are denoted by 
$\mathsf{P}_1,\dots,\mathsf{P}_\ell$. 
The recursive formula \eqref{TreeSHAP game} for $v^\Tree(\cdot;T)$ can be written explicitly in terms of these paths as 
\begin{equation}\label{TreeSHAP game explicit} 
v^\Tree(S;T)(\mathbf{x})=\sum_{i=1}^\ell c_i\cdot\tau(S;\mathsf{P}_i)(\mathbf{x}),
\end{equation}
where for any path $\mathsf{P}$ from the root of $T$ to a leaf and any $S\subseteq N$, the function 
$\mathbf{x}\mapsto\tau(S;\mathsf{P})(\mathbf{x})$ is defined  as follows: 
Starting from the root, write the non-leaf nodes of $\mathsf{P}$ as $\mathsf{v}_1,\dots,\mathsf{v}_{m}$
and the leaf where $\mathsf{P}$ terminates as $\mathsf{v}_{m+1}$.
For any $1\leq u\leq m$,
record the split at $\mathsf{v}_u$ as 
$$
\left(X_{j_u},t_{j_u},\epsilon_u\right) (j_u\in N, t_{j_u}\in\Bbb{R}\text{ and }\epsilon_u\in\{\pm 1\})
$$
where, to follow the path and get from $\mathsf{v}_u$ to the next node of $\mathsf{P}$, one should have 
$\epsilon_u(X_{j_u}-t_{j_u})>0$, rather than the alternative $\epsilon_u(X_{j_u}-t_{j_u})<0$.
(Keep in mind that the events $X_{j_u}=t_{j_u}$ are of probability zero due to our hypothesis \eqref{assumption}, and are neglected henceforth.)
The function $\tau(S;\mathsf{P})$ is now defined as
\begin{equation}\label{path function}
\begin{split}
&\tau(S;\mathsf{P}):=\prod_{u=1}^m\tau^{(u)}(S;\mathsf{P}) \text{ where }
\tau^{(u)}(S;\mathsf{P})(\mathbf{x}):=
\begin{cases}
1&\text{ if }j_u\in S \text{ and }\epsilon_u(x_{j_u}-t_{j_u})>0,\\
0&\text{ if }j_u\in S \text{ and }\epsilon_u(x_{j_u}-t_{j_u})<0,\\
w_u(\mathsf{P})&\text{ otherwise;}
\end{cases}\\
&\text{ in which }w_u(\mathsf{P}):=
\frac{|\{\mathbf{x}\in D\mid \epsilon_1(x_{j_1}-t_{j_1})>0,\dots,\epsilon_u(x_{j_u}-t_{j_u})>0\}|}
{|\{\mathbf{x}\in D\mid \epsilon_1(x_{j_1}-t_{j_1})>0,\dots,\epsilon_{u-1}(x_{j_{u-1}}-t_{j_{u-1}})>0\}|}.    
\end{split}
\end{equation}
Notice that $w_u(\mathsf{P})$ is the following ratio:
$$
w_u(\mathsf{P})=\frac{\text{ number of data points ended up at }\mathsf{v}_{u+1}}
{\text{ number of data points ended up at }\mathsf{v}_{u}}
=\frac{\sum_{R\in\mathscr{P}(T) \text{ corresponds to a leaf descendant of }\mathsf{v}_{u+1}}\hat{{\rm{P}}}_{\mathbf{X}}(R)}
{\sum_{R\in\mathscr{P}(T) \text{ corresponds to a leaf descendant of }\mathsf{v}_{u}}\hat{{\rm{P}}}_{\mathbf{X}}(R)}.
$$
We now finish the proof. Throughout the interior of any grid element $\tilde{R}\in\widetilde{\mathscr{P}(T)}$,
each expression $\epsilon_u(x_{j_u}-t_{j_u})$ is either always positive or always negative. Hence the function $\tau(S;\mathsf{P})$ does not vary therein. Its value  on ${\rm{int}}(\tilde{R})$
($0$, $1$ or a product of numbers $w_u(\mathsf{P})$)
is a rational expression of terms $\hat{{\rm{P}}}_{\mathbf{X}}(R)\,(R\in \mathscr{P}(T))$. 
All these hold for the restriction of 
$v^\Tree(S;T)$ to ${\rm{int}}(\tilde{R})$ too since $v^\Tree(S;T)$ is a linear combination of functions 
$\tau(S;\mathsf{P})$ as in \eqref{TreeSHAP game explicit}.
\end{proof}

\subsection{Proofs for \Sec \ref{subsec:CatBoost}}\label{subappendix:CatBoost}
We begin with the following combinatorial lemma which is used multiple times in our proofs.
\begin{lemma}\label{hockey-stick}
Let $N$ be a finite set and $Z\subseteq W$ two proper subsets of $N$. Then
$$
\sum_{Z\subseteq S\subseteq W}\frac{|S|!\,(|N|-|S|-1)!}{|N|!}=\frac{|Z|!\,(|N|-|W|-1)!}{(|N|+|Z|-|W|)!}.
$$
\end{lemma}
\begin{proof}
One way to establish this identity is by induction on $|W\setminus Z|$. 
For a direct proof, denote $|S|,|N|,|Z|,|W|$ by $s,n,z,w$ respectively. We then have 
\begin{equation*}
\begin{split}
\sum_{Z\subseteq S\subseteq W}\frac{|S|!\,(|N|-|S|-1)!}{|N|!}
&=\sum_{s=z}^w\frac{s!(n-s-1)!}{n!}\binom{w-z}{s-z}
=\frac{z!(w-z)!(n-w-1)!}{n!}\sum_{s=z}^w\binom{s}{z}\binom{n-s-1}{n-w-1}\\
&=\frac{z!(w-z)!(n-w-1)!}{n!}\binom{n}{n+z-w}=\frac{z!(n-w-1)!}{(n+z-w)!};
\end{split}    
\end{equation*}
where on the last line we have used
$\sum_{s=z}^w\binom{s}{z}\binom{n-s-1}{n-w-1}=\binom{n}{n+z-w}$
which is a generalization of the \textit{hockey-stick identity} for binomial coefficients. 
\end{proof}

\begin{proof}[Proof of Theorem \ref{Catboost theorem}]
By the virtue of \eqref{fewer features},
formula \eqref{symmetric formula ensemble} for $\varphi_i\big[v^\ME\big]$
may be obtained from the case of a single tree. 
Therefore, we shall derive a formula for marginal Shapley and Banzhaf values of a single oblivious decision tree $T$.
With the notation as in Definition-Notation \ref{groundwork}, suppose $T$ is of depth $m$ and
write the distinct features on which the tree splits as 
$X'_1,\dots,X'_k$.
So the tree $T$ implements a simple function $g:\Bbb{R}^k\rightarrow\Bbb{R}$; and the goal is to investigate Shapley and Banzhaf values of the corresponding  marginal game   
$$
v^\ME=v^\ME(\cdot,\mathbf{X}',g)\quad (\mathbf{X}':=(X'_1,\dots,X'_k))
$$
which to each $\mathbf{x}\in\Bbb{R}^k$ assigns a game played on $K=\{1,\dots,k\}$. 
We will suppress $\mathbf{X}',g$ for the sake of brevity hereafter.\\
\indent
The leaves of $T$ can be encoded by elements of $\{0,1\}^m$. The subset of realizable codes is denoted by 
$\mathcal{R}\subseteq\{0,1\}^m$ which captures leaves that determine non-vacuous regions. The leaf score and
the rectangular region in $\Bbb{R}^k$ corresponding to a binary code $\mathbf{b}\in\{0,1\}^m$ are written as $c_{\mathbf{b}}$ and $R_{\mathbf{b}}$ respectively. We set 
$$p_{\mathbf{b}}:=\Bbb{P}\left(\mathbf{X}'\in R_{\mathbf{b}}\right)
=\Bbb{E}\left[\mathbbm{1}_{R_\mathbf{b}}(\mathbf{X}')\right].$$
The grid partition of the hypercube $\mathcal{B}\subset\Bbb{R}^k$, where the features are supported, is
given by 
$$\mathscr{P}(T)=\widetilde{\mathscr{P}(T)}=\{R_{\mathbf{b}}\}_{\mathbf{b}\in\mathcal{R}}.$$
One can write $g$ as 
$g(\mathbf{x})=\sum_{\mathbf{b}\in\mathcal{R}}c_{\mathbf{b}}\cdot\mathbbm{1}_{R_\mathbf{b}}(\mathbf{x})$. 
Thus for any $\mathbf{x}\in\Bbb{R}^k$ and $Q\subseteq K$:
\begin{equation}\label{auxiliary9}
v^\ME(Q)(\mathbf{x})=
\sum_{\mathbf{b}\in\mathcal{R}}c_{\mathbf{b}}\cdot
\Bbb{E}\left[\mathbbm{1}_{R_\mathbf{b}}(\mathbf{x}_Q,\mathbf{X}'_{-Q})\right]
=\sum_{\mathbf{b}\in\mathcal{R}}c_{\mathbf{b}}\cdot
\mathbbm{1}_{R_{\mathbf{b},Q}}(\mathbf{x}_Q)\cdot
\Bbb{P}\left(\mathbf{X}'_{-Q}\in R_{\mathbf{b},-Q}\right),    
\end{equation}
where the subscripts $Q$ and $-Q$ denote projections 
$(x_i)_{i\in K}\mapsto(x_i)_{i\in Q}$ and $(x_i)_{i\in K}\mapsto(x_i)_{i\in K\setminus Q}$
respectively. 
Boundaries of regions cut by $T$ are of measure zero due to \eqref{assumption}. Hence ${\rm{P}}_{\mathbf{X}}$-a.e. data point $\mathbf{x}\in\Bbb{R}^k$ belongs to the interior of a grid element; say 
$\mathbf{x}\in {\rm{int}}(R_{\mathbf{a}})$ where $\mathbf{a}\in\mathcal{R}$.
Projections $R_{\mathbf{a},Q}$ and $R_{\mathbf{b},Q}$ of two different grid elements have an interior point in common if and only if the paths from the root to the leaves encoded by $\mathbf{a}$ and $\mathbf{b}$ split similarly whenever a feature $X'_i$ with $i\in Q$ is encountered. Recall that the levels of $T$ on which  features $X'_1,\dots,X'_k$ appear determine a partition $\mathsf{p}=\{S_1,\dots,S_k\}$ of $M$.
Therefore, the interiors of $R_{\mathbf{a},Q}$ and $R_{\mathbf{b},Q}$ intersect non-trivially if and only if
$\mathbf{a}_{\mathsf{p}(Q)}=\mathbf{b}_{\mathsf{p}(Q)}$ where $\mathsf{p}(Q)$ is the disjoint union 
\begin{equation}\label{partition_set}
\mathsf{p}(Q):=\cup_{q\in Q}S_q.
\end{equation}
Similarly, the projection onto coordinates in $K\setminus Q$ belongs to ${\rm{int}}(R_{\mathbf{b},-Q})$ 
precisely when on levels indexed by elements of $-\mathsf{p}(Q):=M\setminus \mathsf{p}(Q)$ we split in the same manner that the path to the leaf encoded by $\mathbf{b}$ splits. 
In view of all this, \eqref{auxiliary9} yields:
\begin{equation}\label{auxiliary10}
v^\ME(Q)(\mathbf{x})=
\sum_{\substack{\mathbf{b}\in\mathcal{R}\\
\mathbf{b}_{\mathsf{p}(Q)}=\mathbf{a}_{\mathsf{p}(Q)}}}
c_{\mathbf{b}}\cdot
\Bigg(\sum_{\substack{\mathbf{u}\in\mathcal{R}\\
\mathbf{u}_{-\mathsf{p}(Q)}=\mathbf{b}_{-\mathsf{p}(Q)}}}
p_{\mathbf{u}}\Bigg),
\quad Q\subseteq K \text{ and }\mathbf{x}\in{\rm{int}}(R_{\mathbf{a}}).
\end{equation}
Fix $\mathbf{a}\in\mathcal{R}$, $\mathbf{x}\in{\rm{int}}(R_\mathbf{a})$ and $i\in K$. 
We will use \eqref{auxiliary10} to simplify 
\begin{equation}\label{auxiliary11}
\varphi_i\big[v^\ME\big](\mathbf{x})=
\sum_{Q\subseteq K\setminus\{i\}}\frac{|Q|!\,(|K|-|Q|-1)!}{|K|!}
\left(v^\ME(Q\cup\{i\})(\mathbf{x})-v^\ME(Q)(\mathbf{x})\right).
\end{equation}
Notice that for any $Q\subseteq K\setminus\{i\}$:
\begin{equation*}
v^\ME(Q)(\mathbf{x})
=\sum_{\substack{\mathbf{b}\in\mathcal{R}\\
\mathbf{b}_{S_i}=\mathbf{a}_{S_i}\\
\mathbf{b}_{\mathsf{p}(Q)}=\mathbf{a}_{\mathsf{p}(Q)}}}
\sum_{\substack{\mathbf{u}\in\mathcal{R}\\
\mathbf{u}_{S_i}=\mathbf{b}_{S_i}\\
\mathbf{u}_{-\mathsf{p}(Q\cup\{i\})}=\mathbf{b}_{-\mathsf{p}(Q\cup\{i\})}}}
c_{\mathbf{b}}\,p_{\mathbf{u}}
+\sum_{\substack{\mathbf{b}\in\mathcal{R}\\
\mathbf{b}_{S_i}\neq \mathbf{a}_{S_i}\\
\mathbf{b}_{\mathsf{p}(Q)}=\mathbf{a}_{\mathsf{p}(Q)}}}
\sum_{\substack{\mathbf{u}\in\mathcal{R}\\
\mathbf{u}_{S_i}=\mathbf{b}_{S_i}\\
\mathbf{u}_{-\mathsf{p}(Q\cup\{i\})}=\mathbf{b}_{-\mathsf{p}(Q\cup\{i\})}}}
c_{\mathbf{b}}\,p_{\mathbf{u}};
\end{equation*}
and
\begin{equation*}
v^\ME(Q\cup\{i\})(\mathbf{x})
=\sum_{\substack{\mathbf{b}\in\mathcal{R}\\
\mathbf{b}_{S_i}=\mathbf{a}_{S_i}\\
\mathbf{b}_{\mathsf{p}(Q)}=\mathbf{a}_{\mathsf{p}(Q)}}}
\sum_{\substack{\mathbf{u}\in\mathcal{R}\\
\mathbf{u}_{S_i}=\mathbf{b}_{S_i}\\
\mathbf{u}_{-\mathsf{p}(Q\cup\{i\})}=\mathbf{b}_{-\mathsf{p}(Q\cup\{i\})}}}
c_{\mathbf{b}}\,p_{\mathbf{u}}
+\sum_{\substack{\mathbf{b}\in\mathcal{R}\\
\mathbf{b}_{S_i}=\mathbf{a}_{S_i}\\
\mathbf{b}_{\mathsf{p}(Q)}=\mathbf{a}_{\mathsf{p}(Q)}}}
\sum_{\substack{\mathbf{u}\in\mathcal{R}\\
\mathbf{u}_{S_i}\neq\mathbf{b}_{S_i}\\
\mathbf{u}_{-\mathsf{p}(Q\cup\{i\})}=\mathbf{b}_{-\mathsf{p}(Q\cup\{i\})}}}
c_{\mathbf{b}}\,p_{\mathbf{u}}.
\end{equation*}
Substituting in \eqref{auxiliary11}:
\begin{equation}\label{auxiliary12}
\begin{split}
\varphi_i\big[v^\ME\big](\mathbf{x})=
\sum_{Q\subseteq K\setminus\{i\}}\frac{|Q|!\,(|K|-|Q|-1)!}{|K|!}
\Bigg(
&\sum_{\substack{\mathbf{b}\in\mathcal{R}\\
\mathbf{b}_{S_i}\neq \mathbf{a}_{S_i}\\
\mathbf{b}_{\mathsf{p}(Q)}=\mathbf{a}_{\mathsf{p}(Q)}}}
\sum_{\substack{\mathbf{u}\in\mathcal{R}\\
\mathbf{u}_{S_i}=\mathbf{b}_{S_i}\\
\mathbf{u}_{-\mathsf{p}(Q\cup\{i\})}=\mathbf{b}_{-\mathsf{p}(Q\cup\{i\})}}}
c_{\mathbf{b}}\,p_{\mathbf{u}}\\
&-\sum_{\substack{\mathbf{b}\in\mathcal{R}\\
\mathbf{b}_{S_i}=\mathbf{a}_{S_i}\\
\mathbf{b}_{\mathsf{p}(Q)}=\mathbf{a}_{\mathsf{p}(Q)}}}
\sum_{\substack{\mathbf{u}\in\mathcal{R}\\
\mathbf{u}_{S_i}\neq\mathbf{b}_{S_i}\\
\mathbf{u}_{-\mathsf{p}(Q\cup\{i\})}=\mathbf{b}_{-\mathsf{p}(Q\cup\{i\})}}}
c_{\mathbf{b}}\,p_{\mathbf{u}}
\Bigg).    
\end{split}
\end{equation}
\normalsize
The preceding formula can be simplified by noticing that a pair $(\mathbf{b},\mathbf{u})$ of realizable binary codes may come up more than once. In each double summation from \eqref{auxiliary12},  
conditions are posed on bits that belong to $S_i$, and also on those that belong to 
$M\setminus S_i=\mathsf{p}(K\setminus\{i\})=\cup_{q\in K\setminus\{i\}}S_q$. For the latter bits, one should have
\begin{equation}\label{inclusion1}
\{j\in M\setminus S_i\mid u_j\neq b_j\}\subseteq\mathsf{p}(Q)=\cup_{q\in Q}S_q\subseteq 
\{j\in M\setminus S_i\mid b_j=a_j\}
\end{equation}
which is equivalent to
\begin{equation}\label{inclusion2}
\{q\in K\setminus\{i\}\mid \mathbf{u}_{S_q}\neq \mathbf{b}_{S_q}\}
\subseteq Q\subseteq 
\{q\in K\setminus\{i\}\mid \mathbf{b}_{S_q}=\mathbf{a}_{S_q}\}.
\end{equation}
In view of the notation introduced in Definition-Notation \ref{groundwork}, \eqref{inclusion2} may be written as 
\begin{equation}\label{inclusion3}
K\setminus 
\left(\mathcal{E}(\mathbf{b},\mathbf{u};\mathsf{p})\cup\{i\}\right)
\subseteq Q\subseteq 
\mathcal{E}(\mathbf{a},\mathbf{b};\mathsf{p})\setminus\{i\}.    
\end{equation}
Set $Z:=K\setminus\mathcal{E}(\mathbf{b},\mathbf{u};\mathsf{p})$ and 
$W:=\mathcal{E}(\mathbf{a},\mathbf{b};\mathsf{p})$. 
Thus we have $\mathbf{u}\in\mathcal{E}^{-1}(\mathbf{b},-Z,;\mathsf{p})$ and 
$\mathbf{b}\in\mathcal{E}^{-1}(\mathbf{a},W;\mathsf{p})$ where $-Z$ denotes the complement $K\setminus Z$ as usual. 
Next, notice that \eqref{inclusion3} may be rewritten as 
$Z\setminus\{i\}\subseteq Q\subseteq W\setminus\{i\}$.
Furthermore, in the first double summation in \eqref{auxiliary12} we have 
$i\in Z$ while $i\notin W$ in the second one. 
All in all, \eqref{auxiliary12} may be rewritten as the difference below of 
two terms which are expressed in terms of summations over pairs $Z\subseteq W$ of subsets of $K=\{1,\dots,k\}$:
\small
\begin{equation}\label{auxiliary12'}
\begin{split}
&\sum_{\substack{Z\subseteq K\\ i\in Z}}   
\sum_{\substack{W\subseteq K\\
W\supseteq Z}}
\Big(\sum_{\{Q\mid Z\setminus\{i\}\subseteq Q\subseteq W\setminus\{i\}\}}
\frac{|Q|!\,(|K|-|Q|-1)!}{|K|!}\Big)\cdot 
\Bigg(\sum_{\mathbf{b}\in\mathcal{E}^{-1}(\mathbf{a},W;\mathsf{p})\cap\mathcal{R}}
c_{\mathbf{b}}\cdot\Big(\sum_{\mathbf{u}\in\mathcal{E}^{-1}
(\mathbf{b},-Z;\mathsf{p})\cap\mathcal{R}}p_{\mathbf{u}}\Big)\Bigg)\\
&-\sum_{\substack{W\subseteq K\\ i\notin W}}   
\sum_{\substack{Z\subseteq K\\
Z\subseteq W}}
\Big(\sum_{\{Q\mid Z\subseteq Q \subseteq W\}}
\frac{|Q|!\,(|K|-|Q|-1)!}{|K|!}\Big)\cdot 
\Bigg(\sum_{\mathbf{b}\in\mathcal{E}^{-1}(\mathbf{a},W;\mathsf{p})\cap\mathcal{R}}
c_{\mathbf{b}}\cdot\Big(\sum_{\mathbf{u}\in\mathcal{E}^{-1}
(\mathbf{b},-Z;\mathsf{p})\cap\mathcal{R}}p_{\mathbf{u}}\Big)\Bigg).    
\end{split}
\end{equation}
\normalsize
Invoking Lemma \ref{hockey-stick}, the combinatorial coefficients on the first and second lines of 
\eqref{auxiliary12'}
respectively become $\omega^+(|W|,|Z|;k)$ and $\omega^-(|W|,|Z|;k)$
as defined in \eqref{weights_Shapley}.
Moreover, 
$c_{\mathbf{b}}\cdot\Big(\sum_{\mathbf{u}\in\mathcal{E}^{-1}(\mathbf{b},-Z;\mathsf{p})}p_\mathbf{u}\Big)$ 
from \eqref{auxiliary12'} can be denoted by 
$\mathfrak{s}(\mathbf{b},-Z;T)$ following the notation in \eqref{auxiliary5''}.
Substituting all these,  \eqref{auxiliary12'} becomes 
the difference 
$\phi^+(\mathbf{a};i,T)-\phi^-(\mathbf{a};i,T)$. 
This establishes \eqref{symmetric formula ensemble} in the case of a single decision tree.
As for the Banzhaf value $Bz_i\big[v^\ME\big](\mathbf{x})$, the only difference is that 
the fraction $\frac{|Q|!\,(|K|-|Q|-1)!}{|K|!}$ in \eqref{auxiliary12'} should be changed to $\frac{1}{2^{k-1}}$
which amounts to replacing 
$\omega^+(|W|,|Z|;k)$ and $\omega^-(|W|,|Z|;k)$
with 
$\tilde{\omega}(|W|,|Z|;k)$. \\
\indent
To conclude the proof of Theorem \ref{Catboost theorem}, it remains to 
verify that the complexity bounds 
\eqref{complexity0} and \eqref{complexity multiplication}. 
Rewriting \eqref{auxiliary12'} as 
\begin{equation}\label{auxiliary formula}
\begin{split}
&\sum_{\substack{Z\subseteq K\\ i\in Z}}   
\sum_{\substack{W\subseteq K\\
W\supseteq Z}}
\,\sum_{\mathbf{b}\in\mathcal{E}^{-1}(\mathbf{a},W;\mathsf{p})\cap\mathcal{R}}
\,\sum_{\mathbf{u}\in\mathcal{E}^{-1}
(\mathbf{b},-Z;\mathsf{p})\cap\mathcal{R}}
\omega^+(|W|,|Z|;k)\cdot c_{\mathbf{b}}\cdot p_{\mathbf{u}}\\
&-\sum_{\substack{W\subseteq K\\ i\notin W}}   
\sum_{\substack{Z\subseteq K\\
Z\subseteq W}}
\,\sum_{\mathbf{b}\in\mathcal{E}^{-1}(\mathbf{a},W;\mathsf{p})\cap\mathcal{R}}
\,\sum_{\mathbf{u}\in\mathcal{E}^{-1}
(\mathbf{b},-Z;\mathsf{p})\cap\mathcal{R}}
\omega^-(|W|,|Z|;k)\cdot c_{\mathbf{b}}\cdot p_{\mathbf{u}},  
\end{split}    
\end{equation}
we should prove that the number of pairs $(\mathbf{b},\mathbf{u})$
appearing on each line of \eqref{auxiliary formula} does not exceed \eqref{complexity0}.
We shall show it for the first one, the other one is similar. 
First, notice that when some features occur more than once in the tree, i.e. $k<m$, then the subset $\mathcal{R}$ which captures paths with non-conflicting thresholds can become much smaller than $\{0,1\}^m$. The levels where the tree splits on the $q^{\rm{th}}$ feature $X'_q$ $(q\in K)$ are indexed with the subset $S_q$ of $M$. Having $|S_q|$ thresholds along the $q^{\rm{th}}$ dimension cuts that axis into $|S_q|+1$ intervals, hence $|S_q|+1$ choices for $\mathbf{e}_{S_q}$ if a binary code $\mathbf{e}\in\{0,1\}^m$ is to be realizable. 
So $\mathbf{e}\in\{0,1\}^m$ lies in $\mathcal{R}$ if and only if for any $q\in K$ the segment $\mathbf{e}_{S_q}$  
is among those $|S_q|+1$ elements of $\{0,1\}^{S_q}$ that are ``admissible''. Notice that 
$|\mathcal{R}|=\prod_{q\in K}(|S_q|+1)$ can be much smaller than $2^m=2^{\sum_{q\in K}|S_q|}$.
We derive the bound \eqref{complexity0} by counting how many times pairs $(\mathbf{b},\mathbf{u})$ of realizable binary codes with 
$\mathbf{b}\in\mathcal{E}^{-1}(\mathbf{a},W;\mathsf{p})$ and $\mathbf{u}\in\mathcal{E}^{-1}(\mathbf{b},-Z;\mathsf{p})$
come up where $Z\subseteq W$ are two prescribed subsets of $K$.
For $\mathbf{b}\in\mathcal{E}^{-1}(\mathbf{a},W;\mathsf{p})$, or equivalently
$\mathcal{E}(\mathbf{a},\mathbf{b};\mathsf{p})=W$, to hold, 
one should have $\mathbf{b}_{S_q}=\mathbf{a}_{S_q}$ for any $q\in W$. 
To determine $\mathbf{b}$, it remains to pick segments $\mathbf{b}_{S_q}$ for any $q\in K\setminus W$. 
There are only $|S_q|+1$ admissible choices from which $\mathbf{a}_{S_q}$ must be excluded since $q\notin W$;
hence $\prod_{q\in K\setminus W}|S_q|$ total possibilities for $\mathbf{b}$. Now assuming that $\mathbf{b}$ is known, we count the number of possibilities for $\mathbf{u}$ if $\mathbf{u}\in\mathcal{E}^{-1}(\mathbf{b},-Z;\mathsf{p})$, i.e. $\mathcal{E}(\mathbf{b},\mathbf{u};\mathsf{p})=K\setminus Z$. 
Whenever $q\in Z$, there are $(|S_q|+1)-1=|S_q|$ choices for $\mathbf{u}_{S_q}$
because it must be admissible and different from $\mathbf{b}_{S_q}$. 
As for $q\in K\setminus Z$, $\mathbf{u}_{S_q}$ must coincide with $\mathbf{b}_{S_q}$. In conclusion:
\begin{equation}\label{possibilities}
\big|\left\{(\mathbf{b},\mathbf{u})\mid \mathbf{b}\in
\mathcal{E}^{-1}(\mathbf{a},W;\mathsf{p})\cap\mathcal{R}, 
\mathbf{u}\in\mathcal{E}^{-1}(\mathbf{b},-Z;\mathsf{p})\cap\mathcal{R}\right\}\big|=
\prod_{q\in K\setminus W}|S_q|\cdot \prod_{q\in Z}|S_q|
=\prod_{q\in K\setminus(W\setminus Z)}|S_q|
\end{equation}
\normalsize
for any $Z\subseteq W\subseteq K$. The AM-GM inequality\footnote{The \textit{inequality of arithmetic and geometric means} states that 
$\sqrt[p]{z_1\dots z_p}\leq\frac{z_1+\dots+z_p}{p}$ for non-negative numbers $z_1,\dots,z_p$.} provides an estimate for the cardinality above:
\begin{equation}\label{AM-GM}
\prod_{q\in K\setminus(W\setminus Z)}|S_q|
\leq \prod_{q\in K}|S_q|
\leq 
\left(\frac{\sum_{q\in K}|S_q|}{|K|}\right)^{|K|}=
\left(\frac{m}{k}\right)^k.
\end{equation}
\normalsize
This may be used to bound the total number of summands
$$
\sum_{\substack{Z\subseteq K\\ i\in Z}}   
\sum_{\substack{W\subseteq K\\
W\supseteq Z}}
\prod_{q\in K\setminus(W\setminus Z)}|S_q| 
$$
on the first line of \eqref{auxiliary formula},
as well as the total number of summands 
$$
\sum_{\substack{W\subseteq K\\ i\notin W}}   
\sum_{\substack{Z\subseteq K\\
Z\subseteq W}}
\prod_{q\in K\setminus(W\setminus Z)}|S_q| 
$$
on the second line of \eqref{auxiliary formula}.
In each case, the number of pairs $(Z,W)$ of nested subsets is $3^{|K|-1}=3^{k-1}$. This along with \eqref{AM-GM} yields 
$$
\sum_{\substack{Z\subseteq K\\ i\in Z}}   
\sum_{\substack{W\subseteq K\\
W\supseteq Z}}
\prod_{q\in K\setminus(W\setminus Z)}|S_q|,
\sum_{\substack{W\subseteq K\\ i\notin W}}   
\sum_{\substack{Z\subseteq K\\
Z\subseteq W}}
\prod_{q\in K\setminus(W\setminus Z)}|S_q|
\leq 3^{k-1}\cdot\left(\frac{m}{k}\right)^k.
$$
Next, notice that $3^{k-1}\cdot\left(\frac{m}{k}\right)^k\leq 3^{m-1}$. This is clear when $k=m$ and follows from $\left(1+\frac{m-k}{k}\right)^{\frac{k}{m-k}}<e<3$ when $0<k<m$. We have thus derived the bound \eqref{complexity0} for the total number of terms in expansions of $\phi^+(\mathbf{a};i,T)$ and $\phi^-(\mathbf{a};i,T)$. 
We claim that the number of multiplications required for computing each of them is at most 
$\left(2+\frac{m}{k}\right)^{k}$, i.e. the bound from \eqref{complexity multiplication}.
Computing any term of the form $\mathfrak{s}(\mathbf{e},Q;T)$ (cf. \eqref{auxiliary5''}) involves only one multiplication. Therefore, the total number of multiplications needed for calculating the two lines of \eqref{contributions} are
\small
\begin{equation}
\begin{split}
& \sum_{\substack{Z\subseteq W\\ i\in Z}}   
\sum_{\substack{W\subseteq K\\
W\supseteq Z}}\Big(1+\big|\mathcal{E}^{-1}(\mathbf{a},W;\mathsf{p})\cap\mathcal{R}\big|\Big)
=3^{k-1}+\sum_{\substack{W\subseteq K\\ i\in W}}2^{|W|-1}\cdot\prod_{q\in K\setminus W}|S_q|
=3^{k-1}+\prod_{q\in K\setminus\{i\}}(|S_q|+2),\\
&\sum_{\substack{W\subseteq K\\ i\notin W}}   
\sum_{\substack{Z\subseteq K\\
Z\subseteq W}}
\Big(1+\big|\mathcal{E}^{-1}(\mathbf{a},W;\mathsf{p})\cap\mathcal{R}\big|\Big)
=3^{k-1}+\sum_{\substack{W\subseteq K\\ i\notin W}}2^{|W|-1}\cdot\prod_{q\in K\setminus W}|S_q|
=3^{k-1}+|S_i|\cdot\prod_{q\in K\setminus\{i\}}(|S_q|+2).
\end{split}    
\end{equation}
\normalsize
Each of the numbers above is less than $\prod_{q\in K}(|S_q|+2)$, which is no larger than $\left(2+\frac{m}{k}\right)^{k}$ 
by the AM-GM inequality. This finishes the proof. 
\end{proof}

\begin{proof}[Time-complexity analysis for Algorithm \ref{algorithm}]
To establish the per-leaf complexity $O\left(|\mathcal{T}|\cdot \mathcal{L}^{\log_2 3}\cdot\log(\mathcal{L})\right)$  of the precomputation step claimed in Table \ref{Tab: complexity}, we shall show that the total time complexity of precomputation is 
$O\left(|\mathcal{T}|\cdot \mathcal{L}^{\log_2 6}\cdot\log(\mathcal{L})\right)$.
A prerequisite for this step is 
the combinatorial data 
\begin{equation}\label{auxiliary13}
\big\{\big\{\mathcal{E}^{-1}(\mathbf{e},Q;\mathsf{p}(T))\big\}_{\mathbf{e}\in\{0,1\}^{m(T)},Q\subseteq K(T)}\big\}_{T\in\mathcal{T}}.    
\end{equation}
To obtain this, one can first compute and record $\mathcal{E}(\mathbf{e},\mathbf{e}';\mathsf{p}(T))$ as $T$ varies in 
$\mathcal{T}$, and $\mathbf{e},\mathbf{e}'$ come from $\{0,1\}^{m(T)}$. There are no more than $|\mathcal{T}|\cdot\mathcal{L}^2$ possibilities for triples $(\mathbf{e},\mathbf{e}',T)$; and given such a triple, 
$\mathcal{E}(\mathbf{e},\mathbf{e}';\mathsf{p}(T))$ may be obtained with at most $m(T)\leq\log_2(\mathcal{L})$
comparison of bits. Once 
$\big\{\big\{\mathcal{E}(\mathbf{e},\mathbf{e}',T)\big\}_{\mathbf{e},\mathbf{e}'\in\{0,1\}^{m(T)}}\big\}_{T\in\mathcal{T}}$
is computed and recorded, one can traverse over it 
to get \eqref{auxiliary13}. The complexity so far has been 
$O\left(|\mathcal{T}|\cdot \mathcal{L}^2\cdot\log(\mathcal{L})\right)=o\left(|\mathcal{T}|\cdot \mathcal{L}^{\log_2 6}\cdot\log(\mathcal{L})\right)$.
Next, provided with the appropriate inputs, the algorithm computes 
$$\{\{\boldsymbol{\hat{\phi}}(\mathbf{a};T)=(\hat{\phi}_i(\mathbf{a};T))_{i\in K(T)}\}_{\mathbf{a}\in\mathcal{R}(T)}\}_{T\in\mathcal{T}}$$
through the formula presented in Theorem \ref{Catboost theorem} with true leaf probabilities $p(\mathbf{u};T)$ replaced with estimated probabilities $\hat{p}(\mathbf{u};T)$. As established in that theorem, 
the number of arithmetic operations involved in computing $\hat{\phi}_i(\mathbf{a};T)$ 
does not exceed $3^{m(T)}\leq\mathcal{L}^{\log_2 3}$ times a constant.  
Each tree $T\in\mathcal{T}$ has at most $2^{m(T)}\leq\mathcal{L}$ leaves (encoded by $\mathbf{a}$), and splits on at most 
$k(T)\leq m(T)\leq\log_2(\mathcal{L})$ distinct features (captured by elements $i$ of $K(T)$). We therefore arrive at the total time complexity $O\left(|\mathcal{T}|\cdot \mathcal{L}^{\log_2 6}\cdot\log(\mathcal{L})\right)$
for precomputing 
$\{\{\boldsymbol{\hat{\phi}}(\mathbf{a};T)\}_{\mathbf{a}\in\mathcal{R}(T)}\}_{T\in\mathcal{T}}$.\\
\indent Finally, the time complexity of the on-the-fly stage is $O(|\mathcal{T}|\cdot\log(\mathcal{L}))$: 
For any explicand $\mathbf{x}$, one can determine the leaf of a given oblivious tree at which $\mathbf{x}$ ends up by at most $\log_2(\mathcal{L})$ comparisons.
Hence the complexity of obtaining 
$\{\mathbf{a}=\{\mathbf{a}(\mathbf{x};T)\}_{\mathbf{a}\in\mathcal{R}(T)}\}_{T\in\mathcal{T}}$ is
$O(|\mathcal{T}|\cdot\log(\mathcal{L}))$. 
Then, based on them, precomputed numbers $\hat{\phi}_i(\mathbf{a};T)$ should be added suitably to output the vector $\boldsymbol{\hat{\varphi}}$ of estimated marginal Shapley values at $\mathbf{x}$. 
Here, $T$ varies in $\mathcal{T}$ and $i$ varies among the distinct features appearing in $T$. The total number of addition operations involved is thus no more than $|\mathcal{T}|\cdot\log_2(\mathcal{L})$.
\end{proof}

\begin{proof}[Proof of Theorem \ref{error analysis}]
As in the proof of Theorem \ref{Catboost theorem}, one can reduce the problem to the case of a single regressor tree. Fixing a tree $T$ from the ensemble and an explicand $\mathbf{x}$, we shall adapt the conventions in Definition-Notation \ref{groundwork}, and we suppress the dependence on $T$: The depth of $T$ is $m:=m(T)$ and $k:=k(T)$ distinct features $X'_1,\dots,X'_k$ appearing in it amount to a partition $\mathsf{p}=\{S_q\mid q\in K\}$ of $M=\{1,\dots,m\}$. 
Fixing $i\in K$, 
and assuming that $\mathbf{x}$ belongs to the interior of the region encoded by $\mathbf{a}\in\mathcal{R}\subseteq\{0,1\}^m$,
the marginal Shapley value for $X'_i$ at $\mathbf{x}$ is given by formula \eqref{auxiliary formula}
derived in the proof of Theorem \ref{Catboost theorem}. 
The algorithm estimates this quantity via replacing 
$p_{\mathbf{u}}$ with $\hat{p}_\mathbf{u}$ in that formula; here $p_{\mathbf{u}}$ is the true probability associated with the leaf encoded by $\mathbf{u}$ while $\hat{p}_\mathbf{u}$ is its estimation based on the training data---it thus should be treated as a random variable. All in all, 
the error term
$\widehat{\varphi_i\big[v^\ME\big]}(\mathbf{x};\mathbf{D})
-\varphi_i\big[v^\ME\big](\mathbf{x})$ 
in estimating the marginal Shapley value for $X'_i$ at $\mathbf{x}$ becomes 
\begin{equation}\label{auxiliary15}
\begin{split}
&\sum_{\substack{Z\subseteq K\\ i\in Z}}   
\sum_{\substack{W\subseteq K\\
W\supseteq Z}}
\,\sum_{\mathbf{b}\in\mathcal{E}^{-1}(\mathbf{a},W;\mathsf{p})\cap\mathcal{R}}
\,\sum_{\mathbf{u}\in\mathcal{E}^{-1}
(\mathbf{b},-Z;\mathsf{p})\cap\mathcal{R}}
\omega^+(|W|,|Z|;k)\cdot c_{\mathbf{b}}\cdot (\hat{p}_\mathbf{u}-p_\mathbf{u})\\
&-\sum_{\substack{W\subseteq K\\ i\notin W}}   
\sum_{\substack{Z\subseteq K\\
Z\subseteq W}}
\,\sum_{\mathbf{b}\in\mathcal{E}^{-1}(\mathbf{a},W;\mathsf{p})\cap\mathcal{R}}
\,\sum_{\mathbf{u}\in\mathcal{E}^{-1}
(\mathbf{b},-Z;\mathsf{p})\cap\mathcal{R}}
\omega^-(|W|,|Z|;k)\cdot c_{\mathbf{b}}\cdot (\hat{p}_\mathbf{u}-p_\mathbf{u}),
\end{split}    
\end{equation}
where $c_{\mathbf{b}}:=c(\mathbf{b};T)$ is the score of the leaf encoded by $\mathbf{b}$.
Dropping $\left|\mathcal{T}^{(i)}\right|$ from \eqref{Lipschitz constant} and adjusting the notation, the goal is to show that the $L^2$-norm of the expression above does not exceed 
\begin{equation}\label{bound1}
\frac{4}{\sqrt{|\mathbf{D}|}}\cdot 
\sqrt[4]{\rm{Gini}}\cdot
\sqrt[4]{\frac{1.5}{k}\left(1+\frac{m}{k}\right)^{k}}
\cdot\sqrt{\sum_{\mathbf{b}\in\mathcal{R}\subseteq\{0,1\}^{m}}c_\mathbf{b}^2},    
\end{equation}
where for the ease of notation we have written ${\rm{Gini}}(\mathbf{X},T)$ as ${\rm{Gini}}$.
Shifting a model by a constant shifts all outputs of the associated marginal game  by the same constant, and hence does not affect the marginal Shapley values  
(cf. \eqref{games}, \eqref{Shapley values}).\footnote{This amounts to 
$$
\sum_{\substack{Z\subseteq K\\ i\in Z}}   
\sum_{\substack{W\subseteq K\\
W\supseteq Z}}
\,\sum_{\{\mathbf{b}\in\mathcal{R}\mid \mathbf{b}\in\mathcal{E}^{-1}(\mathbf{a},W;\mathsf{p}),
\mathbf{u}\in\mathcal{E}^{-1}(\mathbf{b},-Z;\mathsf{p})\}}
\omega^+(|W|,|Z|;k)=
\sum_{\substack{W\subseteq K\\ i\notin W}}   
\sum_{\substack{Z\subseteq K\\
Z\subseteq W}}
\,\sum_{\{\mathbf{b}\in\mathcal{R}\mid \mathbf{b}\in\mathcal{E}^{-1}(\mathbf{a},W;\mathsf{p}),
\mathbf{u}\in\mathcal{E}^{-1}(\mathbf{b},-Z;\mathsf{p})\}}
\omega^-(|W|,|Z|;k)
$$
for any two realizable binary codes $\mathbf{a},\mathbf{u}\in\mathcal{R}$.}
Therefore, one can replace each $c_\mathbf{b}$ with 
$$
c'_\mathbf{b}:=c_\mathbf{b}+\max_{\mathbf{b}'\in\mathcal{R}\subseteq\{0,1\}^{m}}|c_{\mathbf{b}'}|
\geq 0  
$$
to get a tree with non-negative leaf scores without changing \eqref{auxiliary15}. 
Observe that 
$\max_{\mathbf{b}\in\mathcal{R}\subseteq\{0,1\}^{m}}|c'_\mathbf{b}|
\leq 2\cdot\max_{\mathbf{b}\in\mathcal{R}\subseteq\{0,1\}^{m}}|c_\mathbf{b}|$. 
Modifying \eqref{bound1} accordingly, we reduce the problem to showing that when all leaf scores $c_\mathbf{b}$
are non-negative, then the $L^2$-norm of \eqref{auxiliary15} is at most 
\begin{equation}\label{bound2}
\frac{2}{\sqrt{|\mathbf{D}|}}\cdot 
\sqrt[4]{{\rm{Gini}}}\cdot
\sqrt[4]{\frac{1.5}{k}\left(1+\frac{m}{k}\right)^{k}}
\cdot\sqrt{\sum_{\mathbf{b}\in\mathcal{R}\subseteq\{0,1\}^{m}}c_\mathbf{b}^2}.  
\end{equation}
This will follow if we show that 
\begin{equation}\label{auxiliary15'}
\begin{split}
&\Big\lVert\sum_{\substack{Z\subseteq K\\ i\in Z}}   
\sum_{\substack{W\subseteq K\\
W\supseteq Z}}
\,\sum_{\mathbf{b}\in\mathcal{E}^{-1}(\mathbf{a},W;\mathsf{p})\cap\mathcal{R}}
\,\sum_{\mathbf{u}\in\mathcal{E}^{-1}
(\mathbf{b},-Z;\mathsf{p})\cap\mathcal{R}}
\omega^+(|W|,|Z|;k)\cdot c_{\mathbf{b}}\cdot (\hat{p}_\mathbf{u}-p_\mathbf{u})\Big
\rVert_{L^2(\Omega,\mathcal{F},\Bbb{P})},\\
&\Big\lVert\sum_{\substack{W\subseteq K\\ i\notin W}}   
\sum_{\substack{Z\subseteq K\\
Z\subseteq W}}
\,\sum_{\mathbf{b}\in\mathcal{E}^{-1}(\mathbf{a},W;\mathsf{p})\cap\mathcal{R}}
\,\sum_{\mathbf{u}\in\mathcal{E}^{-1}
(\mathbf{b},-Z;\mathsf{p})\cap\mathcal{R}}
\omega^-(|W|,|Z|;k)\cdot c_{\mathbf{b}}\cdot (\hat{p}_\mathbf{u}-p_\mathbf{u})\Big
\rVert_{L^2(\Omega,\mathcal{F},\Bbb{P})}
\end{split}
\end{equation}
\normalsize
are both less than or equal to half the constant appearing in \eqref{bound2}. 
We shall prove this for the $L^2$-norm of the first summation; the second one is similar. 
Working with the norm squared, the goal is to establish
\begin{equation}\label{bound3}
\begin{split}
&\Big\lVert
\sum_{\substack{Z\subseteq K\\ i\in Z}}   
\sum_{\substack{W\subseteq K\\
W\supseteq Z}}
\,\sum_{\mathbf{b}\in\mathcal{E}^{-1}(\mathbf{a},W;\mathsf{p})\cap\mathcal{R}}
\,\sum_{\mathbf{u}\in\mathcal{E}^{-1}
(\mathbf{b},-Z;\mathsf{p})\cap\mathcal{R}}
\omega^+(|W|,|Z|;k)\cdot c_{\mathbf{b}}\cdot (\hat{p}_\mathbf{u}-p_\mathbf{u})\Big
\rVert_{L^2(\Omega,\mathcal{F},\Bbb{P})}^2\\
&\leq 
\frac{1}{|\mathbf{D}|}\cdot 
\sqrt{{\rm{Gini}}}\cdot\sqrt{\frac{1.5}{k}\left(1+\frac{m}{k}\right)^{k}}
\cdot 
\Big(\sum_{\mathbf{b}\in\mathcal{R}\subseteq\{0,1\}^{m}}c_\mathbf{b}^2\Big)
\end{split}
\end{equation}
\normalsize
under the assumption that $c_\mathbf{b}$ is always non-negative.
Expanding the left-hand side of \eqref{bound3} yields 
\begin{equation}\label{auxiliary16}
\sum_{\substack{Z\subseteq K\\ i\in Z}}
\sum_{\substack{W\subseteq K\\
W\supseteq Z}}
\,\sum_{\mathbf{b}\in\mathcal{E}^{-1}(\mathbf{a},W;\mathsf{p})\cap\mathcal{R}}
\,\sum_{\mathbf{u}\in\mathcal{E}^{-1}
(\mathbf{b},-Z;\mathsf{p})\cap\mathcal{R}}
\omega^+(|W|,|Z|;k)^2\cdot c_{\mathbf{b}}^2\cdot 
\parallel\hat{p}_\mathbf{u}-p_\mathbf{u}\parallel^2_{L^2(\Omega,\mathcal{F},\Bbb{P})}
\end{equation}
plus pairwise $L^2$-inner products of random variables 
$$
\omega^+(|W|,|Z|;k)\cdot c_{\mathbf{b}}\cdot (\hat{p}_\mathbf{u}-p_\mathbf{u}).
$$
But these inner products are all non-positive because 
$\omega^+(|W|,|Z|;k)>0$, 
$c_\mathbf{b}\geq 0$ (given our assumption),
and by Lemma \ref{error}:
$$
\left\langle \hat{p}_\mathbf{u}-p_\mathbf{u},
\hat{p}_{\mathbf{u}'}-p_{\mathbf{u}'}\right\rangle_{L^2(\Omega,\mathcal{F},\Bbb{P})}\leq 0\quad (\mathbf{u}\neq\mathbf{u}').
$$
Consequently, the left-hand side of \eqref{bound3} is less than or equal to \eqref{auxiliary16}; and showing that the latter is not greater than 
\begin{equation}\label{bound4}
\frac{1}{|\mathbf{D}|}\cdot 
\sqrt{{\rm{Gini}}}\cdot\sqrt{\frac{1.5}{k}\left(1+\frac{m}{k}\right)^{k}}
\cdot 
\Big(\sum_{\mathbf{b}\in\mathcal{R}\subseteq\{0,1\}^{m}}c_\mathbf{b}^2\Big)
\end{equation}
concludes the proof of \eqref{bound3}. 
We have $\parallel\hat{p}_\mathbf{u}-p_\mathbf{u}\parallel^2_{L^2(\Omega,\mathcal{F},\Bbb{P})}
\leq\frac{p_\mathbf{u}(1-p_\mathbf{u})}{|\mathbf{D}|}$
from Lemma \ref{error}, and 
\begin{equation}\label{bound4'}
\omega^+(|W|,|Z|;k)\leq
\frac{1}{k+|Z|-|W|}
\leq\frac{1}{|Z|}.    
\end{equation}
Substituting in \eqref{auxiliary16}: 
\small
\begin{equation}\label{auxiliary16'}
\begin{split}
&\sum_{\substack{Z\subseteq K\\ i\in Z}}
\sum_{\substack{W\subseteq K\\
W\supseteq Z}}
\,\sum_{\mathbf{b}\in\mathcal{E}^{-1}(\mathbf{a},W;\mathsf{p})\cap\mathcal{R}}
\,\sum_{\mathbf{u}\in\mathcal{E}^{-1}
(\mathbf{b},-Z;\mathsf{p})\cap\mathcal{R}}
\omega^+(|W|,|Z|;k)^2\cdot c_{\mathbf{b}}^2\cdot 
\parallel\hat{p}_\mathbf{u}-p_\mathbf{u}\parallel^2_{L^2(\Omega,\mathcal{F},\Bbb{P})}\\  
&\leq \frac{1}{|\mathbf{D}|}\cdot
\max_{\mathbf{b}\in\mathcal{R}}
\Big(\sum_{\substack{Z\subseteq K\\ i\in Z}}
\sum_{\mathbf{u}\in\mathcal{E}^{-1}
(\mathbf{b},-Z;\mathsf{p})\cap\mathcal{R}}
\frac{p_\mathbf{u}(1-p_\mathbf{u})}{|Z|^2}\Big)
\cdot 
\Big(\sum_{W\subseteq K}
\,\sum_{\mathbf{b}\in\mathcal{E}^{-1}(\mathbf{a},W;\mathsf{p})\cap\mathcal{R}}c_\mathbf{b}^2\Big)\\
&\leq 
\frac{1}{|\mathbf{D}|}\cdot
\max_{\mathbf{b}\in\mathcal{R}}
\Bigg(\Big(\sum_{\substack{Z\subseteq K\\ i\in Z}}
\sum_{\mathbf{u}\in\mathcal{E}^{-1}
(\mathbf{b},-Z;\mathsf{p})\cap\mathcal{R}}p_\mathbf{u}^2(1-p_\mathbf{u})^2\Big)^{\frac{1}{2}}\cdot
\Big(\sum_{\substack{Z\subseteq K\\ i\in Z}}
\sum_{\mathbf{u}\in\mathcal{E}^{-1}
(\mathbf{b},-Z;\mathsf{p})\cap\mathcal{R}}\frac{1}{|Z|^4}\Big)^{\frac{1}{2}}\Bigg)
\cdot 
\Big(\sum_{\mathbf{b}\in\mathcal{R}\subseteq\{0,1\}^{m}}c_\mathbf{b}^2\Big)\\
&\leq \frac{1}{|\mathbf{D}|}\cdot
\max_{\mathbf{b}\in\mathcal{R}}
\Bigg(\Big(\sum_{\mathbf{u}\in\mathcal{R}}p_\mathbf{u}^2\Big)^{\frac{1}{2}}\cdot
\Big(\sum_{\substack{Z\subseteq K\\ i\in Z}}
\sum_{\mathbf{u}\in\mathcal{E}^{-1}
(\mathbf{b},-Z;\mathsf{p})\cap\mathcal{R}}\frac{1}{|Z|^4}\Big)
^{\frac{1}{2}}\Bigg)
\cdot 
\Big(\sum_{\mathbf{b}\in\mathcal{R}\subseteq\{0,1\}^{m}}c_\mathbf{b}^2\Big)\\
&= \frac{1}{|\mathbf{D}|}\cdot \sqrt{{\rm{Gini}}}\cdot
\Big(
\sum_{\substack{Z\subseteq K\\ i\in Z}}
\frac{1}{|Z|^4}\cdot\prod_{q\in Z}|S_q|\Big)\
^{\frac{1}{2}}
\cdot \Big(\sum_{\mathbf{b}\in\mathcal{R}\subseteq\{0,1\}^{m}}c_\mathbf{b}^2\Big),
\end{split}
\end{equation}
\normalsize
where we have used the fact that for different subsets $W_1\neq W_2$ or $Z_1\neq Z_2$, the subsets 
$\mathcal{E}^{-1}(\mathbf{a},W_1;\mathsf{p})$ and $\mathcal{E}^{-1}(\mathbf{a},W_2;\mathsf{p})$
of binary codes are disjoint as well as subsets 
$\mathcal{E}^{-1}(\mathbf{b},-Z_1;\mathsf{p})$ and $\mathcal{E}^{-1}(\mathbf{b},-Z_2;\mathsf{p})$.
Furthermore, the Cauchy-Schwarz inequality was invoked for the third line. 
Finally, for the last line, we relied on the fact that the number of choices for 
$\mathbf{u}\in\mathcal{E}^{-1}(\mathbf{b},-Z;\mathsf{p})$
is $\prod_{q\in Z}|S_q|$ since the segment $\mathbf{u}_{S_q}$ is known when $q\in K\setminus Z$ (as one should have $\mathbf{u}_{S_q}=\mathbf{b}_{S_q}$) 
while there are $(|S_q|+1)-1=|S_q|$ choices for it when 
$q\in Z$ (as the only requirement is $\mathbf{u}_{S_q}\neq\mathbf{b}_{S_q}$). 
In the case of the second summation in \eqref{auxiliary15'}, one should similarly bound 
\begin{equation}\label{auxiliary16''}
\sum_{\substack{W\subseteq K\\ i\notin W}}   
\sum_{\substack{Z\subseteq K\\
Z\subseteq W}}
\,\sum_{\mathbf{b}\in\mathcal{E}^{-1}(\mathbf{a},W;\mathsf{p})\cap\mathcal{R}}
\,\sum_{\mathbf{u}\in\mathcal{E}^{-1}
(\mathbf{b},-Z;\mathsf{p})\cap\mathcal{R}}
\omega^-(|W|,|Z|;k)^2\cdot c_{\mathbf{b}}^2\cdot 
\parallel\hat{p}_\mathbf{u}-p_\mathbf{u}\parallel^2_{L^2(\Omega,\mathcal{F},\Bbb{P})}.
\end{equation}
When $W$ is a proper subset of $K$, \eqref{bound4'} can be sharpened to
$\omega^-(|W|,|Z|;k)\leq\frac{1}{k+|Z|-|W|}\leq\frac{1}{|Z|+1}$. 
Adjusting the inequalities accordingly, the right-hand side of \eqref{auxiliary16'}
becomes 
$$
\frac{1}{|\mathbf{D}|}\cdot \sqrt{{\rm{Gini}}}\cdot
\Big(
\sum_{\substack{Z\subseteq K\\ i\notin Z}}
\frac{1}{(|Z|+1)^4}\cdot\prod_{q\in Z}|S_q|\Big)\
^{\frac{1}{2}}
\cdot \Big(\sum_{\mathbf{b}\in\mathcal{R}\subseteq\{0,1\}^{m}}c_\mathbf{b}^2\Big).
$$
In both situations 
$$
\Big(
\sum_{\substack{Z\subseteq K\\ i\in Z}}
\frac{1}{|Z|^4}\cdot\prod_{q\in Z}|S_q|\Big)\
^{\frac{1}{2}},
\Big(
\sum_{\substack{Z\subseteq K\\ i\notin Z}}
\frac{1}{(|Z|+1)^4}\cdot\prod_{q\in Z}|S_q|\Big)\
^{\frac{1}{2}}
\leq 
\sqrt{\frac{1.5}{k}\left(1+\frac{m}{k}\right)^{k}}
$$
by Lemma \ref{auxiliary lemma 1}. Hence we arrive at the desired bound \eqref{bound4} for 
the first line of \eqref{auxiliary16'} and also for \eqref{auxiliary16''}.\\
\indent 
Thus far, we have established the $L^2$-error estimation \eqref{inequality}. 
It remains to show that the constant $C$ from \eqref{Lipschitz constant} appearing there is less than or equal to
$4\,|\mathcal{T}|\cdot \sqrt[4]{\frac{3\mathcal{L}}{\log_2(\mathcal{L})}}$. 
It suffices to show 
$$
\max_{T\in\mathcal{T}}\frac{1.5}{k(T)}\left(1+\frac{m(T)}{k(T)}\right)^{k(T)}
\leq \frac{3\mathcal{L}}{\log_2(\mathcal{L})}.
$$
This can be obtained from Lemma \ref{auxiliary lemma 2} and Lemma \ref{auxiliary lemma 3} below:
One has 
$\frac{1}{k(T)}\left(1+\frac{m(T)}{k(T)}\right)^{k(T)}\leq 2\cdot\frac{2^{m(T)}}{m(T)}$; and 
$\max_{T\in\mathcal{T}}\frac{2^{m(T)}}{m(T)}\leq
\frac{\mathcal{L}}{\log_2(\mathcal{L})}$
(keep in mind that $\mathcal{L}=\max_{T\in\mathcal{T}}2^{m(T)}$).
\end{proof}

\begin{remark}\label{limited features 2}
When $m(T)\gg k(T)$ across the ensemble, we expect a faster convergence in Theorem \ref{error analysis}: 
when $k(T)\leq k_*$ as $m(T)\to\infty$, 
$\sqrt[4]{\frac{1.5}{k(T)}\left(1+\frac{m(T)}{k(T)}\right)^{k(T)}}=
O\left(\sqrt[4]{\frac{\mathcal{L}}{\log_2(\mathcal{L})}}\right)$
can be improved to 
$\sqrt[4]{\frac{1.5}{k(T)}\left(1+\frac{m(T)}{k(T)}\right)^{k(T)}}=
O\left(\log_2(\mathcal{L})^{\frac{k_*}{4}}\right)$.
Compare with Remark \ref{limited features 1}. 
\end{remark}

We finish the section by the lemmas used in the previous proof.
\begin{lemma}\label{error}
Let $\mathbf{D}:=\left\{\mathcal{X}^{(1)},\dots,\mathcal{X}^{(\mathscr{D})}\right\}$
$(\mathscr{D}:=|\mathbf{D}|)$
be a random dataset of i.i.d. random vectors with values in $\Bbb{R}^n$ and distributed  
 according to  ${\rm{P}}_{\mathbf{X}}$. 
For any arbitrary Borel subset $R$ of $\Bbb{R}^n$ define the random variable 
$$\hat{{\rm{P}}}_{\mathbf{X}}(R):=\frac{|\mathbf{D}\cap R|}{|\mathbf{D}|}=\frac{1}{|\mathbf{D}|}
\sum_{i=1}^{|\mathbf{D}|}\mathbbm{1}_{R}\big(\mathcal{X}^{(i)}\big).
$$
Then one always has 
\begin{equation}\label{covest}
{\rm{Cov}}(\hat{{\rm{P}}}_{\mathbf{X}}(R),\hat{{\rm{P}}}_{\mathbf{X}}(S))=
\Bbb{E}\left[
\big(\hat{{\rm{P}}}_{\mathbf{X}}(R)-{\rm{P}}_{\mathbf{X}}(R)\big)
\big(\hat{{\rm{P}}}_{\mathbf{X}}(S)-{\rm{P}}_{\mathbf{X}}(S)\big)
\right]
=\frac{{\rm{P}}_{\mathbf{X}}(R\cap S)-{\rm{P}}_{\mathbf{X}}(R)\cdot{\rm{P}}_{\mathbf{X}}(S)}{|\mathbf{D}|}.
\end{equation}
In particular:
$$
{\rm{Var}}(\hat{{\rm{P}}}_{\mathbf{X}}(R))=
\Bbb{E}\left[\big|
\hat{{\rm{P}}}_{\mathbf{X}}(R)-{\rm{P}}_{\mathbf{X}}(R)
\big|^2\right]
=\frac{{\rm{P}}_{\mathbf{X}}(R)(1-{\rm{P}}_{\mathbf{X}}(R))}{|\mathbf{D}|},
$$
and 
$$
\Bbb{E}\left[
\big(\hat{{\rm{P}}}_{\mathbf{X}}(R)-{\rm{P}}_{\mathbf{X}}(R)\big)
\big(\hat{{\rm{P}}}_{\mathbf{X}}(S)-{\rm{P}}_{\mathbf{X}}(S)\big)
\right]\leq 0
$$
if $R\cap S$ is of measure zero.
\end{lemma}
\begin{proof} 
 Since $\hat{{\rm{P}}}_{\mathbf{X}}(R)$ and $\hat{{\rm{P}}}_{\mathbf{X}}(S)$ are unbiased estimators, $\Bbb{E}\left[\hat{{\rm{P}}}_{\mathbf{X}}(R)\right]={\rm{P}}_{\mathbf{X}}(R)$ and $\Bbb{E}\left[\hat{{\rm{P}}}_{\mathbf{X}}(S)\right]={\rm{P}}_{\mathbf{X}}(S)$. Hence $${\rm{Cov}}(\hat{{\rm{P}}}_{\mathbf{X}}(R),\hat{{\rm{P}}}_{\mathbf{X}}(S))
 =\Bbb{E}\left[\big(\hat{{\rm{P}}}_{\mathbf{X}}(R)-{\rm{P}}_{\mathbf{X}}(R)\big)
\big(\hat{{\rm{P}}}_{\mathbf{X}}(S)-{\rm{P}}_{\mathbf{X}}(S)\big)\right]=
\Bbb{E}\left[\hat{{\rm{P}}}_{\mathbf{X}}(R) \cdot \hat{{\rm{P}}}_{\mathbf{X}}(S)\right]
-\Bbb{E}\left[\hat{{\rm{P}}}_{\mathbf{X}}(R)\right] \cdot \Bbb{E}\left[\hat{{\rm{P}}}_{\mathbf{X}}(S)\right].$$  Furthermore, since $\mathcal{X}^{(i)}$, $i\in \{1,\dots,|\mathbf{D}|\}$ are i.i.d.:
 \[
 \Bbb{E}\left[\hat{{\rm{P}}}_{\mathbf{X}}(R) \cdot \hat{{\rm{P}}}_{\mathbf{X}}(S)\right] 
 = \frac{1}{|\mathbf{D}|^2}\sum_{i,j} \Bbb{E}\left[\mathbbm{1}_{R}\big(\mathcal{X}^{(j)}\big)\cdot \mathbbm{1}_{S}\big(\mathcal{X}^{(i)}\big)\right]=\frac{1}{|\mathbf{D}|}\,{\rm{P}}_{\mathbf{X}}(R \cap S)+ \left(1-\frac{1}{|\mathbf{D}|}\right)({\rm{P}}_{\mathbf{X}}(R) \cdot {\rm{P}}_{\mathbf{X}}(S)).
 \]
 Combining with the previous identity gives \eqref{covest}. The rest follows directly from \eqref{covest}.
\end{proof}

\begin{lemma}\label{auxiliary lemma 1}
Consider a set $\{s_q\mid q\in K\}$ of numbers greater than or equal to $1$ which are indexed by elements of the set 
$K=\{1,\dots,k\}$. For any $i\in K$ and $W\subseteq K\setminus\{i\}$ one has 
\begin{equation}
\sum_{Z\subseteq W}\frac{1}{(|Z|+1)^4}\cdot\prod_{q\in Z\cup\{i\}}s_q
\leq\frac{1.5}{|W|+1}\cdot\prod_{q\in W\cup\{i\}}(s_q+1)
\leq\frac{1.5}{k}\left(1+\frac{\sum_{q\in K}s_q}{k}\right)^k.
\end{equation}
In particular:
$$
\sum_{Z\subseteq K\setminus\{i\}}\frac{1}{(|Z|+1)^4}\cdot\prod_{q\in Z}s_q
\leq 
\sum_{Z\subseteq K\setminus\{i\}}\frac{1}{(|Z|+1)^4}\cdot\prod_{q\in Z\cup\{i\}}s_q
\leq
\frac{1.5}{k}\left(1+\frac{\sum_{q\in K}s_q}{k}\right)^k.
$$
\end{lemma}
\begin{proof}
First, notice that if one increases the size of $W$ by adding a new element $w_*$:
$$
\frac{\frac{1.5}{|W|+2}\cdot\prod_{q\in W\cup\{w_*\}\cup\{i\}}(s_q+1)}
{\frac{1.5}{|W|+1}\cdot\prod_{q\in W\cup\{i\}}(s_q+1)}
=\frac{|W|+1}{|W|+2}\cdot(s_{w_*+1})\geq 2\,\frac{|W|+1}{|W|+2}\geq 1.
$$
Hence changing $W$ to $K\setminus\{i\}$ does not decrease 
$\frac{1}{|W|+1}\cdot\prod_{q\in W\cup\{i\}}(s_q+1)$:
$$
\frac{1.5}{|W|+1}\cdot\prod_{q\in W\cup\{i\}}(s_q+1)
\leq\frac{1.5}{k}\cdot\prod_{q\in K}(s_q+1)
\leq\frac{1.5}{k}\left(1+\frac{\sum_{q\in K}s_q}{k}\right)^k,
$$
where the AM-GM inequality was used for the last step.
It remains to show that 
$$
\sum_{Z\subseteq W}\frac{1}{(|Z|+1)^4}\cdot\prod_{q\in Z\cup\{i\}}s_q
\leq\frac{1.5}{|W|+1}\cdot\prod_{q\in W\cup\{i\}}(s_q+1)
$$
for any $W\subseteq K\setminus\{i\}$. We establish this via induction on $|W|$. 
For the base of the induction, the inequality can be checked directly: 
It trivially holds if $W=\varnothing$; and when $|W|=1$, say $W=\{w_0\}$, it amounts to
$$
s_i+\frac{s_is_{w_0}}{16}\leq\frac{3}{4}(s_i+1)(s_{w_0}+1)\Leftrightarrow
4s_i\leq 11s_is_{w_0}+12s_{w_0}+12;
$$
which holds since $s_i,s_{w_0}\geq 1$. Now suppose $|W|\geq 1$.
Adding an element $w_*\in K\setminus(W\cup\{i\})$ to $W$, one should obtain an upper bound for 
\begin{equation*}
\sum_{Z\subseteq W\cup\{w_*\}}\frac{1}{(|Z|+1)^4}\cdot\prod_{q\in Z\cup\{i\}}s_q
=\sum_{Z\subseteq W}\frac{1}{(|Z|+1)^4}\cdot\prod_{q\in Z\cup\{i\}}s_q
+s_{w_*}\Big(\sum_{Z\subseteq W}\frac{1}{(|Z|+2)^4}\cdot\prod_{q\in Z\cup\{i\}}s_q\Big).
\end{equation*}
In the second term on the right, one has 
$$
\frac{1}{(|Z|+2)^4}\leq\frac{(|W|+1)^4}{(|W|+2)^4}\cdot\frac{1}{(|Z|+1)^4};
$$
and moreover, by the induction hypothesis: 
$$
\sum_{Z\subseteq W}\frac{1}{(|Z|+1)^4}\cdot\prod_{q\in Z\cup\{i\}}s_q
\leq\frac{1.5}{|W|+1}\cdot\prod_{q\in W\cup\{i\}}(s_q+1).
$$
Consequently: 
\begin{equation}
\sum_{Z\subseteq W\cup\{w_*\}}\frac{1}{(|Z|+1)^4}\cdot\prod_{q\in Z\cup\{i\}}s_q
\leq\Big(1+\frac{(|W|+1)^4}{(|W|+2)^4}\cdot s_{w_*}\Big)\Big(\frac{1.5}{|W|+1}
\cdot\prod_{q\in W\cup\{i\}}(s_q+1)\Big);
\end{equation}
and it remains to show that 
$$
\frac{1.5}{|W|+1}\Big(1+\frac{(|W|+1)^4}{(|W|+2)^4}\cdot s_{w_*}\Big)
\leq \frac{1.5}{|W|+2}(1+s_{w_*}),
$$
or equivalently
$$
s_{w_*}\Big(\frac{1}{|W|+2}-\frac{(|W|+1)^3}{(|W|+2)^4}\Big)\geq
\frac{1}{|W|+1}-\frac{1}{|W|+2}.
$$
Since $s_{w_*}\geq 1$, it suffices to show that 
$\frac{1}{|W|+2}-\frac{(|W|+1)^3}{(|W|+2)^4}
\geq\frac{1}{|W|+1}-\frac{1}{|W|+2}$.
This always holds; indeed, denoting $|W|+1\geq 2$ by $x$, one has 
$$
\frac{2}{x+1}>\frac{1}{x}+\frac{x^3}{(x+1)^4}\quad 
x\in [2,\infty)
$$
since 
$$
x(x+1)^4\Big(\frac{2}{x+1}-\frac{1}{x}-\frac{x^3}{(x+1)^4}\Big)
=2x^3-2x-1>0
$$
whenever $x\geq 2$.
\end{proof}

\begin{lemma}\label{auxiliary lemma 2}
For any two positive integers $k\leq m$ one has 
$\frac{1}{k}\left(1+\frac{m}{k}\right)^k\leq 2\cdot\frac{2^m}{m}$.
\end{lemma}

\begin{proof}
The function $x\mapsto\left(1+\frac{m}{x}\right)^x$ is increasing on $[1,\infty)$. Thus 
$\left(1+\frac{m}{k}\right)^k\leq 2^m$. So the desired inequality holds for 
$k\geq\frac{m}{2}$ since in that case 
$\frac{1}{k}\leq\frac{2}{m}$.
Next, we focus on the case of $k\leq\frac{m}{2}$. 
If $k=1$, one has 
$$
\frac{2^{m+1}}{m}-(1+m)\geq 
\frac{1+\binom{m+1}{1}+\binom{m+1}{2}+\binom{m+1}{3}}{m}-(m+1)
=\frac{1}{6}(m^3-3m^2+2m+12)>0\quad (m\in\Bbb{N}).
$$
So let us assume that $2\leq k\leq\frac{m}{2}$. Then:
$$
\frac{1}{k}\left(1+\frac{m}{k}\right)^k\leq 
\frac{1}{2}\left(1+\frac{m}{k}\right)^k\leq 
\frac{1}{2}\left(1+\frac{m}{\frac{m}{2}}\right)^{\frac{m}{2}}=\frac{(\sqrt{3})^m}{2}.
$$
It suffices to show that $\frac{(\sqrt{3})^m}{2}\leq 2\cdot\frac{2^m}{m}$ for all positive integers $m$. The inequality can be checked directly for $m\leq 7$. When $m\geq 7$, increasing $m$ by $1$, the left-hand side is multiplied by $\sqrt{3}$ while the right-hand side is multiplied by $\frac{2m}{m+1}$; and 
$\frac{2m}{m+1}>\sqrt{3}$ for $m\geq 7$.
\end{proof}

\begin{lemma}\label{auxiliary lemma 3}
For any two positive integers $m_1\leq m_2$ one has 
$\frac{2^{m_1}}{m_1}\leq \frac{2^{m_2}}{m_2}$.
\end{lemma}

\begin{proof}
The ratio
$$
\frac{\frac{2^{m+1}}{m+1}}{\frac{2^m}{m}}=\frac{2m}{m+1}
$$
is never smaller than $1$. Hence the sequence $\left\{\frac{2^m}{m}\right\}_{m\in\Bbb{N}}$
is increasing. 
\end{proof}

\section{Game values suitable for explaining tree ensembles}\label{appendix:general game values}

\subsection{Proof of Theorem \ref{classification}}\label{subappendix:classification}
In \Sec \ref{subsec:gamebackground}, we set forth four properties that we believe are desirable for game values used for generating marginal feature attributions of  tree ensembles.
Here, we prove Theorem \ref{classification} which classifies such game values. 
Notice that here we deal with general cooperative games $(N,v)$ where $N$ is a finite subset of $\Bbb{N}$.
\begin{proof}[Proof of Theorem \ref{classification}]
A game value of the form \eqref{good game value} satisfies linearity, symmetry and null-player axioms due to Lemma \ref{properties}. As for the carrier dependence, the identity 
$\sum_{S\subseteq N\setminus\{i\},S\cap U=S'}w(S;N,i)=w(S';U,i)$
from the first part of that lemma turns into 
\begin{equation}\label{auxiliary17}
\sum_{s=s'}^{n-u+s'}\binom{n-u}{s-s'}\alpha(s,n)=\alpha(s',u)    
\end{equation}
where $s:=|S|$, $u:=|U|<n$ and $s':=|S'|< u$.
When $u=n-1$, this amounts to $\alpha(s',n)+\alpha(s'+1,n)=\alpha(s',n-1)$, i.e. the backward Pascal identity \eqref{backward Pascal}. The identity can furthermore be employed to establish \eqref{auxiliary17}
via induction on $n-u$:
\begin{equation*}
\begin{split}
&\sum_{s=s'}^{n-u+s'}\binom{n-u}{s-s'}\alpha(s,n)=
\sum_{s=s'}^{n-u+s'}\left[\binom{n-u-1}{s-s'}+\binom{n-u-1}{s-s'-1}\right]\alpha(s,n)\\
&=\sum_{s=s'}^{n-(u+1)+s'}\binom{n-(u+1)}{s-s'}\alpha(s,n)
+\sum_{s=s'+1}^{n-(u+1)+(s'+1)}\binom{n-(u+1)}{s-(s'+1)}\alpha(s,n)\\
&=\alpha(s',u+1)+\alpha(s'+1,u+1)=\alpha(s',u)
\end{split}    
\end{equation*}
where \eqref{backward Pascal} is applied for the last equality and the induction hypothesis for the one before that.\\
\indent Next, we establish the converse implication. Invoking Lemma \ref{properties}, a game value linear game value $h$ with symmetry and null-player properties may be written as 
\begin{equation}\label{auxiliary18}
h_i[N,v]:=\sum_{S\subseteq N\setminus\{i\}}\alpha(|S|;N)\left(v(S\cup\{i\})-v(S)\right)    
\end{equation}
where $N$ is an arbitrary finite non-empty subset of $\Bbb{N}$, $v:2^N\rightarrow\Bbb{R}$  is a cooperative game, $i\in N$ and 
$$
\left\{\alpha(\cdot;N):\{0,\dots,|N|-1\}\rightarrow\Bbb{R}\right\}_{N\subset\Bbb{N},\, 0<|N|<\infty}
$$
is a family of functions encoded by non-empty, finite subsets of natural numbers. 
According to the lemma again, the carrier-dependence condition for \eqref{auxiliary18} can be rephrased as  
\begin{equation}\label{auxiliary17'}
\sum_{s=|S'|}^{|N|-|U|+|S'|}\binom{|N|-|U|}{s-|S'|}\alpha(s;N)=\alpha(|S'|;U)    
\end{equation}
whenever $S'\subsetneq U\subsetneq N$. Any finite subset $U$ of $\Bbb{N}$ is properly contained in a segment 
$\{1,\dots,n\}$ for a sufficiently large $n$. Setting $s':=|S'|$ and taking $N$ to be $\{1,\dots,n\}$ in \eqref{auxiliary17'}, we observe that 
\begin{equation}\label{auxiliary17''}
\alpha(s';U)=\sum_{s=s'}^{n-|U|+s'}\binom{n-|U|}{s-s'}\alpha(s;\{1,\dots,n\}) 
\quad\forall U\subset\{1,\dots,n\}, s'\in\{0,\dots,|U|-1\}.   
\end{equation}
Setting $U=\{1,\dots,n-1\}$ yields a backward Pascal identity 
\begin{equation}\label{backward Pascal'}
\alpha(s';\{1,\dots,n-1\})=\alpha(s';\{1,\dots,n\})+\alpha(s'+1;\{1,\dots,n\})\quad \forall s'\in\{0,\dots,n-1\}.    
\end{equation}
Therefore, the inductive argument used to derive \eqref{auxiliary17} can be repeated to simplify the right-hand side of \eqref{auxiliary17''} to $\alpha(s';\{1,\dots,u\})$. We deduce that any $\alpha(s';U)$ depends only on $s'$ and $|U|$. Abusing the notation to denote $\alpha(s;\{1,\dots,n\})$ by $\alpha(s,n)$, we obtain a collection $\mathcal{A}:=\{\alpha(s,n)\}_{\substack{n\in\Bbb{N}\\ 0\leq s<n}}$ for which,
because of \eqref{backward Pascal'}, the backward Pascal identity  holds; and $\alpha(|S|;N)=\alpha(|S|,|N|)$ for any two subsets $S\subsetneq N\subset\Bbb{N}$. Substituting in \eqref{auxiliary18} yields $h=h^{\mathcal{A}}$.
\end{proof}

The generalization below of Lemma \ref{hockey-stick}  was established in the course of proof. 
We record it for future references. 
\begin{lemma}\label{simplification}
If a game value 
$$
h_i[N,v]:=\sum_{S\subseteq N\setminus \{i\}}w(S;N,i)\left(v(S\cup\{i\})-v(S)\right)\quad (i\in N)   
$$
is equal to $h^\mathcal{A}$ for a collection 
$\mathcal{A}=\{\alpha(s,n)\}_{\substack{n\in\Bbb{N}\\ 0\leq s<n}}$ of real numbers as in Theorem \ref{classification}, 
then 
$$
\sum_{Z\subseteq S\subseteq W}w(S;N,i)=\alpha(|Z|,|N|+|Z|-|W|).
$$
whenever $Z\subseteq W\subseteq N\setminus\{i\}$.
\end{lemma}
\begin{proof}
One needs to show $\sum_{Z\subseteq S\subseteq W}\alpha(|S|,n)=\alpha(|Z|,n+|Z|-|W|)$ where $n=|N|$ as usual.
The left-hand side can be written as $\sum_{s=z}^{w}\binom{w-z}{s-z}\alpha(s,n)$ where $z:=|Z|$ and $w:=|W|$; and we need to show that this is the same as $\alpha(z,n+z-w)$. This was essentially established in the proof of Theorem \ref{classification} where \eqref{auxiliary17} was derived from 
\eqref{backward Pascal}.
\end{proof}
\subsection{Coalitional game values for tree ensembles}\label{subappendix:classification coalitional}
We next introduce a family  of coalitional game values which mimic the definition of the Owen value 
\eqref{OwenFormula} and have desirable properties similar to those outlined in Theorem \ref{classification}. 
\begin{proposition}\label{simplification coalitional}
Let $\mathcal{A}_1:=\{\alpha_1(s,n)\}_{\substack{n\in\Bbb{N}\\ 0\leq s<n}}$ and 
$\mathcal{A}_2:=\{\alpha_2(s,n)\}_{\substack{n\in\Bbb{N}\\ 0\leq s<n}}$ be two collections of real numbers each satisfying the backward Pascal identity (cf. \eqref{backward Pascal}). Consider the following coalitional game value
\begin{equation}\label{coalitional generalization}
\mathfrak{h}^{\mathcal{A}_1,\mathcal{A}_2}_i[N,v,\mathfrak{P}]:=  
\sum_{R\subseteq M\setminus \{j\}}
\sum_{K\subseteq S_j\setminus\{i\}}
\alpha_1(|R|,|M|)\cdot
\alpha_2(|K|,|S_j|)
\left(v\left(Q\cup K\cup\{i\}\right)-
v\left(Q\cup K\right)\right)
\end{equation}
where $(N,v)$ is a cooperative game, $\mathfrak{P}:=\{S_1,\dots,S_m\}$ is a partition of $N$ indexed by elements of $M:=\{1,\dots,m\}$, $i\in N$ lies in $S_j$, and $Q:=\cup_{r\in R}S_r$.
\begin{enumerate}
\item The coalitional game value $\mathfrak{h}^{\mathcal{A}_1,\mathcal{A}_2}$ 
satisfies linearity, coalitional symmetry, null-player and coalitional carrier-dependence properties. 
\item If a coalitional game value 
\begin{equation}\label{linear general variant coalitional'}
\mathfrak{h}_i[N,v,\mathfrak{P}]:=\sum_{S\subseteq N\setminus \{i\}}w(S;N,i,\mathfrak{P})\left(v(S\cup\{i\})-v(S)\right)\quad (i\in N)      
\end{equation}
is equal to $\mathfrak{h}^{\mathcal{A}_1,\mathcal{A}_2}$ defined in \eqref{coalitional generalization}, then,  for any $Z\subseteq W\subseteq N\setminus\{i\}$, one has 
\begin{equation}\label{auxiliary19}
\sum_{Z\subseteq S\subseteq W}w(S;N,i,\mathfrak{P})
=\begin{cases}
\alpha_1(|\mathcal{Z}|,|M|+|\mathcal{Z}|-|\mathcal{W}|)\cdot\alpha_2(|Z\cap S_j|,|S_j|+|Z\cap S_j|-|W\cap S_j|) & \text{if }\mathcal{Z}\subseteq\mathcal{W}, \\
0 &\text{otherwise.}
\end{cases}
\end{equation}
where 
\begin{equation}\label{auxiliary19'}
\mathcal{Z}:=\{r\in M\setminus\{j\}\mid Z\cap S_r\neq\varnothing\},\quad
\mathcal{W}:=\{r\in M\setminus\{j\}\mid S_r\subseteq W\}.
\end{equation}
\item If $\alpha_2(s,n)=\frac{s!(n-s-1)!}{n!}$, then $\mathfrak{h}^{\mathcal{A}_1,\mathcal{A}_2}$ admits the quotient game property \eqref{quotinet game}.
\end{enumerate}
\end{proposition}
\begin{proof}
Any coalitional game value of the form \eqref{linear general variant coalitional'}, including 
$\mathfrak{h}^{\mathcal{A}_1,\mathcal{A}_2}$, is obviously linear and satisfies the null-player property. 
Also given that its coefficients depend only on the cardinalities of subsets,  
it is not hard to see that $\mathfrak{h}^{\mathcal{A}_1,\mathcal{A}_2}$ 
has the coalitional symmetry property.\\
\indent 
We establish the coalitional carrier-dependence and the claim in (2) simultaneously. In \eqref{auxiliary19}, $w(S;N,i,\mathfrak{P})$ is non-zero only if $S$ is of the form 
$(\cup_{r\in R}S_r)\cup K$ where $R\subseteq M\setminus\{j\}$ and $K\subseteq S_j\setminus\{i\}$.  
A subset of this form fits between $Z$ and $W$ if and only if for subsets  $\mathcal{Z},\mathcal{W}$ from \eqref{auxiliary19'} one has $\mathcal{Z}\subseteq\mathcal{W}$.  
Thus the summation in \eqref{auxiliary19} vanishes when $\mathcal{Z}\nsubseteq\mathcal{W}$. 
Assuming that $\mathcal{Z}\subseteq\mathcal{W}$, 
subsets $S=(\cup_{r\in R}S_r)\cup K$ for which 
$Z\subseteq S\subseteq W$ are precisely those with 
$\mathcal{Z}\subseteq R\subseteq\mathcal{W}$ and 
$Z\cap S_j\subseteq K\subseteq W\cap S_j$. Hence:
\begin{equation}\label{auxiliary20'}
\sum_{Z\subseteq S\subseteq W}w(S;N,i,\mathfrak{P})=
\Big(\sum_{\mathcal{Z}\subseteq R\subseteq\mathcal{W}}\alpha_1(|R|,|M|)\Big)\cdot
\Big(\sum_{Z\cap S_j\subseteq K\subseteq W\cap S_j}\alpha_2(|K|,|S_j|)\Big).    
\end{equation}
The right-hand side can be simplified to 
$\alpha_1(|\mathcal{Z}|,|M|+|\mathcal{Z}|-|\mathcal{W}|)\cdot\alpha_2(|Z\cap S_j|,|S_j|+|Z\cap S_j|-|W\cap S_j|)$
by applying Lemma \ref{simplification} twice. Next, to establish the coalitional carrier-dependence property for a value of the form \eqref{linear general variant coalitional'}, we should establish a condition similar to the one from Proposition \ref{properties}:
\begin{equation}\label{auxiliary20}
\sum_{S\subseteq N\setminus\{i\},S\cap U=S'}w(S;N,i,\mathfrak{P})=w(S';U,i,\mathfrak{P}')  \quad 
\forall i\in U\subsetneq N, S'\subseteq U\setminus\{i\}.
\end{equation}
 On the left-hand side, we are dealing with subsets fitting between $S'\subset S'\sqcup(N\setminus U)$. Setting $Z=S'$ and $W=S'\sqcup(N\setminus U)$, 
consider $\mathcal{Z}$ and $\mathcal{W}$ as in \eqref{auxiliary19'}. If $\mathcal{Z}\nsubseteq\mathcal{W}$, then, as discussed before, the left-hand side of \eqref{auxiliary20} is zero. Notice that $\mathcal{Z}\nsubseteq\mathcal{W}$ means that there exists an 
$r\in R\setminus\{j\}$ with 
$S'\cap S_r\neq\varnothing$ and $S_r\nsubseteq S'\sqcup(N\setminus U)$. 
But then $w(S';U,i,\mathfrak{P}')$ should be zero because $S'$ intersects $S_r\cap U\in\mathfrak{P}'$ non-trivially, an element distinct  from 
$S_j\cap U\in\mathfrak{P}'$ (which contains $i$). 
Alternatively, when $\mathcal{Z}\subseteq\mathcal{W}$,
$S'$ may be written as the disjoint union
$$\left(\cup_{r\in\mathcal{Z}}(S'\cap S_r)\right)\cup(S'\cap S_j)
=\left(\cup_{r\in\mathcal{Z}}(S_r\cap U)\right)\cup(S'\cap S_j).$$ 
Thus
$w(S';U,i,\mathfrak{P}')=\alpha_1(|\mathcal{Z}|,|\mathfrak{P}'|)\cdot\alpha_2(|S'\cap S_j|,|S_j\cap U|).$
On the other hand, again, the left-hand side of \eqref{auxiliary20} may be rewritten as \eqref{auxiliary20'}
and then simplified by  applying Lemma \ref{simplification} to obtain
\small
$$\alpha_1(|\mathcal{Z}|,|M|+|\mathcal{Z}|-|\mathcal{W}|)
\cdot\alpha_2(|Z\cap S_j|,|S_j|+|Z\cap S_j|-|W\cap S_j|)
=\alpha_1(|\mathcal{Z}|,|M\setminus(\mathcal{W}\setminus\mathcal{Z})|)\cdot 
\alpha_2(|S'\cap S_j|,|S_j\setminus(W\setminus Z)|).$$ 
\normalsize
This coincides with $\alpha_1(|\mathcal{Z}|,|\mathfrak{P}'|)\cdot\alpha_2(|S'\cap S_j|,|S_j\cap U|)$ because $W\setminus Z$ is the complement of $U$, and 
$$
\mathcal{W}\setminus\mathcal{Z}=\{r\in M\mid S_r\subseteq W\setminus Z=N\setminus U\}
=\{r\in M\mid S_r\cap U=\varnothing\}
=\{r\in M\mid S_r\cap U\notin\mathfrak{P}'\}.
$$
Finally, to establish part (3), notice that the coalitional value 
$\mathfrak{h}^{\mathcal{A}_1,\mathcal{A}_2}$  admits a 
\textit{two-step formulation} (cf. \cite[\Sec 4.4]{2021arXiv210210878M}): 
\eqref{coalitional generalization} may be written as 
$$
\mathfrak{h}_i^{\mathcal{A}_1,\mathcal{A}_2}[N,v,\mathfrak{P}]
=h^{\mathcal{A}_2}_i\big[S_j,v^{(j)}\big]
$$
where 
$$
v^{(j)}(T):=h^{\mathcal{A}_1}_j\big[M,\hat{v}_T\big]\quad (T\subseteq S_j)
$$
with the \textit{intermediate game} $\hat{v}_T$ defined as 
$$
\hat{v}_T(R):=
\begin{cases}
v\left(\cup_{r\in R}S_r\right)  &\text{ if }j\notin R,\\
v\left(\cup_{r\in R\setminus\{j\}}S_r\cup T\right)  &\text{ if }j\in R,
\end{cases}
\quad (R\subseteq M, T\subseteq S_j).
$$
When $h^{\mathcal{A}_2}$ is the Shapley value, due to the efficiency property, one has
$$\sum_{i\in S_j}\mathfrak{h}_i^{\mathcal{A}_1,\mathcal{A}_2}[N,v,\mathfrak{P}]
=\sum_{i\in S_j}h^{\mathcal{A}_2}_i\big[S_j,v^{(j)}\big]
=v^{(j)}(S_j)=h^{\mathcal{A}_1}_j\big[M,\hat{v}_{S_j}=v^{\mathfrak{P}}\big]$$
where $v^{\mathfrak{P}}$ is the quotient game (cf. \eqref{quotinet game}). 
It is clear from \eqref{coalitional generalization} that $h^{\mathcal{A}_1}_j\big[M,v^{\mathfrak{P}}\big]$ is the same as 
$\mathfrak{h}^{\mathcal{A}_1,\mathcal{A}_2}_i\big[M,v^{\mathfrak{P}},\bar{M}\big]$, thus the quotient game property. 
\end{proof}

\section{Generalizations of Theorem \ref{Catboost theorem}}\label{appendix:generalization}
With proper adjustments, the formula presented in Theorem 
\ref{Catboost theorem} for marginal Shapley values of ensembles of oblivious trees carries over to more general game values $h^{\mathcal{A}}$ and $\mathfrak{h}^{\mathcal{A}_1,\mathcal{A}_2}$ introduced in \ref{subappendix:classification coalitional}.
For the sake of brevity, we consider the case of a single symmetric decision tree $T$, and we drop the dependence on $T$ in our notation.\\
\indent 
As in Notation-Convention \ref{groundwork} and the proof of Theorem \ref{Catboost theorem}, we denote the depth of $T$ by $m$, the distinct features on which $T$ splits (from top to bottom) by $X'_1,\dots,X'_k$, and the corresponding partition of $M=\{1,\dots,m\}$ by $\mathsf{p}$; that is, $\mathsf{p}$ is the partition of levels of $T$ determined by distinct features appearing in the tree which is indexed by $K=\{1,\dots,k\}$. In the case of coalitional values $\mathfrak{h}^{\mathcal{A}_1,\mathcal{A}_2}$, there also exists a partition 
$\mathfrak{P}:=\big\{\tilde{S}_1,\dots,\tilde{S}_{\tilde{m}}\big\}$ of $K$  
determined by the partition in hand of predictors $X'_1,\dots,X'_k$. 
This is indexed by $\tilde{M}:=\big\{1,\dots,\tilde{m}\big\}$.
Given $i\in K$, the theorem below provides formulas for marginal feature attributions 
$h^{\mathcal{A}}_i\big[v^\ME\big](\mathbf{x})$
and $\mathfrak{h}^{\mathcal{A}_1,\mathcal{A}_2}_i\big[v^\ME,\mathfrak{P}\big](\mathbf{x})$
on a rectangular region $R_{\mathbf{a}}$ corresponding to a leaf encoded by $\mathbf{a}$. In the coalitional case, we assume that $i$ belongs to the partition element $\tilde{S}_{j}$ where $j\in\tilde{M}$.
\begin{theorem}\label{last}
With the notation as above, for ${\rm{P}}_{\mathbf{X}}$-almost every  $\mathbf{x}\in R_{\mathbf{a}}$ one has 
\begin{equation}\label{symmetric formula tree generic}
\begin{split}
h^{\mathcal{A}}_i\big[v^\ME\big](\mathbf{x})=&
\sum_{\substack{Z\subseteq K\\ i\in Z}}   
\sum_{\substack{W\subseteq K\\
W\supseteq Z}}\alpha(|Z|-1,|K|+|Z|-|W|)
\cdot 
\Big(\sum_{\mathbf{b}\in\mathcal{E}^{-1}(\mathbf{a},W;\mathsf{p})\cap\mathcal{R}}
\mathfrak{s}(\mathbf{b},-Z)\Big)\\
&-\sum_{\substack{W\subseteq K\\ i\notin W}}   
\sum_{\substack{Z\subseteq K\\
Z\subseteq W}}\alpha(|Z|,|K|+|Z|-|W|)\cdot
\Big(\sum_{\mathbf{b}\in\mathcal{E}^{-1}(\mathbf{a},W;\mathsf{p})\cap\mathcal{R}}
\mathfrak{s}(\mathbf{b},-Z)\Big),
\end{split}    
\end{equation}
\normalsize
and
\small
\begin{equation}\label{symmetric formula tree Owen}
\begin{split}
\mathfrak{h}^{\mathcal{A}_1,\mathcal{A}_2}_i\big[v^\ME,\mathfrak{P}\big](\mathbf{x})=\\
\sum_{\substack{\mathcal{W}\subseteq \tilde{M}\\ j\notin \mathcal{W}}}  
\sum_{\substack{\mathcal{Z}\subseteq \tilde{M}\\
\mathcal{Z}\subseteq \mathcal{W}}}
\alpha_1(|\mathcal{Z}|,|\tilde{M}|+|\mathcal{Z}|-|\mathcal{W}|)\cdot
\Bigg(
&\sum_{\substack{Z_*\subseteq \tilde{S}_{j}\\ i\in Z_*}}   
\sum_{\substack{W_*\subseteq\tilde{S}_{j}\\
W_*\supseteq Z_*}}\alpha_2(|Z_*|-1,|\tilde{S}_j|+|Z_*|-|W_*|)
\cdot
\Big(\sum_{\substack{Z\in\mathsf{Int}(\mathcal{Z},Z_*,j;\mathfrak{P}) \\
W\in\mathsf{Inc}(\mathcal{W},W_*,j;\mathfrak{P}) }}
\mathcal{S}(\mathbf{a},Z,W;\mathsf{p},\mathcal{R})\Big)\\
&-\sum_{\substack{W_*\subseteq \tilde{S}_{j}\\ i\notin W_*}}   
\sum_{\substack{Z_*\subseteq\tilde{S}_{j}\\
Z_*\subseteq W_*}}\alpha_2(|Z_*|,|\tilde{S}_j|+|Z_*|-|W_*|)
\cdot
\Big(\sum_{\substack{Z\in\mathsf{Int}(\mathcal{Z},Z_*,j;\mathfrak{P}) \\W\in\mathsf{Inc}(\mathcal{W},W_*,j;\mathfrak{P}) }}
\mathcal{S}(\mathbf{a},Z,W;\mathsf{p},\mathcal{R})\Big)\Bigg)
\end{split}    
\end{equation}
\normalsize
where terms $\mathcal{E}^{-1}$ and $\mathfrak{s}$ are defined as in 
\eqref{auxiliary5'} and \eqref{auxiliary5''}; and 
\begin{equation}\label{auxiliary21'}
\mathcal{S}(\mathbf{a},Z,W;\mathsf{p},\mathcal{R})
:=\sum_{\mathbf{b}\in\mathcal{E}^{-1}(\mathbf{a},W;\mathsf{p})\cap\mathcal{R}}
\mathfrak{s}(\mathbf{b},-Z);
\end{equation}
and for any $j\in \tilde{M}$, $Q_*\subseteq\tilde{S}_j$ and 
$\mathcal{Q}\subseteq\tilde{M}\setminus\{j\}$,  
$\mathsf{Int}(\mathcal{Q},Q_*,j;\mathfrak{P})$ 
and $\mathsf{Inc}(\mathcal{Q},Q_*,j;\mathfrak{P})$
are defined as 
\begin{equation}\label{two modes}
\begin{split}
& \mathsf{Int}(\mathcal{Q},Q_*,j;\mathfrak{P}):=
\big\{Q\mid Q\subseteq K, Q\cap\tilde{S}_j=Q_*,
\{r\in\tilde{M}\setminus\{j\}\mid Q\cap\tilde{S}_r\neq\varnothing\}=\mathcal{Q}\big\},\\
&\mathsf{Inc}(\mathcal{Q},Q_*,j;\mathfrak{P}):=
\big\{Q\mid Q\subseteq K, Q\cap\tilde{S}_j=Q_*,
\{r\in\tilde{M}\setminus\{j\}\mid \tilde{S}_r\subseteq Q\}=\mathcal{Q}\big\}.
\end{split}
\end{equation}

\end{theorem}

\begin{proof}
Replacing the Shapley value with something of the form \eqref{linear general variant} or \eqref{linear general variant coalitional}, one only needs to accordingly adjust  the summations over subsets $Q$ that appear in \eqref{auxiliary12'}.
In the case of $h^{\mathcal{A}}$, we only need to use Corollary \ref{simplification} to simplify 
$\sum_{\{Q\mid Z\setminus\{i\}\subseteq Q\subseteq W\setminus\{i\}\}}
\alpha(|Q|,|K|)$ as $\alpha(|Z|-1,|K|+|Z|-|W|)$
and 
$\sum_{\{Q\mid Z\subseteq Q\subseteq W\}}\alpha(|Q|,|K|)$ 
as $\alpha(|Z|,|K|+|Z|-|W|)$. \\
\indent
As for the coalitional game value $\mathfrak{h}^{\mathcal{A}_1,\mathcal{A}_2}$, the aforementioned summations in \eqref{auxiliary12'} should be simplified by invoking the second part of Proposition \ref{simplification coalitional}.
Notice that not all pairs $Z,W$ with $i\in Z\subseteq W\subseteq K$ or 
with $Z\subseteq W\subseteq K\setminus\{i\}$ come up, because the definition \eqref{coalitional generalization} of $\mathfrak{h}^{\mathcal{A}_1,\mathcal{A}_2}$ allows only subsets $Q$ that, except $\tilde{S}_{j}$, intersect any other element of the partition 
$\mathfrak{P}=\big\{\tilde{S}_1,\dots,\tilde{S}_{\tilde{m}}\big\}$ of $K$ either fully or trivially. 
Indeed, adapting the notation from Proposition \ref{simplification coalitional} to the setting here, such a subset $Q$ fits between $Z$ and $W$ 
(or between $Z\setminus\{i\}$ and $W\setminus\{i\}$ when $i\in Z\subseteq W$)
if and only if 
\begin{equation}\label{auxiliary19''}
\mathcal{Z}:=
\big\{r\in \tilde{M}\setminus\{j\}\mid Z\cap \tilde{S}_r\neq\varnothing\big\}\subseteq 
\mathcal{W}:=
\big\{r\in \tilde{M}\setminus\{j\} \mid \tilde{S}_r\subseteq W\big\}.
\end{equation}
We next modify the expression \eqref{auxiliary12'} for the value of 
$\varphi_i\big[v^\ME\big]$ over $R_{\mathbf{a}}$ to obtain the value of 
$\mathfrak{h}^{\mathcal{A}_1,\mathcal{A}_2}_i\big[v^\ME,\mathfrak{P}\big]$ over that rectangular region.
The expression
$\sum_{\mathbf{b}\in\mathcal{E}^{-1}(\mathbf{a},W;\mathsf{p})\cap\mathcal{R}}
c_{\mathbf{b}}\cdot\Big(\sum_{\mathbf{u}\in\mathcal{E}^{-1}
(\mathbf{b},-Z;\mathsf{p})\cap\mathcal{R}}p_{\mathbf{u}}\Big)$
therein is just 
$\mathcal{S}(\mathbf{a},Z,W;\mathsf{p},\mathcal{R})$
as defined in \eqref{auxiliary21'}.
Moreover, the summations of coefficients of $\mathfrak{h}^{\mathcal{A}_1,\mathcal{A}_2}$ over subsets of $Q$ fitting between $Z$ and $W$ may be simplified as in \eqref{auxiliary19}. One can reorder the terms in \eqref{auxiliary12'} to first sum over subsets 
$\mathcal{Z}\subseteq\mathcal{W}\subseteq \tilde{M}\setminus\{j\}$,
then over intersections with $\tilde{S}_j$
$$
Z_*:=Z\cap\tilde{S}_j\subseteq W_*:=W\cap\tilde{S}_j,
$$
and finally, over subsets $Z\subseteq W$ for which $\mathcal{Z},Z_*,\mathcal{W}$ and $W_*$ are prescribed.
This last summation takes place over 
\small
$$
\Big\{(Z,W)\mid Z\subseteq W\subseteq K, 
\{r\in\tilde{M}\setminus\{j\}\mid Z\cap\tilde{S}_r\neq\varnothing\}=\mathcal{Z},
\{r\in\tilde{M}\setminus\{j\}\mid \tilde{S}_r\subseteq W\}=\mathcal{W},
Z\cap\tilde{S}_j=Z_*,
W\cap\tilde{S}_j=W_*\Big\}.
$$
\normalsize
This can be described by 
$Z\in\mathsf{Int}(\mathcal{Z},Z_*,j;\mathfrak{P})$ 
and 
$W\in\mathsf{Inc}(\mathcal{W},W_*,j;\mathfrak{P})$---the inclusion $Z\subseteq W$ follows automatically from inclusions 
$\mathcal{Z}\subseteq\mathcal{W}$ and $Z_*\subseteq W_*$.
The details are left to the reader.
\end{proof}
We finish by analyzing the complexity formulas exhibited in Theorem \ref{last} for feature attributions. The complexity of \eqref{symmetric formula tree generic} is no different than the complexity presented in Theorem \ref{Catboost theorem} for the Shapley value. 
As for the case of \eqref{symmetric formula tree Owen} and coalitional values, we carry out a complexity analysis below.
\begin{proposition}\label{complexity Owen}
Let $T$ be an oblivious decision tree of depth $m$ splitting on $k$ distinct features 
$(X'_1,\dots,X'_k)$ which results in a partition $\mathsf{p}$ of $M$ indexed by $K$ 
(see Definition-Convention \ref{groundwork}). 
Suppose the features $(X'_1,\dots,X'_k)$ are grouped based on a partition 
$\mathfrak{P}=\big\{\tilde{S}_1,\dots,\tilde{S}_{\tilde{m}}\big\}$ of $K$. 
The formula \eqref{symmetric formula tree Owen} 
for the marginal coalitional value of a feature at a given explicand, once completely expanded, has no more than
\begin{equation}\label{complexity0'}
\min\Big(2^k\cdot (1.5)^{|\tilde{S}_j|-1}\cdot (2.5)^{|\mathfrak{P}|-1},2\cdot 3^{k-1}\Big)
\cdot\left(\frac{m}{k}\right)^k
\end{equation}
terms where $\tilde{S}_j$ is the partition element that contains the feature under consideration.
\end{proposition}

\begin{proof}
In \eqref{symmetric formula tree Owen}, terms are added up by first iterating over 
$\mathcal{Z},Z_*,\mathcal{W}$ and $W_*$, 
and then over elements $Z$ of 
$\mathsf{Int}(\mathcal{Z},Z_*,j;\mathfrak{P})$
and elements $W$ of 
$\mathsf{Inc}(\mathcal{W},W_*,j;\mathfrak{P})$.
Given $\mathcal{Z}\subseteq\tilde{M}\setminus\{j\}$ and 
$Z_*$, the number of choices for $Z\subseteq K$ is no more than 
$$
\prod_{r\in \mathcal{Z}}\left(2^{|\tilde{S}_r|}-1\right)
< 2^{\sum_{r\in \mathcal{Z}}|\tilde{S}_r|}
$$
because, among elements of the partition 
$\mathfrak{P}=\big\{\tilde{S}_1,\dots,\tilde{S}_{\tilde{m}}\big\}$ of $K$,  
$Z$ can only intersect $\tilde{S}_r$ when $r\in\mathcal{Z}$, or when $r=j$ in which case the intersection 
$Z\cap\tilde{S}_r$ is $Z_*$. By a similar argument, 
the number of choices for $W\subseteq K$ once $\mathcal{W}\subseteq\tilde{M}\setminus\{j\}$ and $W_*$ are provided is no more than 
$$
\prod_{r\in \tilde{M}\setminus(\mathcal{W}\cup\{j\})}
\left(2^{|\tilde{S}_r|}-1\right)
<2^{\sum_{r\in \tilde{M}\setminus(\mathcal{W}\cup\{j\})}|\tilde{S}_r|}
$$
since the intersection of $W$ with $(\cup_{r\in\mathcal{W}}\tilde{S}_r)\cup\tilde{S}_j$ is then known.  
Therefore, once $\mathcal{Z},Z_*,\mathcal{W}$ and $W_*$ are decided on in \eqref{symmetric formula tree Owen}, the number of possibilities for $Z$ and $W$ is smaller than
$$
2^{\sum_{r\in \mathcal{Z}}|\tilde{S}_r|}\cdot
2^{\sum_{r\in \tilde{M}\setminus(\mathcal{W}\cup\{j\})}|\tilde{S}_r|}
=2^{|K|-\sum_{r\in \mathcal{W}\setminus\mathcal{Z}}|\tilde{S}_r|-|\tilde{S}_j|}\leq 
2^{k-|\mathcal{W}\setminus\mathcal{Z}|-|\tilde{S}_j|}.
$$
On the other hand, from the proof of Theorem \ref{Catboost theorem}, the number of terms in 
$\mathcal{S}(\mathbf{a},Z,W;\mathsf{p},\mathcal{R})$, once expanded, is 
$\prod_{q\in K\setminus(W\setminus Z)}|S_q|$; see \eqref{possibilities}.
This is no greater than $\left(\frac{m}{k}\right)^k$ as established in \eqref{AM-GM}.
Finally, notice that the number of pairs $(Z_*,W_*)$ of nested subsets is $3^{|\tilde{S}_j|-1}$, 
either when $i\in Z_*\subseteq W_*\subseteq\tilde{S}_j$ or when 
$Z_*\subseteq W_*\subseteq\tilde{S}_j\setminus\{i\}$. 
All in all,   considered as a weighted sum of   products of the form 
$c_{\mathbf{b}}\cdot p_{\mathbf{u}}$ once expanded, \eqref{symmetric formula tree Owen} has less than 
\begin{equation}\label{auxiliary22}
2\cdot 3^{|\tilde{S}_j|-1}\cdot\left(\frac{m}{k}\right)^k\cdot
\Big(\sum_{\substack{\mathcal{W}\subseteq \tilde{M}\\ j\notin \mathcal{W}}}  
\sum_{\substack{\mathcal{Z}\subseteq \tilde{M}\\
\mathcal{Z}\subseteq \mathcal{W}}}2^{k-|\mathcal{W}\setminus\mathcal{Z}|-|\tilde{S}_j|}\Big)    
=(1.5)^{|\tilde{S}_j|-1}\cdot\left(\frac{m}{k}\right)^k\cdot 2^k\cdot
\Big(\sum_{\substack{\mathcal{W}\subseteq \tilde{M}\\ j\notin \mathcal{W}}}  
\sum_{\substack{\mathcal{Z}\subseteq \tilde{M}\\
\mathcal{Z}\subseteq \mathcal{W}}}2^{-|\mathcal{W}\setminus\mathcal{Z}|}\Big)
\end{equation}
terms.  
Moreover, the expression in the parentheses may be computed.
Denoting $\mathcal{W}\setminus\mathcal{Z}$ by $\mathcal{Y}$ and its cardinality by $y$:
\begin{equation*}
\begin{split}
&\sum_{\substack{\mathcal{W}\subseteq \tilde{M}\\ j\notin \mathcal{W}}}  
\sum_{\substack{\mathcal{Z}\subseteq \tilde{M}\\
\mathcal{Z}\subseteq \mathcal{W}}}2^{-|\mathcal{W}\setminus\mathcal{Z}|}
=\sum_{\mathcal{Y}\subseteq \tilde{M}\setminus\{j\}}
\sum_{\mathcal{Z}\subseteq\tilde{M}\setminus(\mathcal{Y}\cup\{j\})}2^{-|\mathcal{Y}|}
=\sum_{\mathcal{Y}\subseteq \tilde{M}\setminus\{j\}}2^{-|\mathcal{Y}|}\cdot 
2^{|\tilde{M}|-1-|\mathcal{Y}|}\\
&=\sum_{y=0}^{\tilde{m}-1}\binom{\tilde{m}-1}{y}2^{\tilde{m}-1-2y}
=(2.5)^{\tilde{m}-1}=(2.5)^{|\mathfrak{P}|-1}.
\end{split}    
\end{equation*}
Substituting in \eqref{auxiliary22}, we obtain the bound 
$2^k\cdot (1.5)^{|\tilde{S}_j|-1}\cdot (2.5)^{|\mathfrak{P}|-1}\cdot
\left(\frac{m}{k}\right)^k$. 
Of course,  the formula for a coalitional value such as \eqref{coalitional generalization} cannot have more terms than the formula for the Shapley value 
from Theorem \ref{Catboost theorem}. Taking minimum with the complexity \eqref{complexity0} from that theorem,  we arrive at the desired complexity \eqref{complexity0'}.
\end{proof}

\begin{figure}
\includegraphics[width=14cm]{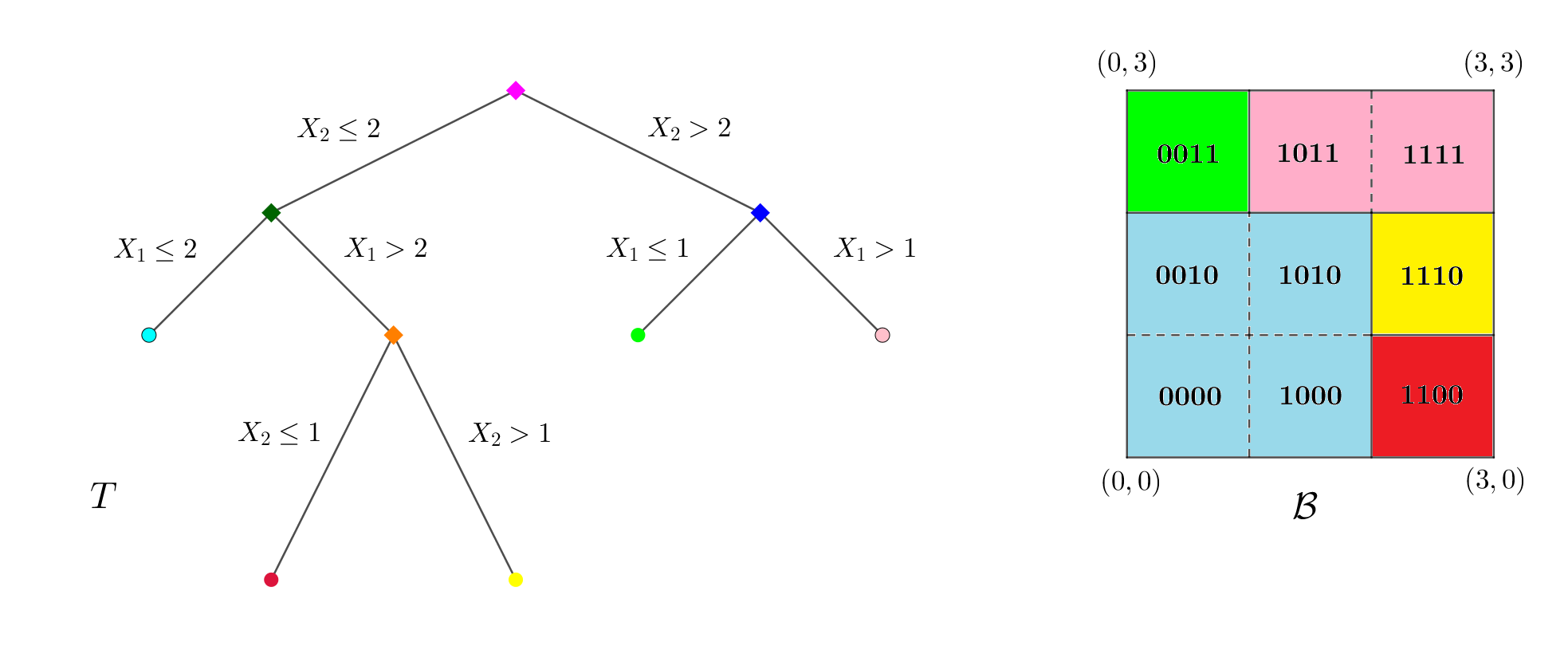}
\includegraphics[width=14cm]{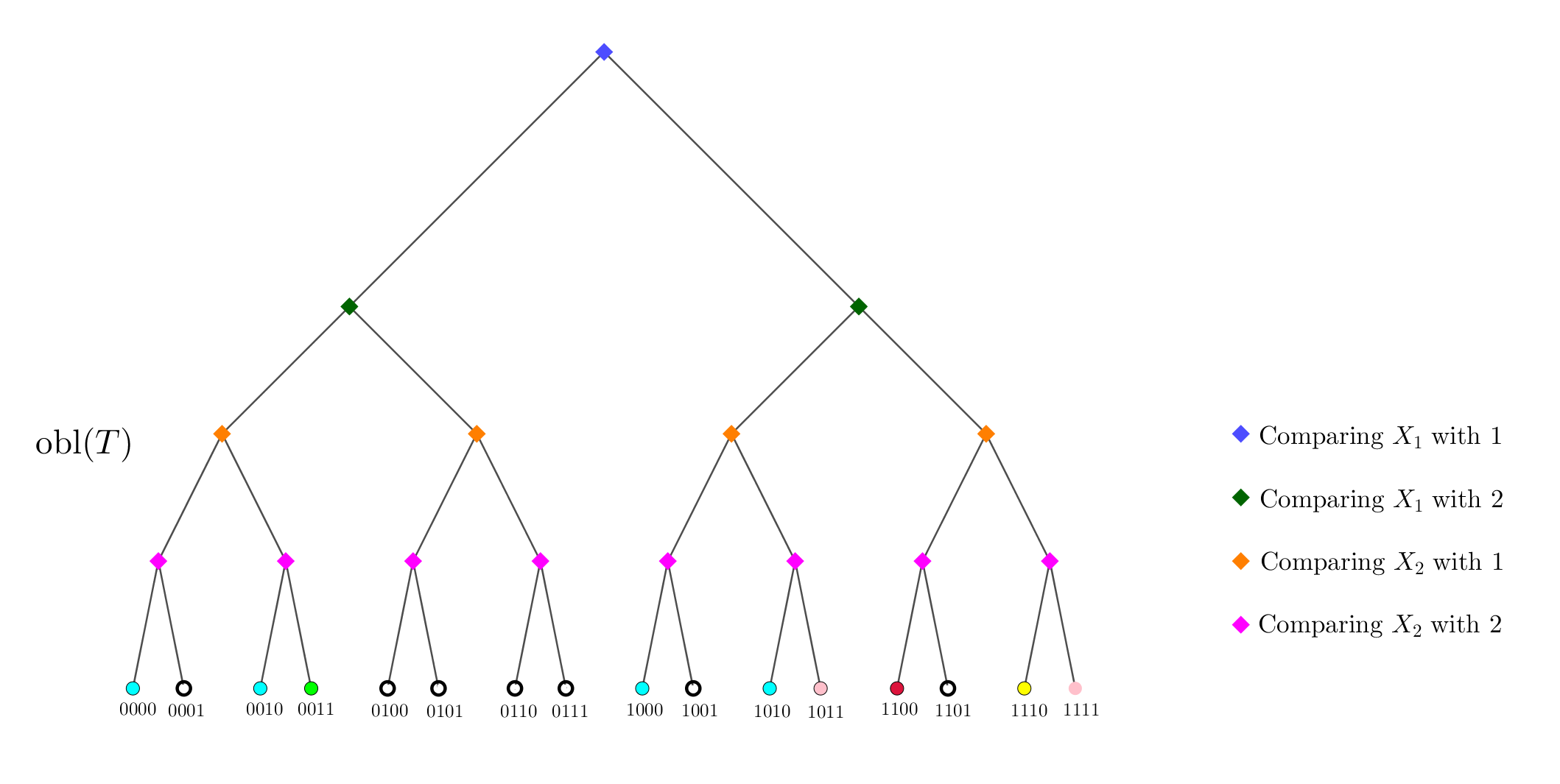}
\caption{The picture illustrating the construction outlined in Appendix \ref{appendix:general tree} 
which to a decision tree $T$ assigns an oblivious decision tree ${\rm{obl}}(T)$ computing the same function. 
Here, the features $X_1$ and $X_2$ are supported in the square $\mathcal{B}=[0,3]\times [0,3]$. 
On the top, a decision tree $T$ is demonstrated along with its induced partition 
$\mathscr{P}(T)$ of $\mathcal{B}$ where each piece of $\mathscr{P}(T)$ has the same color as the corresponding leaf of $T$. 
On the bottom, the oblivious decision tree ${\rm{obl}}(T)$ is shown which, at each level, splits on the feature and threshold appearing at an internal node of $T$. 
(Each internal node of ${\rm{obl}}(T)$ is colored similarly as its corresponding internal node from $T$.)
At each split of ${\rm{obl}}(T)$, we go right if the feature is larger than the threshold and go left otherwise; hence an encoding of leaves of ${\rm{obl}}(T)$ with binary codes. The leaves which are colored are realizable (i.e. no conflicting thresholds along the path to them), and are in a bijection with elements of the grid partition $\widetilde{\mathscr{P}(T)}$ of $\mathcal{B}$. (For an element of $\widetilde{\mathscr{P}(T)}$, 
the colors of the corresponding leaf of ${\rm{obl}}(T)$ and the element of the coarser partition $\mathscr{P}(T)$ containing it coincide.)}
\label{fig:construction}
\end{figure}

\begin{figure}
\includegraphics[width=5cm]{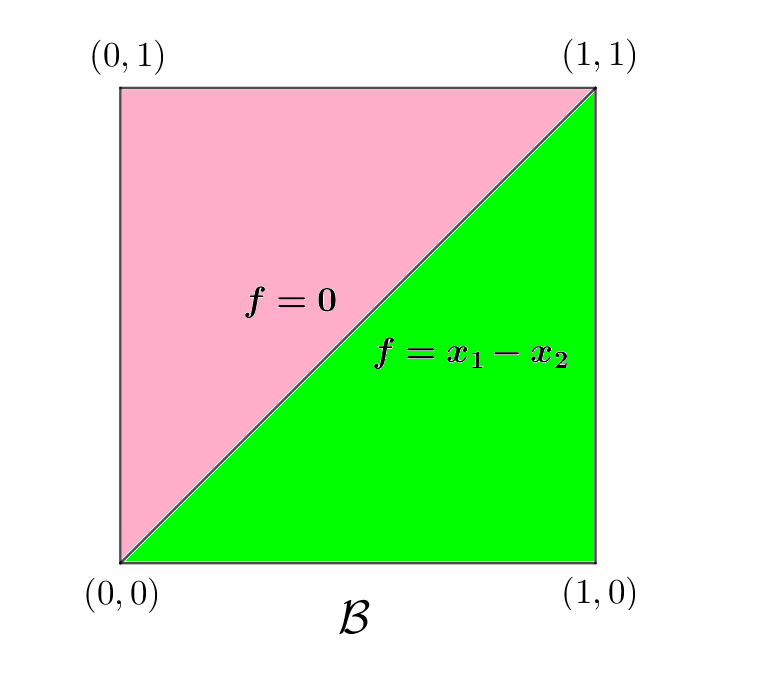}
\caption{The picture for Example \ref{triangular} illustrating a piecewise-linear function $f$ computed by a ReLU network. 
The features $(X_1,X_2)$ are supported in $\mathcal{B}=[0,1]\times[0,1]$ and $f(x_1,x_2)=(x_1-x_2)^+$. 
The Shapley values become non-linear; see \eqref{non-linear}.}
\label{fig:triangular}
\end{figure}

\section{A remark on non-oblivious trees}\label{appendix:general tree}
In this short appendix, we point out that Theorem \ref{Catboost theorem} can be utilized to explicitly compute marginal game values such as Shapley or Banzhaf for any decision tree $T$.
To this end, we present a systematic method of turning a general decision tree $T$ into  an oblivious decision tree ${\rm{obl}}(T)$ with the same input-output function---a tree to which Theorem \ref{Catboost theorem} can then be applied. 
\\
\indent
Impose an ordering on features appearing in $T$, and then order the splits (feature,threshold) occurring in $T$ lexicographically. Next, take ${\rm{obl}}(T)$ to be the oblivious decision tree which, throughout its $i^{\rm{th}}$ level (the root being on the $1^{\rm{st}}$ level), splits on the $i^{\rm{th}}$ split of $T$.
This is all illustrated in Figure \ref{fig:construction}.
Clearly, the function computed by ${\rm{obl}}(T)$ is identical to the one that $T$ computes. But the partition determined by the former is more intricate---it is the grid completion $\widetilde{\mathscr{P}(T)}$  of the partition $\mathscr{P}(T)$ determined by the latter. 
The formula from Theorem \ref{Catboost theorem} (or its generalizations in Appendix \ref{appendix:generalization}) 
can now be applied to ${\rm{obl}}(T)$  to obtain marginal Shapley (or Banzhaf or Owen) values of the original tree $T$ in terms of the leaf values of $T$ and the probabilities of regions in $\widetilde{\mathscr{P}(T)}$. Nevertheless, the probabilities associated with regions that lie in $\widetilde{\mathscr{P}(T)}\setminus\mathscr{P}(T)$
cannot be estimated directly based on leaf weights of the trained tree $T$; they should be
estimated separately using the training data or a background dataset (e.g. with a formula such as \eqref{probability estimation precomputation}). Assuming that these probabilities are provided for a non-oblivious tree $T$, the complexity of computing the marginal Shapley value of a feature is no higher than
\begin{equation*}
3^{k(T)-1}\cdot\left(\frac{\ell(T)-1}{k(T)}\right)^{k(T)} 
\end{equation*}
according to \eqref{complexity0}. Here, $k(T)$ is the number of distinct features appearing in $T$ (and hence in ${\rm{obl}}(T)$), and
$\ell(T)$ is the number of leaves of $T$ (so the number of internal nodes of $T$, $\ell(T)-1$, is an upper bound for the depth of ${\rm{obl}}(T)$).  
Hence for an ensemble $\mathcal{T}$ of decision trees in which the number of distinct features per tree is limited, it may be feasible to compute marginal Shapley (or Banzhaf or Owen) feature attributions by applying the theory developed in \Sec \ref{subsec:CatBoost} to the ensemble $\{{\rm{obl}}(T)\mid T\in\mathcal{T}\}$ of oblivious trees.

\section{Comparison with ReLU networks}\label{appendix:ReLU}
It is well known that a neural network with ReLU activations  cuts the input space into convex polytope regions such that on each of them the computed function is linear \cite{montufar2014number}. 
These linear regions are determined by \textit{activation patterns}, i.e. the on/off states of hidden neurons. 
Deep ReLU networks often result in very complicated partitions of the space into polytopes. These partitions can moreover be utilized to study the networks' expressive power \cite{raghu2017expressive}. \\
\indent 
One can immediately draw a parallel between the partition into polytope regions induced by a ReLU network and the rectangular partition of the input space that a tree-based model determines. 
With respect to the partition in hand, the input-output function is piecewise linear in the case of the former and piecewise constant in the case of the latter. Indeed, it is suggested that ``unwrapping'' a ReLU network into  its activation regions may be used for generating global feature importance values as well as creating visual aids for interpretation and diagnostic \cite{2020arXiv201104041S}. 
Here, we shall see that when it comes to game-theoretic local feature attributions, Theorem \ref{main theorem} for tree-based models does not carry over to ReLU networks.

\begin{example}\label{triangular}
Suppose the predictors $(X_1,X_2)$ are supported in the unit square and consider the piecewise constant function 
$$
f(x_1,x_2)=(x_1-x_2)^+
=\begin{cases}
x_1-x_2 & \text{if }x_1>x_2,\\
0       & \text{otherwise;}
\end{cases}
$$
which is obtained from applying ReLU to a linear function. The unit square is thus cut into two triangles as in Figure \ref{fig:triangular}. Suppose $X_1$ and $X_2$ admit probability density functions, and denote them by 
$\rho_1$ and $\rho_2$ respectively. Also denote the joint probability density function for $(X_1,X_2)$ by $\rho_3$.
Identity \eqref{two-player Shapley} yields the marginal Shapley values of the model as
\begin{equation}\label{non-linear}
\begin{split}
\varphi_1\big[v^\ME\big](x_1,x_2)&=
\frac{1}{2}\left(\int_0^{x_1}(x_1-t)\rho_2(t){\rm{d}}t-\int_{x_2}^1(t-x_2)\rho_1(t){\rm{d}}t\right)\\
&+\frac{1}{2}\left((x_1-x_2)^+
-\int_{0}^1\int_{t_2}^1(t_1-t_2)\rho_3(t_1,t_2){\rm{d}}t_1{\rm{d}}t_2\right),\\
\varphi_2\big[v^\ME\big](x_1,x_2)&=
\frac{1}{2}\left(\int_{x_2}^1(t-x_2)\rho_1(t){\rm{d}}t-\int_0^{x_1}(x_1-t)\rho_2(t){\rm{d}}t\right)\\
&+\frac{1}{2}\left((x_1-x_2)^+
-\int_{0}^1\int_{t_2}^1(t_1-t_2)\rho_3(t_1,t_2){\rm{d}}t_1{\rm{d}}t_2\right).
\end{split}    
\end{equation}
The expressions above are non-linear; so a piecewise-linear ReLU model does not necessarily result in piecewise-linear marginal feature attributions. This is in contrast to the case of piecewise-constant functions computed by tree-based models where, according to Theorem \ref{main theorem}, marginal feature attributions are piecewise constant. Moreover, compared to tree ensembles, the null-player axiom cannot be utilized as easily. For instance, in the example above, $f$ is constant on the triangular region $x_1<x_2$ but this does not mean that  $\varphi_1\big[v^\ME\big]$ and $\varphi_2\big[v^\ME\big]$ are constant over that region. 
\end{example}
The last example demonstrates some of the subtleties that emerge once one works with polytope regions arising from ReLU networks instead of rectangular regions determined by tree-based models. As a matter of fact, it easily follows from the arguments in the proof of Theorem \ref{main theorem} that if a model is piecewise constant or piecewise linear with respect to a rectangular partition $\mathscr{P}$ of the input space, then its feature attributions are the same type of functions with respect to the grid partition $\widetilde{\mathscr{P}}$. 
As observed in the example, this fails if $\mathscr{P}$ is not rectangular because then the feature attributions can become completely non-linear. 

\textbf{Acknowledgments.} The authors would like to thank Raghu Kulkarni, Patrick
Hall, Kenrick Fernandes, Ashkan Golgoon and Mohammad Ahmadpoor for fruitful
conversations and helpful comments. We are also grateful to the anonymous referees
for their valuable feedback.

\bibliography{biblography}

\begin{thebibliography}{10}

\bibitem{aas2021explaining}
K.~Aas, M.~Jullum, and A.~L{\o}land.
\newblock Explaining individual predictions when features are dependent: More accurate approximations to {S}hapley values.
\newblock {\em Artificial Intelligence}, 298:103502, 2021.

\bibitem{al2019comparison}
E.~Al~Daoud.
\newblock Comparison between {XGB}oost, {L}ight{GBM} and {C}at{B}oost using a home credit dataset.
\newblock {\em International Journal of Computer and Information Engineering}, 13(1):6--10, 2019.

\bibitem{alvarez2018towards}
D.~Alvarez~Melis and T.~Jaakkola.
\newblock Towards robust interpretability with self-explaining neural networks.
\newblock {\em Advances in neural information processing systems}, 31, 2018.

\bibitem{amoukou2022accurate}
S.~I. Amoukou, T.~Sala{\"u}n, and N.~Brunel.
\newblock Accurate {S}hapley values for explaining tree-based models.
\newblock In {\em International Conference on Artificial Intelligence and Statistics}, pages 2448--2465. PMLR, 2022.

\bibitem{aumann1974cooperative}
R.~J. Aumann and J.~H. Dreze.
\newblock Cooperative games with coalition structures.
\newblock {\em International Journal of game theory}, 3(4):217--237, 1974.

\bibitem{banzhaf1964weighted}
J.~F. Banzhaf~III.
\newblock Weighted voting doesn't work: A mathematical analysis.
\newblock {\em Rutgers L. Rev.}, 19:317, 1964.

\bibitem{breiman2001random}
L.~Breiman.
\newblock Random forests.
\newblock {\em Machine learning}, 45(1):5--32, 2001.

\bibitem{campbell2022exact}
T.~W. Campbell, H.~Roder, R.~W. Georgantas~III, and J.~Roder.
\newblock Exact {S}hapley values for local and model-true explanations of decision tree ensembles.
\newblock {\em Machine Learning with Applications}, page 100345, 2022.

\bibitem{caruana2006empirical}
R.~Caruana and A.~Niculescu-Mizil.
\newblock An empirical comparison of supervised learning algorithms.
\newblock In {\em Proceedings of the 23rd international conference on Machine learning}, pages 161--168, 2006.

\bibitem{2022arXiv220707605C}
H.~{Chen}, I.~C. {Covert}, S.~M. {Lundberg}, and S.-I. {Lee}.
\newblock {Algorithms to estimate Shapley value feature attributions}.
\newblock {\em arXiv e-prints}, page arXiv:2207.07605v1, July 2022.

\bibitem{2020arXiv200616234C}
H.~{Chen}, J.~D. {Janizek}, S.~{Lundberg}, and S.-I. {Lee}.
\newblock {True to the Model or True to the Data?}
\newblock {\em arXiv e-prints}, page arXiv:2006.16234, June 2020.

\bibitem{chen2021explaining}
H.~Chen, S.~Lundberg, and S.-I. Lee.
\newblock Explaining models by propagating {S}hapley values of local components.
\newblock {\em Explainable AI in Healthcare and Medicine: Building a Culture of Transparency and Accountability}, pages 261--270, 2021.

\bibitem{chen2016xgboost}
T.~Chen and C.~Guestrin.
\newblock {XGB}oost: A scalable tree boosting system.
\newblock In {\em Proceedings of the 22nd acm sigkdd international conference on knowledge discovery and data mining}, pages 785--794, 2016.

\bibitem{chopra2018application}
A.~Chopra and P.~Bhilare.
\newblock Application of ensemble models in credit scoring models.
\newblock {\em Business Perspectives and Research}, 6(2):129--141, 2018.

\bibitem{datta2016algorithmic}
A.~Datta, S.~Sen, and Y.~Zick.
\newblock Algorithmic transparency via quantitative input influence: Theory and experiments with learning systems.
\newblock In {\em 2016 IEEE symposium on security and privacy (SP)}, pages 598--617. IEEE, 2016.

\bibitem{dietterich2000ensemble}
T.~G. Dietterich.
\newblock Ensemble methods in machine learning.
\newblock In {\em International workshop on multiple classifier systems}, pages 1--15. Springer, 2000.

\bibitem{CatBoostdoc}
Documentation.
\newblock {C}at{B}oost \url{https://catboost.ai/en/docs}.

\bibitem{LightGBMdoc}
Documentation.
\newblock {L}ight{GBM} \url{https://lightgbm.readthedocs.io/en/latest/index.html}.

\bibitem{TreeSHAPdoc}
Documentation.
\newblock {T}ree{SHAP} \url{https://shap-lrjball.readthedocs.io/en/latest/generated/shap.TreeExplainer.html}.

\bibitem{XGBoostdoc}
Documentation.
\newblock {XGB}oost \url{https://xgboost.readthedocs.io/en/stable}.

\bibitem{2018arXiv181011363V}
A.~{Dorogush}, V.~{Ershov}, and A.~{Gulin}.
\newblock {CatBoost: gradient boosting with categorical features support}.
\newblock {\em arXiv e-prints}, page arXiv:1810.11363, Oct. 2018.

\bibitem{ECOA}
ECOA.
\newblock Equal {C}redit {O}pportunity {A}ct \url{https://www.justice.gov/crt/equal-credit-opportunity-act-3}.

\bibitem{misc_online_news_popularity_332}
K.~Fernandes, P.~Vinagre, P.~Cortez, and P.~Sernadela.
\newblock {Online News Popularity}.
\newblock UCI Machine Learning Repository, 2015.
\newblock {DOI}: \href{https://doi.org/10.24432/C5NS3V}{10.24432/C5NS3V}.

\bibitem{2016arXiv160905610F}
M.~{Ferov} and M.~{Modr{\'y}}.
\newblock {Enhancing LambdaMART Using Oblivious Trees}.
\newblock {\em arXiv e-prints}, page arXiv:1609.05610, Sept. 2016.

\bibitem{FHA}
FHA.
\newblock Fair {H}ousing {A}ct \url{https://www.justice.gov/crt/fair-housing-act-1}.

\bibitem{friedman2001greedy}
J.~H. Friedman.
\newblock Greedy function approximation: a gradient boosting machine.
\newblock {\em Annals of statistics}, pages 1189--1232, 2001.

\bibitem{grinsztajn2022tree}
L.~Grinsztajn, E.~Oyallon, and G.~Varoquaux.
\newblock Why do tree-based models still outperform deep learning on typical tabular data?
\newblock In {\em Thirty-sixth Conference on Neural Information Processing Systems Datasets and Benchmarks Track}, 2022.

\bibitem{misc_superconductivty_data_464}
K.~Hamidieh.
\newblock {Superconductivty Data}.
\newblock UCI Machine Learning Repository, 2018.
\newblock {DOI}: \href{https://doi.org/10.24432/C53P47}{10.24432/C53P47}.

\bibitem{hancock2020catboost}
J.~T. Hancock and T.~M. Khoshgoftaar.
\newblock {C}at{B}oost for big data: an interdisciplinary review.
\newblock {\em Journal of big data}, 7(1):1--45, 2020.

\bibitem{hastie2009elements}
T.~Hastie, R.~Tibshirani, J.~H. Friedman, and J.~H. Friedman.
\newblock {\em The elements of statistical learning: data mining, inference, and prediction}, volume~2.
\newblock Springer, 2009.

\bibitem{he2014practical}
X.~He, J.~Pan, O.~Jin, T.~Xu, B.~Liu, T.~Xu, Y.~Shi, A.~Atallah, R.~Herbrich, S.~Bowers, et~al.
\newblock Practical lessons from predicting clicks on ads at {F}acebook.
\newblock In {\em Proceedings of the eighth international workshop on data mining for online advertising}, pages 1--9, 2014.

\bibitem{hu2021supervised}
L.~Hu, J.~Chen, J.~Vaughan, S.~Aramideh, H.~Yang, K.~Wang, A.~Sudjianto, and V.~N. Nair.
\newblock Supervised machine learning techniques: An overview with applications to banking.
\newblock {\em International Statistical Review}, 89(3):573--604, 2021.

\bibitem{janzing2020feature}
D.~Janzing, L.~Minorics, and P.~Bl{\"o}baum.
\newblock Feature relevance quantification in explainable {AI}: {A} causal problem.
\newblock In {\em International Conference on artificial intelligence and statistics}, pages 2907--2916. PMLR, 2020.

\bibitem{neptune2}
B.~John.
\newblock \href{https://neptune.ai/blog/when-to-choose-catboost-over-xgboost-or-lightgbm}{When to Choose CatBoost Over XGBoost or LightGBM [Practical Guide]} (neptun.ai), 2022.

\bibitem{2021arXiv210612228J}
M.~{Jullum}, A.~{Redelmeier}, and K.~{Aas}.
\newblock {groupShapley: Efficient prediction explanation with Shapley values for feature groups}.
\newblock {\em arXiv e-prints}, page arXiv:2106.12228, June 2021.

\bibitem{kamath2021explainable}
U.~Kamath and J.~Liu.
\newblock {\em Explainable Artificial Intelligence: An Introduction to Interpretable Machine Learning}.
\newblock Springer, 2021.

\bibitem{kamijo2009two}
Y.~Kamijo.
\newblock A two-step {S}hapley value for cooperative games with coalition structures.
\newblock {\em International Game Theory Review}, 11(02):207--214, 2009.

\bibitem{karczmarz2022improved}
A.~Karczmarz, T.~Michalak, A.~Mukherjee, P.~Sankowski, and P.~Wygocki.
\newblock Improved feature importance computation for tree models based on the {B}anzhaf value.
\newblock In {\em The 38th Conference on Uncertainty in Artificial Intelligence}, 2022.

\bibitem{ke2017lightgbm}
G.~Ke, Q.~Meng, T.~Finley, T.~Wang, W.~Chen, W.~Ma, Q.~Ye, and T.-Y. Liu.
\newblock Light{GBM}: A highly efficient gradient boosting decision tree.
\newblock {\em Advances in neural information processing systems}, 30, 2017.

\bibitem{koralov2007}
L.~{Koralov} and Y.~{Sinai}.
\newblock In {\em Theory of Probability and Random Processes}. Springer, 2007.

\bibitem{2023arXiv230310216K}
K.~{Kotsiopoulos}, A.~{Miroshnikov}, K.~{Filom}, and A.~R. {Kannan}.
\newblock {Approximation of group explainers with coalition structure using Monte Carlo sampling on the product space of coalitions and features}.
\newblock {\em arXiv e-prints}, page arXiv:2303.10216, Mar. 2023.

\bibitem{2017arXiv170701154L}
H.~{Lakkaraju}, E.~{Kamar}, R.~{Caruana}, and J.~{Leskovec}.
\newblock {Interpretable \& Explorable Approximations of Black Box Models}.
\newblock {\em arXiv e-prints}, page arXiv:1707.01154, July 2017.

\bibitem{lou2013accurate}
Y.~Lou, R.~Caruana, J.~Gehrke, and G.~Hooker.
\newblock Accurate intelligible models with pairwise interactions.
\newblock In {\em Proceedings of the 19th ACM SIGKDD international conference on Knowledge discovery and data mining}, pages 623--631, 2013.

\bibitem{lundberg2020local}
S.~M. Lundberg, G.~Erion, H.~Chen, A.~DeGrave, J.~M. Prutkin, B.~Nair, R.~Katz, J.~Himmelfarb, N.~Bansal, and S.-I. Lee.
\newblock From local explanations to global understanding with explainable {AI} for trees.
\newblock {\em Nature machine intelligence}, 2(1):56--67, 2020.

\bibitem{2018arXiv180203888L}
S.~M. {Lundberg}, G.~G. {Erion}, and S.-I. {Lee}.
\newblock {Consistent Individualized Feature Attribution for Tree Ensembles}.
\newblock {\em arXiv e-prints}, page arXiv:1802.03888, Feb. 2018.

\bibitem{lundberg2017unified}
S.~M. Lundberg and S.-I. Lee.
\newblock A unified approach to interpreting model predictions.
\newblock {\em Advances in neural information processing systems}, 30, 2017.

\bibitem{merrick2020explanation}
L.~Merrick and A.~Taly.
\newblock The explanation game: Explaining machine learning models using {S}hapley values.
\newblock In {\em International Cross-Domain Conference for Machine Learning and Knowledge Extraction}, pages 17--38. Springer, 2020.

\bibitem{2021arXiv210210878M}
A.~{Miroshnikov}, K.~{Kotsiopoulos}, K.~{Filom}, and A.~R. {Kannan}.
\newblock {Stability theory of game-theoretic group feature explanations for machine learning models}.
\newblock {\em arXiv e-prints}, page arXiv:2102.10878v5, Feb. 2021.

\bibitem{molnar2020interpretable}
C.~Molnar.
\newblock {\em Interpretable machine learning}.
\newblock Lulu.com, 2020.

\bibitem{montufar2014number}
G.~F. Montufar, R.~Pascanu, K.~Cho, and Y.~Bengio.
\newblock On the number of linear regions of deep neural networks.
\newblock {\em Advances in neural information processing systems}, 27, 2014.

\bibitem{neptune3}
A.~Nahon.
\newblock \href{https://medium.com/riskified-technology/xgboost-lightgbm-or-catboost-which-boosting-algorithm-should-i-use-e7fda7bb36bc}{XGBoost, LightGBM or CatBoost — which boosting algorithm should I use?} (medium.com/riskified-technology), 2019.

\bibitem{2019arXiv190909223N}
H.~{Nori}, S.~{Jenkins}, P.~{Koch}, and R.~{Caruana}.
\newblock {InterpretML: A Unified Framework for Machine Learning Interpretability}.
\newblock {\em arXiv e-prints}, page arXiv:1909.09223, Sept. 2019.

\bibitem{owen1977values}
G.~Owen.
\newblock Values of games with a priori unions.
\newblock In {\em Mathematical economics and game theory}, pages 76--88. Springer, 1977.

\bibitem{raghu2017expressive}
M.~Raghu, B.~Poole, J.~Kleinberg, S.~Ganguli, and J.~Sohl-Dickstein.
\newblock On the expressive power of deep neural networks.
\newblock In {\em international conference on machine learning}, pages 2847--2854. PMLR, 2017.

\bibitem{Reshef2015MeasuringDP}
Y.~A. Reshef, D.~N. Reshef, H.~K. Finucane, P.~C. Sabeti, and M.~Mitzenmacher.
\newblock Measuring dependence powerfully and equitably.
\newblock {\em ArXiv}, abs/1505.02213, 2015.

\bibitem{ribeiro2016should}
M.~T. Ribeiro, S.~Singh, and C.~Guestrin.
\newblock {"Why should I trust you?" Explaining the predictions of any classifier}.
\newblock In {\em Proceedings of the 22nd ACM SIGKDD international conference on knowledge discovery and data mining}, pages 1135--1144, 2016.

\bibitem{ribeiro2018anchors}
M.~T. Ribeiro, S.~Singh, and C.~Guestrin.
\newblock Anchors: High-precision model-agnostic explanations.
\newblock In {\em Proceedings of the AAAI conference on artificial intelligence}, volume~32, 2018.

\bibitem{roe2005boosted}
B.~P. Roe, H.-J. Yang, J.~Zhu, Y.~Liu, I.~Stancu, and G.~McGregor.
\newblock Boosted decision trees as an alternative to artificial neural networks for particle identification.
\newblock {\em Nuclear Instruments and Methods in Physics Research Section A: Accelerators, Spectrometers, Detectors and Associated Equipment}, 543(2-3):577--584, 2005.

\bibitem{rudin2019stop}
C.~Rudin.
\newblock Stop explaining black box machine learning models for high stakes decisions and use interpretable models instead.
\newblock {\em Nature Machine Intelligence}, 1(5):206--215, 2019.

\bibitem{Saabasdoc}
A.~Saabas.
\newblock treeinterpreter python package \url{https://github.com/andosa/treeinterpreter}, 2019.

\bibitem{neptune1}
S.~Saha.
\newblock \href{https://neptune.ai/blog/xgboost-vs-lightgbm}{XGBoost vs LightGBM: How Are They Different} (neptun.ai), 2022.

\bibitem{shapley1953value}
L.~S. Shapley.
\newblock {A value for n-person games, Contributions to the Theory of Games, 2, 307--317}, 1953.

\bibitem{2022arXiv220111931S}
Y.~{Shuo Tan}, C.~{Singh}, K.~{Nasseri}, A.~{Agarwal}, and B.~{Yu}.
\newblock {Fast Interpretable Greedy-Tree Sums (FIGS)}.
\newblock {\em arXiv e-prints}, page arXiv:2201.11931, Jan. 2022.

\bibitem{shwartz2022tabular}
R.~Shwartz-Ziv and A.~Armon.
\newblock {Tabular data: Deep learning is not all you need}.
\newblock {\em Information Fusion}, 81:84--90, 2022.

\bibitem{vstrumbelj2014explaining}
E.~{\v{S}}trumbelj and I.~Kononenko.
\newblock Explaining prediction models and individual predictions with feature contributions.
\newblock {\em Knowledge and information systems}, 41(3):647--665, 2014.

\bibitem{2020arXiv201104041S}
A.~{Sudjianto}, W.~{Knauth}, R.~{Singh}, Z.~{Yang}, and A.~{Zhang}.
\newblock {Unwrapping The Black Box of Deep ReLU Networks: Interpretability, Diagnostics, and Simplification}.
\newblock {\em arXiv e-prints}, page arXiv:2011.04041, Nov. 2020.

\bibitem{2021arXiv211101743S}
A.~{Sudjianto} and A.~{Zhang}.
\newblock {Designing Inherently Interpretable Machine Learning Models}.
\newblock {\em arXiv e-prints}, page arXiv:2111.01743, Nov. 2021.

\bibitem{sundararajan2020many}
M.~Sundararajan and A.~Najmi.
\newblock The many {S}hapley values for model explanation.
\newblock In {\em International conference on machine learning}, pages 9269--9278. PMLR, 2020.

\bibitem{sundararajan2017axiomatic}
M.~Sundararajan, A.~Taly, and Q.~Yan.
\newblock Axiomatic attribution for deep networks.
\newblock In {\em International conference on machine learning}, pages 3319--3328. PMLR, 2017.

\bibitem{misc_ailerons_data}
L.~Torgo and R.~Camacho.
\newblock \href{https://www.openml.org/search?type=data&status=active&id=296}{Ailerons Data}.
\newblock OpenML Data Repository, 2014.

\bibitem{turgeman2016mixed}
L.~Turgeman and J.~H. May.
\newblock A mixed-ensemble model for hospital readmission.
\newblock {\em Artificial intelligence in medicine}, 72:72--82, 2016.

\bibitem{2018arXiv180601933V}
J.~{Vaughan}, A.~{Sudjianto}, E.~{Brahimi}, J.~{Chen}, and V.~N. {Nair}.
\newblock {Explainable Neural Networks based on Additive Index Models}.
\newblock {\em arXiv e-prints}, page arXiv:1806.01933, June 2018.

\bibitem{misc_higgs_280}
D.~Whiteson.
\newblock {HIGGS}.
\newblock UCI Machine Learning Repository, 2014.
\newblock {DOI}: \href{https://doi.org/10.24432/C5V312}{10.24432/C5V312}.

\bibitem{wu2010adapting}
Q.~Wu, C.~J. Burges, K.~M. Svore, and J.~Gao.
\newblock Adapting boosting for information retrieval measures.
\newblock {\em Information Retrieval}, 13(3):254--270, 2010.

\bibitem{2021arXiv210909847Y}
J.~{Yang}.
\newblock {Fast TreeSHAP: Accelerating SHAP Value Computation for Trees}.
\newblock {\em arXiv e-prints}, page arXiv:2109.09847, Sept. 2021.

\bibitem{yang2020enhancing}
Z.~Yang, A.~Zhang, and A.~Sudjianto.
\newblock Enhancing explainability of neural networks through architecture constraints.
\newblock {\em IEEE Transactions on Neural Networks and Learning Systems}, 32(6):2610--2621, 2020.

\bibitem{yang2021gami}
Z.~Yang, A.~Zhang, and A.~Sudjianto.
\newblock {GAMI-Net: An explainable neural network based on generalized additive models with structured interactions}.
\newblock {\em Pattern Recognition}, 120:108192, 2021.

\bibitem{young1985monotonic}
H.~P. Young.
\newblock Monotonic solutions of cooperative games.
\newblock {\em International Journal of Game Theory}, 14(2):65--72, 1985.

\bibitem{zhang2019predictive}
Z.~Zhang, Y.~Zhao, A.~Canes, D.~Steinberg, O.~Lyashevska, et~al.
\newblock Predictive analytics with gradient boosting in clinical medicine.
\newblock {\em Annals of translational medicine}, 7(7), 2019.

\bibitem{2021arXiv210508589Z}
W.~{Zhao}, R.~{Singh}, T.~{Joshi}, A.~{Sudjianto}, and V.~N. {Nair}.
\newblock {Self-interpretable Convolutional Neural Networks for Text Classification}.
\newblock {\em arXiv e-prints}, page arXiv:2105.08589, May 2021.

\end{thebibliography}
\bibliographystyle{abbrv}

\end{document}